\documentclass[11pt]{article}
\usepackage{fullpage}

\usepackage{geometry}

\usepackage{tikz}
\usetikzlibrary{shapes.geometric, arrows, positioning, calc}
\usepackage{booktabs}						
\usepackage{multirow}
\usepackage{amsfonts}						
\usepackage{graphicx}						
\usepackage{duckuments}		
\usepackage{etoolbox}

\usepackage{minitoc}

\AtBeginEnvironment{itemize}{\setlength{\leftmargini}{1em}}
\AtBeginEnvironment{enumerate}{\setlength{\leftmargini}{0.1em}}
\usepackage[sort,round]{natbib}			

\usepackage{tcolorbox}

\usepackage[utf8]{inputenc} 
\usepackage[T1]{fontenc} 
\usepackage{url}
\usepackage{amsmath}
\usepackage{amssymb}
\usepackage{mathtools}
\usepackage{amsthm}
\usepackage{makecell}

\allowdisplaybreaks
\usepackage{booktabs}       
\usepackage{color}
\usepackage{amsfonts}       
\usepackage{nicefrac}       
\usepackage{enumerate}   
\usepackage{bm}
\usepackage{microtype}      
\usepackage{xcolor}         
\usepackage{bbm}
\usepackage{xfrac}
\usepackage{tocbibind}
\usepackage[title]{appendix}
\usepackage{tikz}
\usepackage{physics}
\usepackage{hyperref}       

\hypersetup{
    pdfnewwindow=true,     
    colorlinks=true,     
    linkcolor=cyan,     
    citecolor=cyan,     
    filecolor=magenta,     
    urlcolor=cyan     
}
\usepackage{url} 

\usepackage{caption}
\usepackage{subcaption}

\usepackage{listofitems} 
\usepackage{subfiles}
\usepackage[
  separate-uncertainty = true,
  multi-part-units = repeat
]{siunitx}

 \makeatletter
\newtheorem*{rep@theorem}{\rep@title}
\newcommand{\newreptheorem}[2]{%
\newenvironment{rep#1}[1]{%
 \def\rep@title{#2 \ref{##1}}%
 \begin{rep@theorem}}%
 {\end{rep@theorem}}}
\makeatother

\makeatletter
\newcommand{\maketitlepage}{%
    \let\thanks\@gobble
    \let\footnote\@gobble
    \@maketitle
}
\makeatother

\newcommand{\ols}[1]{\,\widetilde{\!{#1}}} 

\allowdisplaybreaks
\usepackage{xspace}
\theoremstyle{plain}
\newtheorem{theorem}{Theorem}
\newtheorem{proposition}{Proposition}
\newtheorem{lemma}{Lemma}

\newtheorem{definition}{Definition}
\newtheorem{assumption}{Assumption}
\newtheorem{remark}{Remark}

\newreptheorem{theorem}{Theorem}
\newreptheorem{lemma}{Lemma}
\newreptheorem{proposition}{Proposition}
\newreptheorem{corollary}{Corollary}
\newreptheorem{assumption}{Assumption}

\newcommand{\nn}{\nonumber}
\newcommand{\blue}{\color{black}}
\newcommand{\black}{\color{black}}
\newcommand{\nualpha}{\nu_\alpha}

\newcommand{\mfm}{\mathfrak{m}}
\newcommand{\mrd }{\mathrm{d}}
\newcommand{\mdens}{\mfm_{\theta}}

\newcommand{\KLr}{\mathrm{KL}}
\newcommand{\mbR}{\mathbb{R}}
\newcommand{\nupop}{\nu_{\mathrm{pop}}}

\newcommand{\mbZnt}{\mathbf{Z}_{n_t}^t}
\newcommand{\mbZns}{\mathbf{Z}_{n_s}^s}
\newcommand{\nut}{\nu_{n_t}^t}
\newcommand{\nus}{\nu_{n_s}^s}
\newcommand{\nutr}{\nu_{n_t,(1)}^t}

\newcommand{\spec}{\mathrm{sp}}
\newcommand{\thetac}{\theta_{\mathrm{c}}}
\newcommand{\thetas}{\theta_{\spec}}
\newcommand{\Thetac}{\Theta_{\mathrm{c}}}
\newcommand{\Thetas}{\Theta_{\spec}}
\newcommand{\mfmc}{\mathfrak{m}_{\mathrm{c}}}
\newcommand{\mfms}{\mathfrak{m}_{\spec}}

\newcommand{\mrmbarc}{\bar{\mathrm{m}}_{\mathrm{c}}}
\newcommand{\mrmbarsp}{\bar{\mathrm{m}}_{\spec}}
\newcommand{\Dispth}{\mathrm{d}_{\mathfrak{F}_p}}
\newcommand{\Disfth}{\mathrm{d}_{\mathfrak{F}_4}}
\newcommand{\gammac}{\hat\gamma_{8,\mathrm{c}}^\sigma}
\newcommand{\gammasp}{\hat\gamma_{8,\spec}^\sigma}
\newcommand{\tilgammap}{\tilde\gamma_8^\sigma}
\newcommand{\tilgammasp}{\tilde\gamma_{\spec}^\sigma}
\newcommand{\Disnd}{\mathrm{d}_{\mathfrak{F}_2}}
\newcommand{\dimc}{k_{d_\mathrm{c}}}
\newcommand{\dims}{k_{d_{\spec}}}
\newcommand{\dimg}{k_{d}}
\newcommand{\comp}{\mathrm{Comp}}
\newcommand{\mbE}{\mathbb{E}}
\newcommand{\mrR}{\mathrm{R}}

\newcommand{\mcZ}{\mathcal{Z}}

\newcounter{aliCounter}

\newcounter{lukaszCounter}

\newcounter{samCounter}

\usepackage{xspace}
\usepackage[colorinlistoftodos, color=blue!30!white 
	 ,disable
]{todonotes}                                        
\usepackage{parskip}

\usepackage{natbib}
\setcitestyle{citesep={;}}
\usepackage{cite}

\usepackage{xspace}

\usepackage{parskip}

\title{Understanding Transfer Learning via Mean-field Analysis}

\author{
  Gholamali Aminian$^{ 1}$
  \and
  {\L}ukasz Szpruch$^{ 1\,3}$
  \and
  Samuel N. Cohen$^{1\,2}$
}

\begin{document}

\maketitle
\renewcommand\thefootnote{$^1$}\footnotetext{The Alan Turing Institute.}
\renewcommand\thefootnote{$^2$}\footnotetext{Mathematical Institute, University of Oxford.}
\renewcommand\thefootnote{$^3$}\footnotetext{School of Mathematics, University of Edinburgh.}

%

%








\begin{abstract}
We propose a novel framework for exploring generalization errors of transfer learning through the lens of differential calculus on the space of probability measures. In particular, we consider two main transfer learning scenarios, $\alpha$-ERM and fine-tuning with the KL-regularized empirical risk minimization and establish generic conditions under which the generalization error and the population risk convergence rates for these scenarios are studied. Based on our theoretical results, we show the benefits of transfer learning with a one-hidden-layer neural network in the mean-field regime under some suitable integrability and regularity assumptions on the loss and activation functions.
\end{abstract}


\newpage
\tableofcontents
\newpage
\section{Introduction}\label{sec: intro}

In supervised learning, a common assumption posits that both training and test datasets originate from the same data-generating distribution. This assumption often fails in real-world scenarios; for example, while a large amount of data is often available from the source task, we may aim to deploy a model -- trained on this source data and a little data from the target task -- on a disparate target task. To address this discrepancy, methodologies such as transfer learning and domain adaptation\footnote{In domain adaptation, we do not have access to the target task dataset and the model is trained solely on the source task dataset.} have been formulated. Recent advances in transfer learning, particularly algorithms leveraging pre-trained models followed by fine-tuning, or mixing the source and target tasks datasets, have achieved noteworthy progress across many domains  including computer vision, natural language processing and large language models~\citep{li2012human,long2015learning,yosinski2014transferable,ding2023parameter}. However, our understanding of transfer learning within neural networks remains limited.

A crucial problem in transfer learning theory is understanding the performance of a learning algorithm -- trained on both source and the target tasks datasets -- on the target task unseen data. This is described by the transfer generalization error, which quantifies the deviation between the algorithm's performance on the training data from the target task and its performance on unseen data sampled from the target task distribution. In the under-parameterized regime, where the number of model parameters is significantly less than the number of training data points, different approaches have been applied to study the theory of the transfer generalization error~\citep{hanneke2019value,tripuraneni2020theory,kalan2020minimax}. However, in the overparameterized regime, where the number of parameters may greatly exceed the number of training data points in both source and target tasks, these approaches become inadequate. Mathematical models, such as the Neural Tangent Kernel~\citep{jacot2018neural}, Mean-Field \citep{mei/montanari/nguyen:2018}, and Random Feature models \citep{rahimi2008uniform}, have been proposed to understand the behavior of overparameterized neural networks (NNs) in supervised learning. However, our current understanding of generalization in these contexts for transfer learning still needs to be completed.

Our approach is motivated by \citep{aminian2023mean} where a novel framework via differential calculus over measure space is proposed to study the supervised learning algorithms in a mean-field regime. The picture that emerges from the mean-field regime is that the aim of a learning algorithm is to find an optimal distribution over the parameter space (rather than optimal values of the parameters). Our work explores how this mean-field view can illuminate the transfer generalization performance of overparameterized neural networks in transfer learning contexts, specifically addressing $\alpha$-ERM and fine-tuning techniques.

The contributions of this work are as follows:

\begin{itemize} 
\item We utilize calculus on the space of measures to derive the generalization error of $\alpha$-ERM and fine-tuning transfer learning scenarios.
\item The Kullback--Leibler (KL) regularized empirical risk minimization, in both transfer learning scenarios, is studied and convergence rates on generalization error and population risk are provided.
\item For one-hidden-layer neural network in the mean-field regime, our analysis reveals precise conditions that guarantee the desired convergence rate.
\end{itemize}

\section{Preliminaries}\label{Sec:Problem-formulation}

\paragraph{Notation:}We adopt the following convention for random variables and their distributions. 
A random variable is denoted by an upper-case letter (e.g., $Z$), its space of possible values is denoted with the corresponding calligraphic letter (e.g. $\mcZ$), and an arbitrary value of this variable is denoted with the lower-case letter (e.g., $z$). We will write $\mathbb{E}_{Z}[\cdot]$ for the expectation taken over $Z$, all other random variables being left constant\footnote{Formally, this corresponds to the conditional expectation over all variables in our setting excluding $Z$. As we will only have countably many variables in our problem, this does not cause any technical difficulties.}. We will further write $\mathbb{E}_{Z\sim m}[\cdot]$ for the expectation over a random variable $Z$ independent of all others, which is distributed according to $m$, and similarly $\mathbb{V}_{Z\sim m}[\cdot]$ for the variance. We write $\delta_z$ for a Dirac measure supported at $z$.

 If $A$ is a normed space, then $ \mathcal{P}(A)$ denotes the space of probability measures on $A$ and $\mathcal{P}_{p}(A)$ the probability measures with finite $p$-th moment. We equip spaces of probability measures with the weak topology and its associated Borel $\sigma$-algebra. For the spaces under consideration (including finite products thereof), we assume a priori defined metrics. Furthermore, we posit that the metric on a product space is equivalent to the product (Euclidean) metric. 

We now introduce the functional linear derivative~\citep{cardaliaguet2019master} for functionals on measure spaces. For simplicity\footnote{This can be relaxed, along with most integrability assumptions in this paper, at a cost of complexity.}, we will restrict our attention to the finite variance ($\mathcal{P}_2(\mbR^n)$ case).

\begin{definition}{Extending \citep[Definition 5.43]{carmona/delarue:2018}}
\label{def:flatDerivative}
Consider $U:\mathcal P_2(\mathbb R^n) \times \mathbb R^k \to \mathbb R$. We say $m\mapsto U(m,x)$ is of class $\mathcal C^1$ if there exists a map $\frac{\delta U}{\delta m} : \mathcal P_2(\mathbb R^n) \times \mathbb R^k \times \mathbb R^n \to \mathbb R$, such that 
\begin{enumerate}[(i)]
    \item $\frac{\delta U}{\delta m}$ is measurable with respect to $x,a$, and continuous with respect to $m$;
    \item \label{flatder_item1}for every bounded set $B\subset \mathcal P_2(\mathbb R^n) \times \mathbb R^k$, there exists a constant $C>0$ such that  $|\frac{\delta U}{\delta m}(m,x,a)|\leq C(1+|a|^2)$ for all $(m,x)\in B$;
    \item for all $m, m' \in\mathcal P_2(\mathbb R^n)$, 
\end{enumerate}
\[\begin{split}
    &U(m',x) - U(m,x) \\&= \int_0^1 \int_{\mbR^n} \frac{\delta U}{\delta m}(m + \lambda(m' - m),x,a) \, (m'
-m)(da)\,\mrd \lambda\,.
\end{split}
\]
Since $\frac{\delta U}{\delta m}$ is only defined up to a constant we demand $\int \frac{\delta U}{\delta m}(m,x,a)\,m(da) = 0$. By extension, we say $U$ is of class $\mathcal{C}^2$ if both $U$ and $\frac{\delta U}{\delta m}$ are of class $\mathcal{C}^1$.
\end{definition}

Similarly, derivatives for measure-valued functionals can be defined \citep[Definition~2.2]{aminian2023mean}.

\subsection{Problem formulation} \label{sec: problem formulation}
\blue
Consider a transfer learning scenario, with the same input space $\mathcal{X}$ and target space $\mathcal{Y}$ for both source and the target tasks. Define $\mcZ:=\mathcal{X}\times\mathcal{Y}$. We are often interested in learning a function $f : \mathcal{X} \to \mathcal{Y}$  parameterized  using a (large) number of parameters from $\Theta\subseteq \mathbb{R}^{\dimg}$, for some $\dimg>0$. 

Let $\nupop^t \in \mathcal{P}_2(\mcZ)$ be the (unknown) distribution, describing the joint values of $(x,y)$ in the population of the target task. Similarly, we define $\nupop^s \in \mathcal{P}_2(\mcZ)$ for the source task. Intuitively, we suppose that the mass of $\nupop^t$ is near the graph of $f$ that is, $y^t\approx f(x^t)$ when $z^t=(x^t,y^t)$ is sampled from $\nupop^t$. Since $\nupop^t$ is unknown, we approximate $f$ based  on a finite dataset from both target and the source tasks, $\mbZnt =  \{Z_{i}^t\}_{i=1}^{n_t}$ and $\mbZns =  \{Z_{i}^s\}_{i=1}^{n_s}$. We make the following basic assumption on our data throughout the paper, that our observations from the target tasks are given by $\mbZnt =  \{Z_{i}^t\}_{i=1}^{n_t}$, where $Z_i^t=(X_{i}^t,Y_{i}^t)\sim \nupop^t$ are i.i.d. We write $\nut:= \frac{1}{n_t}\sum_{i=1}^{n_t} \delta_{Z_i^t}$ for the corresponding target empirical measure. Similarly, we define $\mbZns =  \{Z_{i}^s\}_{i=1}^{n_s}$, where $Z_i^s=(X_{i}^s,Y_{i}^s)\sim \nupop^s$ are i.i.d. and $\nus:= \frac{1}{n_s}\sum_{i=1}^{n_s} \delta_{Z_i^s}$.

We are interested in quantifying the performance of our model on unseen data from the target task. To this end, let $\mathbf{\ols{Z}}_{n_t}^t=\{\ols{Z}_i^t\}_{i=1}^{n_t}$ be a second i.i.d sample set with law $\nupop^t$, independent of $\mbZnt$. We define the perturbations of $\nu_{n_t}$ obtained by `resampling' one or two data points, by 
\begin{equation}\label{Eq: nu replace one}\nutr=\nut+\frac{1}{n_t}(\delta_{\ols{Z}_1^t}-\delta_{Z_1^t})\,.\end{equation}
\paragraph{Transfer Learning algorithms:}
In a mean-field approach, we represent transfer learning as a map from the observed target and source empirical measures to a measure over parameters, $(\nut,\nus) \mapsto \mfm(\nut,\nus) \in \mathcal{P}(\Theta)$.  We motivate this representation below.
In this work, we study two well-known transfer learning algorithms, $\alpha$-ERM and fine-tuning in a mean-field regime. We will specify how this map is chosen in the next section.
 
 \textbf{$\alpha$-ERM}: Inspired by \citep{ben2010theory}, we train the model based on a convex combination of empirical risks of source and the target tasks,
    \[\alpha \mrR(m,\nut) + (1-\alpha) \mrR(m,\nus), \quad \alpha\in[0,1].\]
    The $\alpha$-ERM transfer learning algorithm is represented as a (measurable) map from the convex combination $\nu_\alpha:=\alpha \nut + (1-\alpha) \nus$ of empirical measures to a measure over parameters $\theta\in\Theta\subset \mathbb{R}^{\dimg}$, $\nualpha\mapsto\mfm_{\alpha}(\nualpha)$. Therefore, for $\alpha$-ERM we have \[\mfm(\nut,\nus)=\mfm_{\alpha}(\nualpha).\]
    
     \textbf{Fine-tuning:} Inspired by \citep{tripuraneni2020theory} in fine-tuning algorithm, we consider two sets of parameters for training, including common parameters $\theta_c\in\Theta_c\subset \mathbb{R}^{\dimc}$, and specific parameters for source and the target tasks, $\theta_{\spec}^s\in \Theta_{\spec}\subset \mathbb{R}^{\dims}$ and $\theta_{\spec}^t\in \Theta_{\spec}\subset \mathbb{R}^{\dims}$, respectively. Conceptually, we first train (a measure over) $\thetas^s$ and $\thetac$ based on the source task's dataset. Then, we fine-tune the specific parameters, $\thetas^t$, based on the target task, while leaving the common parameters fixed. Formally, for the target task, this gives us a parameter measure of the form,
    \begin{equation}\label{eq:fine_tuning_structure}\mfm(\nut,\nus)=\mfm_{\mathrm{c}}(\nus)\mfm_{\spec}^t(\mfm_{\mathrm{c}}(\nus), \nut),\end{equation}
where $\mfm_{\mathrm{c}}:\mathcal{P}(\mcZ)\mapsto \mathcal{P}(\Theta_{\mathrm{c}})$ and $\mfm_{\spec}^t:\mathcal{P}(\Theta_{\mathrm{c}})\times \mathcal{P}(\mcZ)\mapsto \mathcal{P}(\Theta_{\spec})$ represent the measures over the common parameters induced by training on the source dataset and target-specific parameters induced by training on target dataset and fixing the common parameter measure, respectively. 

Note that the model parameter in fine-tuning is the union of common and target specific parameters in fine-tuning, i.e.~$\theta=\theta_{\mathrm{c}}\cup\theta_{\spec}^t$. 

To investigate resampling the training data, we extend \eqref{Eq: nu replace one} and make use of the convex perturbations, for $\lambda\in [0,1]$,
\begin{equation}
\label{mperturb1}
\begin{split}
&\mfm_{(1)}(\lambda)=\mfm(\nu_{n_t,(1)}^t,\nu_{n_s}^s)\\
&\qquad+\lambda (\mfm(\nut,\nus)-\mfm(\nu_{n_t,(1)}^t,\nu_{n_s}^s)),\\
&\nu_{(1)}^{t}(\lambda)=\nu_{n_{t},(1)}^{t}+ \lambda (\nu_{n_{t}}^t-\nu_{n_{t},(1)}^{t}).
\end{split}
\end{equation}
\subsubsection{Loss function}
We focus on training methods where $\mfm$ is the minimizer of a loss function. We generically represent the single-observation loss function as 
\begin{equation}\label{eq: general loss function}
(m,z)\mapsto \ell(m,z) \in \mathbb R^+\,.
\end{equation}
This represents the loss, for the parameter distribution $m$ and a data point $z$. For the fine-tuning scenario, given the structure shown in \eqref{eq:fine_tuning_structure}, we consider the map $(m_c m_{\spec},z)\mapsto \ell(m_{c}m_{\spec},z) \in \mathbb R^+\,$ where $m_c$ and $m_{\spec}$ are the common and specific parameter distributions. 

%

In order to motivate our generic formulation of this problem, we present the concrete example of overparameterized one-hidden-layer neural in mean-field regime.

\textbf{Overparameterized one-hidden-layer neural network (NN):} We consider a one-hidden-layer NN.
Let $x\in \mathbb R^q$ be a feature vector. For each of $r$ hidden neurons, we denote the parameters of the hidden layer by $w\in\mbR^q$ and the outer layer by $a\in\mbR$. The parameter space is $\Theta=\mbR^{(q+1)}=\{a\in\mbR,w\in\mbR^q\}$.  To ease notation, we write $\phi(\theta_i,x)=a_i\varphi(w_i\cdot x)$, where $\varphi(\cdot)$ is the activation function.  The output of the NN is
\begin{equation}\label{Eq: phi def}
\begin{split}
    \hat f(x) &= \frac{1}{r}\sum_{i=1}^r a_i\varphi( w_i\cdot x)=\frac{1}{r}\sum_{i=1}^r \phi(\theta_i,x)\\
    &= \int \phi(\theta, x)\, m_r\! (\mrd \theta)= \mathbb{E}_{\theta\sim m_r}[\phi(\theta,x)],
    \end{split}
\end{equation}
 where $m_r:=\frac{1}{r}\sum_{i=1}^r \delta_{(a_i,w_i)}.$ 
Observe that any hidden nodes can be exchanged (along with their parameters) without changing the output of the network; this symmetry implies that it is the (joint) distribution $m_r$ which is important. As $r$ increases, $m_r$ can converge (weakly) to a continuous distribution over the parameter space; the weights for individual neurons can be viewed as samples from this measure.  This is the mean-field model studied in~\citep{mei/montanari/nguyen:2018,mei/misiakiewicz/montanari:2019,MFLD}. 

Training neural networks chooses parameters $\theta$ to minimize a loss $\ell_o$ between $\hat f(x)$ and the observed value $y$. With $\hat y:=\hat f(x) = \mathbb{E}_{\theta\sim m_r}[\phi(\theta, x)]$, we can write our single-observation loss function as 
\begin{equation}
    \label{eq: nn loss function}
    \begin{split}
            \ell(m, z):=\ell_o\big(\mathbb{E}_{\theta\sim m}[\phi(\theta,x)],y\big).
                \end{split}
\end{equation} 
For example,  consider $\ell_o\big(\hat y,y\big)=c(y\hat y)$ for binary classification, where $c(\cdot)$ is a margin-based loss function~\citep{bartlett2006convexity}, or absolute loss $\ell_o(\hat y,y)=|y-\hat y|$ or quadratic loss $\ell_o(\hat y,y)=(y-\hat y)^2$ for classification and regression problems. As we can observe, for overparameterized one-hidden-layer NN, as in \eqref{eq: nn loss function} the loss function is non-linear in $m$.

\subsubsection{Risks}
The risk function\footnote{The functional $\mrR$ is (weakly) continuous in $\nu$, and is measurable in $m$ (as a consequence of the measurability of $\ell$). As $\ell(m, z)$ is nonnegative, the integral \eqref{eq risk linear} can always be defined, without integrability assumptions on $\ell$, but may take the value $+\infty$.} takes distributions $m$ of parameters and $\nu$ of data points and evaluates the loss: 
\begin{align} \label{eq risk linear}
    \mrR(m,\nu):= \int_{\mathcal Z} \ell\big( m,z\big)\nu(\mrd z)\,.
   \end{align}
   The population risk for the target task is $\mrR(m, \nupop^t)$, and the empirical risk for the source and the target tasks are  $\mrR(m,\nut)$ and $\mrR(m,\nus)$, respectively, as
\begin{equation*}
    \begin{split}
        \mrR(m,\nu_{n_t}^t) &=  \int_{\mathcal Z} \ell\big( m,z\big)\nu_{n_t}^t(\mrd z)=\frac{1}{n_t}\sum_{i=1}^{n_t} \ell\big( m,z_i^{t}\big),
    \end{split}
\end{equation*} 
and similarly for $\nu^s_{n_s}$.
\paragraph{Weak Transfer Generalization error:} We aim to study the performance of the model trained with the empirical data measures $(\nut,\nus)$, and evaluated against the population measure $\nupop^t$, that is, $\mrR(\mfm(\nut,\nus),\nupop^t)$.
The risk can be decomposed:
\begin{equation}
\begin{split}\label{def generalization error} 
 &\mrR(\mfm(\nut,\nus),\nupop^t) \\&=    \underbrace{\Big(\mrR(\mfm(\nut,\nus),\nupop^t) - \mrR(\mfm(\nut,\nus),\nut)\Big)}_{\text{transfer generalization error}}\\&\qquad+ \underbrace{\mrR(\mfm(\nut,\nus),\nut)}_{\text{training error of the target task}}.
\end{split}
\end{equation}
\black
\begin{assumption}\label{assn:training integrable}
    The training map $\mfm$, loss function $\ell$, and population measure $\nu$ are such that the training error satisfies $\mathbb{E}_{\mathbf{Z}_n^t}[\mrR(\mfm(\nut),\nut)] = \mathbb{E}_{\mathbf{Z}_n^t}[\int_\mcZ\ell(\mfm(\nut),z^t)\nut(\mrd z^t)]<\infty$. The same assumption also holds for source task.
\end{assumption}
\blue
 For a training map $\mfm$, we will consider the Weak Transfer Generalization Error (WTGE),
\begin{equation}\label{eq:WTGE}
\begin{split}
       \mathrm{gen}(\mfm,\nupop^t)&\triangleq \mathbb{E}_{\mbZnt,\mbZns}\big[ \mrR(\mfm,\nupop^t)  - \mrR(\mfm,\nut) \big].
       \end{split}
\end{equation}
\paragraph{Weak Transfer Excess Risk:} In addition, we also define the Weak Transfer Excess Risk (WTER) for a training map as,
\begin{equation}\label{eq:WTER}
\begin{aligned}
        &\mathcal{E}(\mfm,\nupop^t)\\&\triangleq \mathbb{E}_{\mbZnt,\mbZns}\big[ \mrR(\mfm,\nupop^t)]-\inf_{m\in\mathcal{P}(\Theta)}\mrR(m,\nupop^t).
        \end{aligned}
\end{equation}
\begin{remark}
    We will use both the WTGE and WTER in what follows. In the context of transfer learning, these two quantities have slightly different interpretations: WTGE measures the average overfitting of our model relative to our training data, while WTER measures the average overfitting relative to a perfect model for the target problem. Ultimately, our goal is to control the excess risk (WTER); we will do this by proving bounds on the WTGE as an intermediate step. 
\end{remark}
\subsection{Related works}
{\blue \paragraph{Generalization error and transfer learning:}
Various works have tried to understand the empirical performance of transfer learning and domain adaptation. The first theoretical study addressing domain adaptation was presented by \citep{ben2007analysis}, focusing on binary classification. This work establishes an excess risk bound for the zero-one loss using the VC-dimension, grounding it in the $d_\mathcal{A}$-distance—a metric quantifying the disparity between source and the target tasks. Further deepening the exploration, \citep{hanneke2019value} introduced a novel metric for assessing discrepancy in transfer learning termed the transfer-exponent, operating under the covariate-shift and  a Bernstein class condition.  \citep{kalan2020minimax} formulated a minimax lower bound on the generalization error associated with transfer learning in the neural network. A more recent development is the introduction of an Empirical Risk Minimization (ERM) strategy motivated by representation learning by \citep{tripuraneni2020theory}, which offers an upper boundary on the excess risk for the new task via Gaussian complexity analysis. Advancing the discourse, \citep{wang2019transfer} delineated an upper bound on excess risk leveraging importance weighting. Lastly, \citep{wu2020information} proposed an information-theoretic generalization error upper bound for transfer learning, employing the KL divergence as a metric for the similarity between source and target data distributions. Our work differs from these, as our focus is on over-parameterized models for transfer learning in a mean-field regime}
\black
\section{Generalization Errors Via Functional  Derivatives}\label{sec_main_results}
In this section, we apply functional calculus to study transfer generalization error (WTGE); these calculations will also enable us to prove bounds on the excess risk (WTER). Our goal is to give abstract conditions under which bounds can be established, which can be applied in a range of contexts, as will be explored in subsequent sections. We begin this section by providing a general representation of generalization error, inspired by the approach of \citep[Lemma 7]{Bousquet-Elisseeff} for transfer learning. All proofs are deferred to Appendix~\ref{App: Proofs of GE}.

\begin{lemma}\label{Lemma: rep gen tf}
   Consider a generic loss function $(m,z)\mapsto \ell(m,z)$, and $(\nutr,\nu_{n_s}^s)$ as defined in \eqref{Eq: nu replace one}. The WTGE \eqref{eq:WTGE} is given by,
    \begin{align}
        \label{eq_intro_loss_gap-transfer}
     &\mathrm{gen}(\mfm(\nut,\nus),\nupop^t) =\\\nonumber&\mathbb{E}_{\mbZnt,\mbZns,\ols{Z}_1^t} \Big[\ell\big(\mfm(\nut,\nus),\ols{Z}_1^t\big) -   \ell\big(\mfm(\nutr,\nus),\ols{Z}_1^t\big) \Big].
     \end{align}
\end{lemma}
The right hand side of \eqref{eq_intro_loss_gap-transfer} measures the expected change in the loss function when resampling one training data point from the target task, which connects weak transfer generalization error to stability of $\mfm$ with respect to both target training dataset. We will next quantify this stability precisely in transfer learning, in terms of functional derivatives. 

\begin{assumption}\label{main_ass_square_bounded}
In addition to Assumption \ref{assn:training integrable}, 
\begin{enumerate}[(i)]  
\item\label{main_ass_square_bounded_item1} The loss function $\ell$ is $\mathcal{C}^1$ (cf. Definition \ref{def:flatDerivative}), nonnegative, and convex with respect to $m$;
\item\label{main_ass_square_bounded_item4} When restricted to $\nu^t,\nu^s \in \mathcal{P}_2(\mcZ)$, the training map $(\nu^t,\nu^s)\mapsto \mfm(\nu^t,\nu^s)$ is $\mathcal{C}^1$, in the sense of \citep[Definition~2.2]{aminian2023mean}.
\end{enumerate}
\end{assumption}

\begin{theorem}\label{thm: weak_gen_func_transfer} 
Given Assumption \ref{main_ass_square_bounded}, the weak transfer generalization error has the representation
\[\begin{split}
    &\mathrm{gen}(\mfm,\nupop^t)\\&=\mathbb{E}_{\mbZnt,\mbZns,\ols{Z}_1^t}\Big[\int_{0}^1 \int_{\Theta}\Big( \frac{\delta \ell}{\delta m_{\spec}}\big(\mfm_{(1)}(\lambda),\ols{Z}_1^t,\theta\big)\Big) \\&\qquad\big(\mfm(\nut,\nus)-\mfm(\nutr,\nus)\big)(\mrd \theta)\,\mrd \lambda\Big].
    \end{split}
\]
\end{theorem}

By applying the functional linear derivative to the term $\big(\mfm(\nut,\nus)-\mfm(\nutr,\nus)\big)$ in Theorem \ref{thm: weak_gen_func_transfer}, we can provide yet another representation of the WTGE for $\alpha$-ERM and fine-tuning scenarios.
\begin{theorem}[$\alpha$-ERM]\label{thm: another rep WTGE-alpha-scenario} The WTGE of $\alpha$-ERM can be written 
\begin{equation}\label{Eq:Finite-term-trnasfer}
        \mathrm{gen}(\mfm,\nupop^t)  =\frac{\alpha}{n_t}\mathbb{E}_{\mbZnt,\mbZns,\ols{Z}_1^t}\big[ h(\mbZnt,\mbZns,\ols{Z}_1^t)\big],
       \end{equation} 
       where \\
    \mbox{\small$h(\mbZnt,\mbZns,\ols{Z}_1^t) =\int_{0}^1\int_{0}^1\int_{\mcZ} A(\lambda,\tilde \lambda)\big(\delta_{Z_1^t}-\delta_{\tilde Z_1^t}\big)(\mrd z) \mrd \tilde{\lambda}\mrd \lambda,$ } {\small
    $ A(\lambda,\tilde \lambda)= \int_{\Theta}  \frac{\delta \ell}{\delta m_{\spec}}\big(\mfm_{(1)}(\lambda),\ols{Z}_1^t,\theta\big)\frac{\delta \mfm}{\delta \nu}\big(\nu_{\alpha,(1)}(\tilde{\lambda}),z\big)(\mrd \theta),$} and $ \nu_{\alpha,(1)}(\tilde\lambda)= ( \alpha \nutr + (1-\alpha)\nus) + \tilde\lambda \alpha (\nut-\nutr).$
   \end{theorem}

\begin{theorem}[Fine-tuning]\label{thm: another rep fine-tune}

The WTGE of fine-tuning can be written,
\begin{equation}\label{eq: fine-tune_func}
\mathrm{gen}(\mfm,\nupop^t)  =\frac{1}{n_t}\mathbb{E}_{\mbZnt,\mbZns,\ols{Z}_1^t}\big[ K(\mbZnt,\mbZns,\ols{Z}_1^t)\big] ,    
\end{equation}
 
where
\begin{align*}&K(\mbZnt,\mbZns,\ols{Z}_1^t)\\
&=\int_{0}^1 \int_{\Theta_{\spec}}\Big( \frac{\delta \ell}{\delta m_{\spec}}\big(\mfm_{(1)}^f(\lambda),\ols{Z}_1^t,\theta_{\spec}^t\big)\Big) \\&\quad
 \times\bigg(\int_{0}^1\int_{\mcZ}\frac{\delta\mfms}{\delta \nu}(\mfmc(\nus),\nutr(\lambda_1),z)(\mrd \thetas^t)\\&\quad\quad(\delta_{Z_1^t}-\delta_{\ols{Z}_1^t})(\mrd z) \mrd \lambda_1\bigg)\,\mrd \lambda,\end{align*}
and $\mfm_{(1)}^f(\lambda)=\mfm_{\mathrm{c}}(\nus)\big[(1-\lambda)(\mfm_{\spec}^t(\mfm_{\mathrm{c}}(\nus), \nut)+\lambda\mfm_{\spec}^t(\mfm_{\mathrm{c}}(\nus), \nutr)\big].$
\end{theorem}

   \begin{remark}[Convergence rate]The representation in \eqref{Eq:Finite-term-trnasfer} reveals that the convergence rates of the WTGE based on $\alpha$-ERM  is at worst $\mathcal{O}(\frac{1}{n_t+n_s})$ provided that $ \mathbb{E}_{\mbZnt,\mbZns,\ols{Z}_1^t}\big[ h(\mbZnt,\mbZns,\ols{Z}_1^t)  \big]\leq\mathcal{O}(1)$ with respect to $n_t$ and $\alpha=\frac{n_t}{n_t+n_s}$. Furthermore, for fine-tuning scenario, the representation in \eqref{eq: fine-tune_func} reveals that the convergence rate is at worst $\mathcal{O}(\frac{1}{n_t})$ provided that $\mathbb{E}_{\mbZnt,\mbZns,\ols{Z}_1^t}\big[ K(\mbZnt,\mbZns,\ols{Z}_1^t)\big]\leq O(1)$ with respect to $n_t$.
\end{remark}
\section{KL-regularized Risk Minimization}\label{Sec: KL-reg}
In the last section, we gave general conditions under which we can establish convergence rates for weak transfer generalization error, in terms of the functional derivatives of the loss function and learning algorithms in different scenarios, including, $\alpha$-ERM and fine-tuning. Practically, these conditions are only useful if one can verify that they hold in specific examples for these scenarios. In this section, we will provide an intermediate step in this direction, where we assume the learning algorithm $\mfm$ is chosen to minimize a regularized version of the empirical risk. This will provide us with criteria which can be easily verified in practical examples; we will see this fully in Section \ref{Sec: Application}.

We will frequently use $\mdens$ to represent the density of a measure or measure-valued functional, that is, $\mfm(\nu)(\mrd \theta) = \mdens(\nu;\theta)d\theta$. For probability distributions $m$ and $m'$ with positive densities $m_\theta(\theta)$ and $m'_\theta(\theta)$ we define the Kullback--Leibler divergence $\KLr(m'\|m) =\int \log\big(\frac{m'_\theta(\theta)}{m_\theta(\theta)}\big)m_\theta(\theta)\mrd \theta$, and $\KLr(m'\|m)=\infty$ otherwise\footnote{Extending to the case where $m$ and $m'$ are equivalent but do not have (Lebesgue) densities will not be needed here.}. 
We will also write the Kullback--Leibler divergence in terms of the densities, with the abuse of notation $\KLr(m'\|m) = \KLr(m_\theta'\|m_\theta) = \KLr(m'\|m_\theta)$ as convenient. For our analysis, we define the following integral probability metric.
\begin{definition}
 We define a general integral probability metric\footnote{This is generally a pseudometric, and is a true metric if $\mathfrak{F}$ separates points in $\mathcal{X}$.} (IPM) $\mathrm{d}_{\mathfrak{F}}(\rho,\eta)$, between two measures $\rho,\eta\in\mathcal{P}(\mathcal{X})$ by
    \[\mathrm{d}_{\mathfrak{F}}(\rho,\eta):=\sup_{f\in\mathfrak{F}}\Big|\int_{\mathcal{X}}f(x)\rho(\mathrm{d}x)-\int_{\mathcal{X}}f(x)\eta(\mathrm{d}x) \Big|,\]
    where $\mathfrak{F}\subset \big\{f: \mathcal{X}\mapsto\mathbb{R}^+  \big\}$ are measurable functions.
\end{definition}
\begin{remark}
 Defining $\mathfrak{F}_p^{\mathcal{C}^1}:=\big\{f: \mathcal{X}\mapsto\mathbb{R}^+ \text{ s.t. } f(x)\leq (1+\|x\|^p), f\in \mathcal{C}^1 \big\}$, we have an $L_p$ growth IPM 
 \[\mathrm{d}_{\mathfrak{F}_p}(\rho,\eta):=\sup_{f\in\mathfrak{F}_p^{\mathcal{C}^1}}\Big|\int_{\mathcal{X}}f(x)\rho(\mathrm{d}x)-\int_{\mathcal{X}}f(x)\eta(\mathrm{d}x) \Big|.\]
 Furthermore, considering the set of $1$-Lipschitz functions, $\mathfrak{F}^{Lip}:=\big\{f: \mathcal{X}\mapsto\mathbb{R}^+ \text{ s.t. }  f\in \mathrm{Lip(1)} \big\}$, we recover the Wasserstein $W_1$ distance; taking indicator functions of intervals, or bounded continuous functions, we obtain the  total variation distance, $\mathbb{TV}$.
\end{remark}
In learning, we sometimes consider learning algorithms which minimize a regularized objective $\mathcal{V}^{\beta}$ given by the sum of a risk function $\mrR$ and the KL-divergence:
\begin{align}\label{Eq: regularized risk}
    &\mathcal{V}^{\beta}(m,\nu)= \mrR(m,\nu)+ \frac{\sigma^2}{2\beta^2} \KLr(m\|\gamma^{\sigma}),
\end{align}
where $\gamma^{\sigma}(\theta) = \frac{1}{F^{\sigma}} \exp\big\{-\frac{1}{\sigma^2} U(\theta)\big\}$
is a Gibbs measure which serves as prior to the parameter measure; here $F^\sigma$ is simply a normalizing constant to ensure $\int \gamma^\sigma(\theta)\mrd\theta =1$, and we call $U:\Theta\to \mathbb{R}$ the `regularizing potential'. In \citep{aminian2023mean}, it is shown that under Assumption~\ref{main_ass_square_bounded}, the minimizer of problem \eqref{Eq: regularized risk} exists and is unique,
\begin{align}\label{Eq: Gibbs measure map}
&\mfm^{\beta}( \nu) = \frac{1}{F_{\beta}}\exp\bigg\{-\frac{2\beta^2}{\sigma^2}\Big[\frac{\delta \mrR}{\delta m}(\mfm^{\beta}( \nu) , \nu,\theta) \\\nn&\qquad\qquad+ \frac{1}{2\beta^2} U(\theta)\Big] \bigg\},
\end{align}
for $F_{\beta}$ is a normalization constant. Furthermore, the map $\nu\mapsto \mfm^{\beta}(\nu)$ is $\mathcal{C}^1$ (Definition \ref{def:flatDerivative}). For notational convenience, we similarly define $\tilgammap = \frac{1}{\tilde F^{\sigma}}\exp\big\{-\frac{1}{\sigma^2}U(\theta) + \|\theta\|^8\big\}$. We observe that, unless $\frac{\delta \mrR}{\delta m}$ does not depend on $m$ (i.e.~unless $\mrR$ is linear in $m$), \eqref{Eq: Gibbs measure map} does not provide an explicit representation of $\mfm^{\beta}(\nu)$, but instead describes it implicitly in terms of a fixed point. 

\subsection{\texorpdfstring{$\alpha$}{alpha}-ERM}
In $\alpha$-ERM, we aim to minimize the KL-regularized risk, $\mathcal{V}^{\beta}(m,\nualpha)$, and denote the solution $\mfm^{\beta}(\nualpha)$.

 Applying Theorem~\ref{thm: another rep WTGE-alpha-scenario} under the following assumption (see Appendix \ref{app: alpha erm} for the detailed formulation of this assumption), we can derive upper bound on the WTGE of the $\alpha$-ERM.
\begin{assumption}\label{ass:KLreg_assn}
The loss function and the functional derivative of loss function with respect to $m$, are such that,

 (i) For all $z\in \mcZ$ and $m\in \mathcal{P}_4(\Theta)$, there exists $L_m>0$ such that the loss satisfies
        \[0\leq \ell(m,z) \leq L_m (1+\mbE_{\theta\sim m}[\|\theta\|^4])\big(1+\|z\|^2\big),\]
(ii) For all $z\in \mcZ$ and $m\in \mathcal{P}_4(\Theta)$, there exists $L_e>0$ such that the functional derivative of loss satisfies
    \[\Big|\frac{\delta \ell}{\delta m}(m,z,\theta) \Big| \leq L_e\big(1+\mbE_{\theta\sim m}[\|\theta\|^4]+\|\theta\|^4\big)\big(1+\|z\|^2\big).\]
\end{assumption}

 \begin{theorem}[WTGE of the $\alpha$-ERM]\label{thm: WTGE alpha erm}
    Given Assumption \ref{ass:KLreg_assn}, $\mathbb{E}_{Z\sim\nupop^t}\big[\|Z\|^{8}\big]<\infty$ and $\mathbb{E}_{Z\sim\nupop^s}\big[\|Z\|^{4}\big]<\infty$, the WTGE for the $\alpha$-ERM satisfies 
\[\begin{split}
      &|\mathrm{gen}( \mfm^{\beta}(\nualpha),\nupop^t)|
\\&\quad\leq \frac{c_t\alpha}{n_t}\frac{2\beta^2}{\sigma^2} \comp(\theta)\Bigg[ (2+\alpha)^2\mathbb{E}_{Z_1^t}\Big[(1+\|Z_1^t\|^2)^4\Big]\\&\qquad+  (1-\alpha)^2\mathbb{E}_{Z_1^t}\Big[(1+\|Z_1^t\|^2)^2\Big]\mathbb{E}_{Z_1^s}\Big[(1+\|Z_1^s\|^2)^2\Big]\Bigg] ,
\end{split}\]
where $c_t= \sqrt{2} L_e^2 (1+\alpha L_m)^2(1+\frac{2}{n_t})^2 $ and \[\comp(\theta)=\big[1+ 2 \int_{\Theta}\|\theta\|^8\tilgammap(\mrd \theta) +2\mbE_{\theta\sim\tilgammap}[\|\theta\|^4]\big]^2.\]
 \end{theorem}
 \begin{remark}
     From Theorem~\ref{thm: WTGE alpha erm}, we have bounded WTGE whenever the 8th moment of the target task and 4th moment of the source task population measures are bounded.
 \end{remark}
From Theorem \ref{thm: WTGE alpha erm}, we observe that when $\alpha=1$, which corresponds to not considering the source training dataset, we recover results for supervised learning on the target task.

Using Theorem~\ref{thm: WTGE alpha erm} , we provide upper bounds on the WTER of the $\alpha$-ERM transfer learning scenario.

\begin{theorem}[WTER of $\alpha$-ERM]\label{thm: population alpha erm}
      Under the same assumptions as in Theorem~\ref{thm: WTGE alpha erm}, there exist constants $C_s$, $C_t$, $C_d$ and $C_m$ such that the WTER under $\alpha$-ERM, satisfies
    \[\begin{split}
    &\mathcal{E}\big(\mdens^{\beta}(\nualpha),\nupop^t \big)\leq \frac{C_t\alpha}{n_t}\frac{8\beta^2}{\sigma^2} + \frac{C_s (1-\alpha)}{n_s}\frac{8\beta^2}{\sigma^2}\\&\qquad+(1-\alpha)C_d\Disnd(\nupop^s,\nupop^t)\\    &\qquad+C_m\Dispth(\bar{m}_{\alpha},\bar{m}^t)+\frac{\sigma^2}{2\beta^2} \KLr(\bar{m}_{\alpha}\|\gamma^{\sigma}).\end{split}\]
where $\bar{m}_{\alpha}=\arg\min_{m\in\mathcal{P}_8(\Theta)} R(m,\nualpha)$ and $\bar{m}^{t}=\arg\min_{m\in\mathcal{P}_8(\Theta)} R(m,\nupop^t)$, provided that $\KLr(\bar{m}_{\alpha}\|\gamma^{\sigma})<\infty$ .
\end{theorem}
\begin{remark}\label{rem: true measure alpha} In order to simplify the upper bound in Theorem~\ref{thm: population alpha erm}, we can assume that there exists\footnote{The universal approximation theorem of neural networks \citep{hornik1989multilayer} states that neural networks with an arbitrary number of hidden units can approximate continuous functions on a compact set with any desired precision. Therefore, the existence of true measure, $\bar{m}_{\alpha}$, is reasonable in the mean-field regime. A similar assumption is made in \citet[Theorem 4.5]{chen2020generalized}.} a true measure $\bar{m}_{\alpha}$ such that $ R(\bar{m}_{\alpha},\nualpha)=0$ and $\KLr(\bar{m}_{\alpha}\|\gamma^{\sigma})<\infty$.  Then
 \[\begin{split}
    \mathcal{E}\big(\mdens^{\beta_t}(\nualpha),\nupop^t \big)&\leq \frac{C_t\alpha}{n_t}\frac{8\beta^2}{\sigma^2} + \frac{C_s (1-\alpha)}{n_s}\frac{8\beta^2}{\sigma^2}\\&\quad+(1-\alpha)C_d\Disnd(\nupop^s,\nupop^t)\\&\quad+\frac{\sigma^2}{2\beta^2} \KLr(\bar{m}_{\alpha}\|\gamma^{\sigma}).\end{split}\]
    \end{remark}
\begin{remark}Assuming a bounded loss function, the similarity metric $\Disnd(\nupop^s,\nupop^t)$ can be simplified to total variation distance $\mathbb{TV}(\nupop^s,\nupop^t)$.
\end{remark}
\subsection{Fine-tuning}
There are two different steps in fine-tuning approach. In the first step, we solve the KL-regularized risk minimization problem on the source task, \begin{equation}\label{eq: fine-tune kl first}
     \begin{split}
      \mathcal{V}^{\beta_s}(m,\nus):= \mrR(m,\nut)+ \frac{\sigma^2}{2\beta_t^2} \KLr(m\|\gamma^{\sigma}),
       \end{split}
    \end{equation}
    where the solution is $\mfm_{s}^{\beta_s}(\nus)$ over both common and specific parameters for the source task, $(\thetac,\thetas^s)$.
    In the second step, we fix the measure over the common parameters, $\mfm_{\mathrm{c}}^{\beta_s}(\nus):=\int_{\Theta_{\spec}} \mfm_{s}^{\beta_s}(\nus)(\mrd \theta_{\spec}^s) $, and solve the following KL-regularized with respect to $m_{\spec}^t$,
     \begin{equation}\label{eq: fine-tune kl second}
     \begin{split}
      &\mathcal{V}^{\beta_t}(\mfm_{\mathrm{c}}^{\beta_s}(\nus)m_{\spec}^t,\nut)\\&:= \mrR(\mfm_{\mathrm{c}}^{\beta_s}(\nus)m_{\spec}^t,\nut)+ \frac{\sigma^2}{2\beta_t^2} \KLr(m_{\spec}^t\|\gamma_{\thetas}^{\sigma}),
       \end{split}
    \end{equation}
    with solution $\mfm_{\spec}^{t,\beta_t}(\mfm_{\mathrm{c}}(\nus),\nut)$.
    For notational convenience, we  define  $\tilgammasp = \frac{1}{\tilde F^{\sigma,\spec}}\exp\big\{-\frac{1}{\sigma^2}U(\thetas^t) + \|\thetas^t\|^8\big\}$ and $\hat\gamma_{8}^\sigma = \frac{1}{\tilde F^{\sigma}}\exp\big\{-\frac{1}{\sigma^2}U(\theta) + \|\thetac\|^8\big\}$.
Applying Theorem~\ref{thm: another rep fine-tune} under the following assumption (see Appendix \ref{app: fine tune} for the detailed formulation of this assumption), we can drive upper bounds on the WTGE and WTER of fine-tuning scenario. 
\begin{assumption}\label{ass:KLreg_assnf}
The loss function and the functional derivative of loss function with respect to $m_{\spec}$, are such that,

 (i) For all $z\in \mcZ$, $m_{\mathrm{c}}\in \mathcal{P}_8(\Thetac)$ and $m_{\spec}\in\mathcal{P}_8(\thetas)$, there exists $L_m$ such that the loss satisfies
        \[ \ell(m_{\mathrm{c}}m_{\spec},z)  \leq  g(m_{\mathrm{c}}m_{\spec})\big(1+\|z\|^2\big),\]
where $g(m_{\mathrm{c}}m_{\spec})=L_m(1+\mbE_{\thetac\sim m_{\mathrm{c}}}[\|\thetac\|^4]+\mbE_{\thetas\sim m_{\spec}}[\|\thetas\|^4]).$

(ii)  For all $z\in \mcZ$, $m_{\mathrm{c}}\in \mathcal{P}_8(\Thetac)$ and $m_{\spec}\in\mathcal{P}_8(\thetas)$, there exists $L_e$ such that the functional derivative of loss function with respect to $m_{\spec}$ satisfies
        \[\Big|\frac{\delta \ell}{\delta m_{\spec}}(m_\mathrm{c}m_{\spec},z,\thetas) \Big| \leq g_{\spec}(m_\mathrm{c}m_{\spec},\thetas)\big(1+\|z\|^2\big);\]
        where $g_{\spec}(m_\mathrm{c}m_{\spec},\thetas)=L_e(1+\mbE_{\thetac\sim m_{\mathrm{c}}}[\|\thetac\|^4]+\mbE_{\thetas\sim m_{\spec}}[\|\thetas\|^4]+\|\thetas\|^4)$.
\end{assumption}
 
\begin{theorem}[WTGE of fine-tuning]\label{thm: WTGE fine tune}
    Given Assumption \ref{ass:KLreg_assnf}, $\mathbb{E}_{Z\sim\nupop^t}\big[\|Z\|^{8}\big]<\infty$ and $\mathbb{E}_{Z\sim\nupop^s}\big[\|Z\|^{4}\big]<\infty$, the WTGE under fine-tuning satisfies 
\[ \begin{split} &|\mathrm{gen}(\mfm_{\mathrm{c}}(\nus)\mfm_{\spec}^{t,\beta_t}(\mfm_{\mathrm{c}}(\nus),\nut),\nupop^t)| \\
&\leq  \frac{2}{n_t}(1+\frac{2}{n_t})^2\frac{16\beta^2}{\sigma^2}  L_e^2(1+L_m)^2 \comp(\thetac,\thetas^t,\thetas^s)
 \\
 & \times \mathbb{E}_{Z_1^s}\Big[(1+\|Z_1^s\|^2)^2\Big] \Big[\mathbb{E}_{Z_1^t}\Big[(1+\|Z_1^t\|^2)^4\Big]\Big].\end{split} \]
 where 
\[\begin{split}
\comp(\thetac&,\thetas^t,\thetas^s)=\Big[1+\int_{\Thetac} \|\thetac\|^4(2+\|\thetac\|^4)\gammac(\mrd\thetac)\\
&\quad+\int_{\Thetas}\|\thetas^t\|^4(1+ 2\|\thetas^t\|^4)\tilgammasp(\mrd \thetas^t)\\
  &\quad+\int_{\Thetas}\|\thetas^s\|^4\gammasp(\mrd \thetas^s)\Big]^2,
  \end{split}\]
 with $\gammac = \int_{\Thetas}\hat\gamma_{8}^{\sigma}(\mrd \thetas^s)$ and $\hat\gamma_{8,\mathrm{sp}}^\sigma = \int_{\Thetac}\hat\gamma_{8}^{\sigma}(\mrd \thetac)$.
 \end{theorem}

Using Theorem~\ref{thm: WTGE fine tune}, we can provide upper bound on the excess risk of fine-tune transfer learning scenario.

\begin{theorem}[WTER of fine-tuning]\label{thm: excess-fine-tune}
    Under the same assumptions in Theorem~\ref{thm: WTGE fine tune} and Assumption~\ref{ass:KLreg_assn}, there exist constants $C_t,C_s,C_{d_2}, C_{d_p}$ and $C_{\dims}$ where the following upper bound holds on WTER under fine-tuning scenario,
    \begin{align}
           &\mathcal{E}(\mfmc^{\beta_s}(\nus)\mfms^{\beta_t}(\mfmc^{\beta_s}(\nus),\nut),\nupop^t)\\\nn
           &\leq \frac{C_t}{n_t}\frac{\beta_t^2}{\sigma^2} +\frac{\sigma^2}{2\beta_t^2}\KLr(\tilde{m}_{\spec}^t\| \tilde\gamma_{\spec}^{\sigma})+\frac{C_s}{n_s}\frac{\beta_s^2}{\sigma^2} \\\nn&\quad +\frac{\sigma^2}{2\beta_s^2}\KLr(\mrmbarc^s\otimes\mrmbarsp^s\|\gamma_c^\sigma\otimes\gamma_{\spec}^{\sigma})\\\nn
&\quad+C_{d_2}\Disnd(\nupop^s,\nupop^t)+C_{d_\spec}\Disfth(\tilde{m}_{\spec}^s,\tilde{m}_{\spec}^t)\\\nn
 &\quad+C_{d_p}\Disfth(\mrmbarc^{s}\otimes\mrmbarsp^{s},\mrmbarc^{t}\otimes\mrmbarsp^{t}) ,
     \end{align}
where $\tilde{m}_{\spec}^t=\arg\min_{m_{\spec}\in\mathcal{P}_8(\Thetas)} R(\mfmc^{\beta_s}(\nus) m_{\spec},\nupop^t)$ provided $\KLr(\tilde{m}_{\spec}^t\| \tilde\gamma_{\spec}^{\sigma})<\infty$, with minimizers $\tilde{m}_{\spec}^s=\arg\min_{m_{\spec}\in\mathcal{P}_8(\Thetas)} R(\mfmc^{\beta_s}(\nus) m_{\spec},\nupop^s)$, $\mrmbarc^s\otimes\mrmbarsp^s=\arg\min_{\mrmbarc^s\otimes\mrmbarsp^s\in\mathcal{P}_8(\Thetac)\times \mathcal{P}_8(\Thetas)} R(\mrmbarc^s\otimes\mrmbarsp^s,\nupop^s)$ and $\mrmbarc^t\otimes\mrmbarsp^t=\arg\min_{\mrmbarc^t\otimes\mrmbarsp^t\in\mathcal{P}_8(\Thetac)\times \mathcal{P}_8(\Thetas)} R(\mrmbarc^t\otimes\mrmbarsp^t,\nupop^t)$. 
\end{theorem}
\begin{remark}\label{rem: true measure fine-tune}
    To simplify the upper bound in Theorem~\ref{thm: excess-fine-tune},  we can assume that there exist $\mrmbarc^s\otimes\mrmbarsp^s\in\mathcal{P}_8(\Thetac)\times \mathcal{P}_8(\Thetas)$ and $\mrmbarc^t\otimes\mrmbarsp^t\in\mathcal{P}_8(\Thetac)\times \mathcal{P}_8(\Thetas)$ such that $R(\mrmbarc^s\otimes\mrmbarsp^s,\nupop^s)=0$ and $R(\mrmbarc^t\otimes\mrmbarsp^t,\nupop^t)=0$, respectively. In addition, assume that there exist $\tilde{m}_{\spec}^t\in\mathcal{P}_8(\Thetas)$ and $\tilde{m}_{\spec}^s\in\mathcal{P}_8(\Thetas)$ such that $ \mathbb{E}_{\mbZns}[R(\mfmc^{\beta_s}(\nus) \tilde{m}_{\spec}^s,\nupop^s)]=0$ and $\mathbb{E}_{\mbZns}[R(\mfmc^{\beta_s}(\nus) \tilde{m}_{\spec}^t,\nupop^t)]=0$. Then, we have,
 \begin{equation}
       \begin{split}
           &\mathcal{E}(\mfmc^{\beta_s}(\nus)\mfms^{\beta_t}(\mfmc^{\beta_s}(\nus),\nut),\nupop^t)\\
           &\leq \frac{C_t}{n_t}\frac{\beta_t^2}{\sigma^2} +\frac{\sigma^2}{2\beta_t^2}\KLr(\tilde{m}_{\spec}^t\| \tilde\gamma_{\spec}^{\sigma})+\frac{C_s}{n_s}\frac{\beta_s^2}{\sigma^2} \\&\qquad+\frac{\sigma^2}{2\beta_s^2}\KLr(\mrmbarc^s\otimes\mrmbarsp^s\|\gamma_c^\sigma\otimes\gamma_{\spec}^{\sigma})\\&\qquad+C_{d_2}\Disnd(\nupop^s,\nupop^t).
     \end{split} 
   \end{equation}
\end{remark}

\subsection{Discussion} 

 \textbf{$\alpha$-ERM:} As shown in Theorem~\ref{thm: WTGE alpha erm}, the WTER depends on the similarity metric $\Disnd(\nupop^s,\nupop^t)$. If source and the target tasks are similar, i.e. $\Disnd(\nupop^s,\nupop^t)\approx 0$, then, choosing $\alpha=\frac{n_t}{n_s+n_t}$ and $\beta=(n_t+n_s)^{1/4}$ results in a convergence rate of $O(1/\sqrt{n_t+n_s})$ for the WTER of the $\alpha$-ERM scenario. However, for non-similar  source and the target tasks, i.e.~$\Disnd(\nupop^s,\nupop^t)>0$, choosing $\alpha=\frac{n_t}{n_s+n_t}$ results in a rate of $\max(\frac{1}{\sqrt{n_t+n_s}},\frac{ n_s}{n_s+n_t})$, which can be worse for large $n_s$. Therefore, for non-similar source and the target tasks, i.e., $\Disnd(\nupop^s,\nupop^t)>0$, by choosing $\alpha=1-n_t^{-1}$ and $\beta^2=\big(\frac{\alpha}{n_t}+\frac{(1-\alpha)}{n_s}\big)^{-1/2}$, we can get an overall convergence rate of $O(\frac{1}{\sqrt{n_t}})$ which is similar to supervised learning scenario. 

\textbf{Fine-tuning:} If source and the target tasks are similar, i.e. $\Disnd(\nupop^s,\nupop^t)\approx 0$, in fine-tuning scenario, we have a convergence rate $O(\frac{1}{\sqrt{n_t}}+\frac{1}{\sqrt{n_s}})$ by choosing $\beta_s^2=\sqrt{n_s}$ and $\beta_t^2=\sqrt{n_t}$.

In Table~\ref{tab:comparison}, more details are provided for comparison of the supervised learning scenario where the source task training dataset is not available, and transfer learning scenarios, including $\alpha$-ERM and fine-tuning. 
 \begin{table}[htb]

  \caption{Comparison of different algorithms  setting under similar tasks ($\Disnd(\nupop^s,\nupop^t)\approx 0$).}
  \label{tab:comparison}
  \centering
  \scriptsize
	\begin{tabular}{cccc}
	\\
    	\toprule
    	 & \textbf{Supervised Learning} & \textbf{$\alpha$-ERM}  & \textbf{Fine-tuning} \\
    	\midrule

    	Generalization Error&$O\Big(\frac{1}{n_t}\Big)$ & $O\Big(\frac{1}{n_s+n_t}\Big)$ & $O\Big(\frac{1}{n_t}\Big)$\\

    	Excess risk &$O\Big(\frac{1}{\sqrt{n_t}}\Big)$ & $O\Big(\frac{1}{\sqrt{n_s+n_t}}\Big)$ & $O\Big(\frac{1}{\sqrt{n_t}}+\frac{1}{\sqrt{n_s}}\Big)$ \\

    	\bottomrule
  \end{tabular}
  \normalsize
\end{table}

Further discussion is provided in Appendix~\ref{app: comparison}.

\textbf{Comparison to previous works:} In \citet{bu2022characterizing}, the authors provide a maximum likelihood estimation analysis of the WTER for $\alpha$-ERM and fine-tuning under asymptotic assumptions, i.e., $n_s\rightarrow\infty$ and $n_t\rightarrow\infty$. However, our result holds for finite number of samples. In addition, their results focus on expected loss, where the loss function is linear in parameter measure. Furthermore, \citet{bu2022characterizing} lacks a clear definition of similarity between target and source tasks, making it difficult to determine when transfer learning is beneficial. A Gaussian complexity analysis of fine-tuning is proposed by \citet{tripuraneni2020theory}, which cannot be extended to the mean-field regime. Furthermore, the convergence rate under Gaussian complexity analysis is worse than our approach. In comparison to \citet{hanneke2019value}, we can apply our method to unbounded loss functions in the over-parameterized regime. 

\section{Application}\label{Sec: Application}

We highlight the application of our model to an overparameterized one-hidden-layer neural network in the mean-field regime~\citep{MFLD,mei/montanari/nguyen:2018} for transfer learning scenarios, $\alpha$-ERM and fine-tuning.
The loss is represented in \eqref{eq: nn loss function}. For simplicity, we also write $(\hat y,y)\mapsto \ell_o(\hat y,y)$ for the loss function, where $\hat y = \mathbb{E}_{\theta\sim m}[\phi(\theta,x)]$ is the output of the NN, where $\phi(\theta,x)=a\varphi(w.x)$ for $a\in\mbR$ and $w\in\mbR^q$.

For $\alpha$-ERM scenario, we assume that the risk function $\mrR$ is given by
$$ \mrR(\mfm_{\alpha}(\nualpha),\nupop^t)= \mathbb{E}_{(X,Y) \sim \nupop^t}[\ell_o(\mathbb{E}_{\theta\sim \mfm^{\beta}(\nualpha)}[\phi(\theta,X)],Y)],$$ as is done in mean-field models of one-hidden-layer neural networks \citep{MFLD,mei/misiakiewicz/montanari:2019,tzen/raginsky:2020}, and is KL-regularized as seen previously. Similarly, for fine-tuning scenario, we assume that the risk function in second step is given by $$ \mrR(\mfm_{\mathrm{c}}(\nus)\mfm_{\spec}^t(\mfm_{\mathrm{c}}(\nus), \nut),\nupop^t)= \mathbb{E}_{(X,Y) \sim \nupop^t}[\ell_o(Y',Y)],$$ 
 where $Y'=\mathbb{E}_{w\sim\mfm_\mathrm{c}(\nus) }[\varphi(wX)]\mbE_{a\sim \mfm_{\spec}^t(\mfm_{\mathrm{c}}(\nus), \nut)}[a]$.  

 We make the following assumptions to investigate weak transfer generalization performance.
\begin{assumption}\label{ass:NN_KL_alpha}
Suppose that the regularizing potential $U$ satisfies $\lim_{\|\theta\|\to \infty}\frac{U(\theta)}{\|\theta\|^4} = \infty,$ the loss $\ell_o$ is convex and nonnegative, and  there exist finite constants  $L_\ell, L_{\ell,1}, L_{\ell,2}$, and $L_\phi$ such that, for all $(x,y)\in\mathcal{X}\times \mathcal{Y}$ and $\theta\in\Theta$, 
         \begin{equation}\begin{split}
          \big|\ell_o(\hat y, y)\big| &\leq L_\ell(1+\|\hat y\|^2 + \|y\|^2),\\
          |\partial_{\hat{y}}\ell_o(\hat y,y)\big|&\leq L_{\ell,1}(1+\|\hat{y}\|+\|y\|),\\
         \big|\partial_{\hat{y}\hat{y}}\ell_o(\hat y,y)\big|&\leq L_{\ell,2},\\
         \big|\varphi(w.x)\big|&\leq L_{\varphi}(1+\|x\|)(1+\|w\|).
           \end{split}\end{equation}       
\end{assumption}
\vspace{-1em}
As shown in Appendix~\ref{app: Sec: Application}, Assumption~\ref{ass:KLreg_assn} and Assumption~\ref{ass:KLreg_assnf} are verified for overparameterized one-hidden layer neural network under Assumption~\ref{ass:NN_KL_alpha}. Therefore, Theorems \ref{thm: WTGE alpha erm}, \ref{thm: population alpha erm}, \ref{thm: WTGE fine tune} and \ref{thm: excess-fine-tune} yield upper bounds on WTGE and WTER under $\alpha$-ERM and fine-tuning algorithms.

\textbf{Activation and Loss Functions:} If the activation function is of linear growth, then Assumption \ref{ass:NN_KL_alpha} is satisfied. Note that the ReLU, Heaviside unit-step, tanh and sigmoid activation functions are of this type, and no smoothness assumption is needed. For loss functions, twice differentiable loss functions which are of quadratic growth with Lipschitz derivatives  satisfy Assumption~\ref{ass:NN_KL_alpha}. For example, the quadratic, product-type (as in \citep{bartlett2006convexity}), and logcosh loss functions~\citep{wang2022comprehensive} are of this type. 

\section{Conclusions and Future Works}
\label{sec_conclusions}
Our study introduces a novel framework for analyzing the generalization error of risk functions and excess risk in transfer learning. The framework utilizes calculus on the space of probability measures, which allows us to gain a deeper understanding of the factors that influence the generalization of machine learning models in transfer learning scenarios, $\alpha$-ERM and fine-tuning. We demonstrate the efficacy of our framework by applying it to over-parameterized one-hidden layer neural network in mean-field regime. 

Given that mean-field analysis is applicable to single-hidden layer neural networks, we aim to extend our current work to multiple-layer neural networks in future research based on \citep{sirignano/spiliopoulos:2019}, \citet{chizat2022infinite}, \citet{jabir2019mean} and \citet{geshkovski2023emergence}.

  \section*{Acknowledgements} Authors acknowledge the support of the UKRI Prosperity Partnership Scheme (FAIR) under EPSRC Grant EP/V056883/1 and the Alan Turing Institute. Samuel N. Cohen  and {\L}ukasz Szpruch also acknowledge the support of the Oxford--Man Institute for Quantitative Finance.
    





\bibliographystyle{plainnat}
\bibliography{Refs}

\begin{thebibliography}{45}
\providecommand{\natexlab}[1]{#1}
\providecommand{\url}[1]{\texttt{#1}}
\expandafter\ifx\csname urlstyle\endcsname\relax
  \providecommand{\doi}[1]{doi: #1}\else
  \providecommand{\doi}{doi: \begingroup \urlstyle{rm}\Url}\fi

\bibitem[Allen-Zhu et~al.(2019)Allen-Zhu, Li, and Liang]{allen2019learning}
Zeyuan Allen-Zhu, Yuanzhi Li, and Yingyu Liang.
\newblock Learning and generalization in overparameterized neural networks, going beyond two layers.
\newblock \emph{Advances in neural information processing systems}, 32, 2019.

\bibitem[Aminian et~al.(2023)Aminian, Cohen, and Szpruch]{aminian2023mean}
Gholamali Aminian, Samuel~N Cohen, and {\L}ukasz Szpruch.
\newblock Mean-field analysis of generalization errors.
\newblock \emph{arXiv preprint arXiv:2306.11623}, 2023.

\bibitem[Aminian et~al.(2024)Aminian, He, Reinert, Szpruch, and Cohen]{aminian2024generalization}
Gholamali Aminian, Yixuan He, Gesine Reinert, {\L}ukasz Szpruch, and Samuel~N Cohen.
\newblock Generalization error of graph neural networks in the mean-field regime.
\newblock \emph{ICML}, 2024.

\bibitem[Arora et~al.(2019{\natexlab{a}})Arora, Du, Hu, Li, and Wang]{arora2019fine}
Sanjeev Arora, Simon Du, Wei Hu, Zhiyuan Li, and Ruosong Wang.
\newblock Fine-grained analysis of optimization and generalization for overparameterized two-layer neural networks.
\newblock In \emph{International Conference on Machine Learning}, pages 322--332. PMLR, 2019{\natexlab{a}}.

\bibitem[Arora et~al.(2019{\natexlab{b}})Arora, Du, Hu, Li, Salakhutdinov, and Wang]{arora2019exact}
Sanjeev Arora, Simon~S Du, Wei Hu, Zhiyuan Li, Ruslan Salakhutdinov, and Ruosong Wang.
\newblock On exact computation with an infinitely wide neural net.
\newblock In \emph{Advances in Neural Information Processing Systems}, 2019{\natexlab{b}}.

\bibitem[Bartlett et~al.(2006)Bartlett, Jordan, and McAuliffe]{bartlett2006convexity}
Peter~L Bartlett, Michael~I Jordan, and Jon~D McAuliffe.
\newblock Convexity, classification, and risk bounds.
\newblock \emph{Journal of the American Statistical Association}, 101\penalty0 (473):\penalty0 138--156, 2006.

\bibitem[Ben-David et~al.(2007)Ben-David, Blitzer, Crammer, Pereira, et~al.]{ben2007analysis}
Shai Ben-David, John Blitzer, Koby Crammer, Fernando Pereira, et~al.
\newblock Analysis of representations for domain adaptation.
\newblock \emph{Advances in neural information processing systems}, 19:\penalty0 137, 2007.

\bibitem[Ben-David et~al.(2010)Ben-David, Blitzer, Crammer, Kulesza, Pereira, and Vaughan]{ben2010theory}
Shai Ben-David, John Blitzer, Koby Crammer, Alex Kulesza, Fernando Pereira, and Jennifer~Wortman Vaughan.
\newblock A theory of learning from different domains.
\newblock \emph{Machine learning}, 79\penalty0 (1):\penalty0 151--175, 2010.

\bibitem[Bousquet and Elisseeff(2002)]{Bousquet-Elisseeff}
Olivier Bousquet and Andr\'{e} Elisseeff.
\newblock Stability and generalization.
\newblock \emph{J. Mach. Learn. Res.}, 2:\penalty0 499–526, March 2002.
\newblock ISSN 1532-4435.
\newblock \doi{10.1162/153244302760200704}.
\newblock URL \url{https://doi.org/10.1162/153244302760200704}.

\bibitem[Bu et~al.(2022)Bu, Aminian, Toni, Wornell, and Rodrigues]{bu2022characterizing}
Yuheng Bu, Gholamali Aminian, Laura Toni, Gregory~W Wornell, and Miguel Rodrigues.
\newblock Characterizing and understanding the generalization error of transfer learning with gibbs algorithm.
\newblock In \emph{International Conference on Artificial Intelligence and Statistics}, pages 8673--8699. PMLR, 2022.

\bibitem[Cao and Gu(2019)]{cao2019generalization}
Yuan Cao and Quanquan Gu.
\newblock Generalization bounds of stochastic gradient descent for wide and deep neural networks.
\newblock \emph{Advances in neural information processing systems}, 32, 2019.

\bibitem[Cardaliaguet et~al.(2019)Cardaliaguet, Delarue, Lasry, and Lions]{cardaliaguet2019master}
Pierre Cardaliaguet, Fran{\c{c}}ois Delarue, Jean-Michel Lasry, and Pierre-Louis Lions.
\newblock \emph{The master equation and the convergence problem in mean field games}.
\newblock Princeton University Press, 2019.

\bibitem[Carmona and Delarue(2018)]{carmona/delarue:2018}
Ren\'e Carmona and Fran\c{c}ois Delarue.
\newblock \emph{Probabilistic Theory of Mean Field Games with Applications I}.
\newblock Springer, 2018.

\bibitem[Chen et~al.(2019)Chen, Cao, Zou, and Gu]{chenmuch}
Zixiang Chen, Yuan Cao, Difan Zou, and Quanquan Gu.
\newblock How much over-parameterization is sufficient to learn deep relu networks?
\newblock In \emph{International Conference on Learning Representations}, 2019.

\bibitem[Chen et~al.(2020)Chen, Cao, Gu, and Zhang]{chen2020generalized}
Zixiang Chen, Yuan Cao, Quanquan Gu, and Tong Zhang.
\newblock A generalized neural tangent kernel analysis for two-layer neural networks.
\newblock \emph{Advances in Neural Information Processing Systems}, 33:\penalty0 13363--13373, 2020.

\bibitem[Chizat and Bach(2018)]{chizat/bach:2018}
Lenaic Chizat and Francis Bach.
\newblock On the global convergence of gradient descent for over-parameterized models using optimal transport.
\newblock \emph{Advances in neural information processing systems}, 31, 2018.

\bibitem[Chizat et~al.(2022)Chizat, Colombo, Fern{\'a}ndez-Real, and Figalli]{chizat2022infinite}
L{\'e}na{\"\i}c Chizat, Maria Colombo, Xavier Fern{\'a}ndez-Real, and Alessio Figalli.
\newblock Infinite-width limit of deep linear neural networks.
\newblock \emph{arXiv preprint arXiv:2211.16980}, 2022.

\bibitem[Ding et~al.(2023)Ding, Qin, Yang, Wei, Yang, Su, Hu, Chen, Chan, Chen, et~al.]{ding2023parameter}
Ning Ding, Yujia Qin, Guang Yang, Fuchao Wei, Zonghan Yang, Yusheng Su, Shengding Hu, Yulin Chen, Chi-Min Chan, Weize Chen, et~al.
\newblock Parameter-efficient fine-tuning of large-scale pre-trained language models.
\newblock \emph{Nature Machine Intelligence}, 5\penalty0 (3):\penalty0 220--235, 2023.

\bibitem[Fang et~al.(2019)Fang, Dong, and Zhang]{fang2019over}
Cong Fang, Hanze Dong, and Tong Zhang.
\newblock Over parameterized two-level neural networks can learn near optimal feature representations.
\newblock \emph{arXiv preprint arXiv:1910.11508}, 2019.

\bibitem[Fang et~al.(2021)Fang, Dong, and Zhang]{fang2021mathematical}
Cong Fang, Hanze Dong, and Tong Zhang.
\newblock Mathematical models of overparameterized neural networks.
\newblock \emph{Proceedings of the IEEE}, 109\penalty0 (5):\penalty0 683--703, 2021.

\bibitem[Geshkovski et~al.(2023)Geshkovski, Letrouit, Polyanskiy, and Rigollet]{geshkovski2023emergence}
Borjan Geshkovski, Cyril Letrouit, Yury Polyanskiy, and Philippe Rigollet.
\newblock The emergence of clusters in self-attention dynamics.
\newblock \emph{arXiv preprint arXiv:2305.05465}, 2023.

\bibitem[Hanneke and Kpotufe(2019)]{hanneke2019value}
Steve Hanneke and Samory Kpotufe.
\newblock On the value of target data in transfer learning.
\newblock \emph{Advances in Neural Information Processing Systems}, 32:\penalty0 9871--9881, 2019.

\bibitem[Hornik et~al.(1989)Hornik, Stinchcombe, and White]{hornik1989multilayer}
Kurt Hornik, Maxwell Stinchcombe, and Halbert White.
\newblock Multilayer feedforward networks are universal approximators.
\newblock \emph{Neural networks}, 2\penalty0 (5):\penalty0 359--366, 1989.

\bibitem[Hu et~al.(2020)Hu, Ren, \v{S}i\v{s}ka, and Szpruch]{MFLD}
Kaitong Hu, Zhenjie Ren, David \v{S}i\v{s}ka, and {\L}ukasz Szpruch.
\newblock Mean-field langevin dynamics and energy landscape of neural networks, 2020.

\bibitem[Jabir et~al.(2019)Jabir, {\v{S}}i{\v{s}}ka, and Szpruch]{jabir2019mean}
Jean-Fran{\c{c}}ois Jabir, David {\v{S}}i{\v{s}}ka, and {\L}ukasz Szpruch.
\newblock Mean-field neural odes via relaxed optimal control.
\newblock \emph{arXiv preprint arXiv:1912.05475}, 2019.

\bibitem[Jacot et~al.(2018)Jacot, Gabriel, and Hongler]{jacot2018neural}
Arthur Jacot, Franck Gabriel, and Cl{\'e}ment Hongler.
\newblock Neural tangent kernel: Convergence and generalization in neural networks.
\newblock \emph{Advances in neural information processing systems}, 31, 2018.

\bibitem[Ji and Telgarsky(2020)]{ji2019polylogarithmic}
Ziwei Ji and Matus Telgarsky.
\newblock Polylogarithmic width suffices for gradient descent to achieve arbitrarily small test error with shallow relu networks.
\newblock In \emph{International Conference on Learning Representations}, 2020.

\bibitem[Kalan and Fabian(2020)]{kalan2020minimax}
MM~Kalan and Z~Fabian.
\newblock Minimax lower bounds for transfer learning with linear and one-hidden layer neural networks.
\newblock \emph{Neural Information Processing Systems (NeuRIPS 2020)}, 2020.

\bibitem[Li et~al.(2012)Li, Zhao, and Wang]{li2012human}
Wei Li, Rui Zhao, and Xiaogang Wang.
\newblock Human reidentification with transferred metric learning.
\newblock In \emph{Asian conference on computer vision}, pages 31--44. Springer, 2012.

\bibitem[Li and Liang(2018)]{li2018learning}
Yuanzhi Li and Yingyu Liang.
\newblock Learning overparameterized neural networks via stochastic gradient descent on structured data.
\newblock \emph{Advances in neural information processing systems}, 31, 2018.

\bibitem[Long et~al.(2015)Long, Cao, Wang, and Jordan]{long2015learning}
Mingsheng Long, Yue Cao, Jianmin Wang, and Michael Jordan.
\newblock Learning transferable features with deep adaptation networks.
\newblock In \emph{International conference on machine learning}, pages 97--105. PMLR, 2015.

\bibitem[Ma et~al.(2019)Ma, Wu, et~al.]{ma2019generalization}
Chao Ma, Lei Wu, et~al.
\newblock The generalization error of the minimum-norm solutions for over-parameterized neural networks.
\newblock \emph{arXiv preprint arXiv:1912.06987}, 2019.

\bibitem[Mei and Montanari(2022)]{mei2022generalization}
Song Mei and Andrea Montanari.
\newblock The generalization error of random features regression: Precise asymptotics and the double descent curve.
\newblock \emph{Communications on Pure and Applied Mathematics}, 75\penalty0 (4):\penalty0 667--766, 2022.

\bibitem[Mei et~al.(2018)Mei, Montanari, and Nguyen]{mei/montanari/nguyen:2018}
Song Mei, Andrea Montanari, and Phan-Minh Nguyen.
\newblock A mean field view of the landscape of two-layer neural networks.
\newblock \emph{Proceedings of the National Academy of Sciences}, 115\penalty0 (33):\penalty0 E7665--E7671, 2018.

\bibitem[Mei et~al.(2019)Mei, Misiakiewicz, and Montanari]{mei/misiakiewicz/montanari:2019}
Song Mei, Theodor Misiakiewicz, and Andrea Montanari.
\newblock Mean-field theory of two-layers neural networks: dimension-free bounds and kernel limit, 2019.

\bibitem[Nishikawa et~al.(2022)Nishikawa, Suzuki, Nitanda, and Wu]{nishikawatwo}
Naoki Nishikawa, Taiji Suzuki, Atsushi Nitanda, and Denny Wu.
\newblock Two-layer neural network on infinite dimensional data: global optimization guarantee in the mean-field regime.
\newblock In \emph{Advances in Neural Information Processing Systems}, 2022.

\bibitem[Nitanda et~al.(2021)Nitanda, Wu, and Suzuki]{nitanda2021particle}
Atsushi Nitanda, Denny Wu, and Taiji Suzuki.
\newblock Particle dual averaging: Optimization of mean field neural network with global convergence rate analysis.
\newblock \emph{Advances in Neural Information Processing Systems}, 34:\penalty0 19608--19621, 2021.

\bibitem[Rahimi and Recht(2008)]{rahimi2008uniform}
Ali Rahimi and Benjamin Recht.
\newblock Uniform approximation of functions with random bases.
\newblock In \emph{2008 46th Annual Allerton Conference on Communication, Control, and Computing}, pages 555--561. IEEE, 2008.

\bibitem[Sirignano and Spiliopoulos(2022)]{sirignano/spiliopoulos:2019}
Justin Sirignano and Konstantinos Spiliopoulos.
\newblock Mean field analysis of deep neural networks.
\newblock \emph{Mathematics of Operations Research}, 47\penalty0 (1):\penalty0 120--152, 2022.

\bibitem[Tripuraneni et~al.(2020)Tripuraneni, Jordan, and Jin]{tripuraneni2020theory}
Nilesh Tripuraneni, Michael Jordan, and Chi Jin.
\newblock On the theory of transfer learning: The importance of task diversity.
\newblock \emph{Advances in Neural Information Processing Systems}, 33, 2020.

\bibitem[Tzen and Raginsky(2020)]{tzen/raginsky:2020}
Belinda Tzen and Maxim Raginsky.
\newblock A mean-field theory of lazy training in two-layer neural nets: entropic regularization and controlled {M}c{K}ean--{V}lasov dynamics.
\newblock \emph{arXiv preprint arXiv:2002.01987}, 2020.

\bibitem[Wang et~al.(2019)Wang, Mendez, Cai, and Eaton]{wang2019transfer}
Boyu Wang, Jorge Mendez, Mingbo Cai, and Eric Eaton.
\newblock Transfer learning via minimizing the performance gap between domains.
\newblock \emph{Advances in Neural Information Processing Systems}, 32:\penalty0 10645--10655, 2019.

\bibitem[Wang et~al.(2022)Wang, Ma, Zhao, and Tian]{wang2022comprehensive}
Qi~Wang, Yue Ma, Kun Zhao, and Yingjie Tian.
\newblock A comprehensive survey of loss functions in machine learning.
\newblock \emph{Annals of Data Science}, 9\penalty0 (2):\penalty0 187--212, 2022.

\bibitem[Wu et~al.(2020)Wu, Manton, Aickelin, and Zhu]{wu2020information}
Xuetong Wu, Jonathan~H Manton, Uwe Aickelin, and Jingge Zhu.
\newblock Information-theoretic analysis for transfer learning.
\newblock In \emph{2020 IEEE International Symposium on Information Theory (ISIT)}, pages 2819--2824. IEEE, 2020.

\bibitem[Yosinski et~al.(2014)Yosinski, Clune, Bengio, and Lipson]{yosinski2014transferable}
Jason Yosinski, Jeff Clune, Yoshua Bengio, and Hod Lipson.
\newblock How transferable are features in deep neural networks?
\newblock \emph{arXiv preprint arXiv:1411.1792}, 2014.

\end{thebibliography}
\clearpage
\appendix

\section{Other related works}\label{appendix:related work}

\paragraph{Generalization error and overparameterization:}
There are three main approaches to analyzing learning problems in an overparameterized regime: the neural tangent kernel (NTK), random feature and mean-field approaches. Under some assumptions, the NTK approach (a.k.a. lazy training)  shows that an overparameterized one-hidden-layer NN converges to an (infinite dimensional) linear model. The neural tangent random feature approach is similar to NTK, where the model is defined based on the network gradients at the initialization.
The random feature model is similar to NTK, with an extra assumption of constant weights in the single hidden layer of the NN. The mean-field approach uses the exchangeability of neurons to work with distributions of parameters. The study of these methods' generalization performance allows us to extend our understanding of the overparameterized regime. See, for example, for NTK ~\citep{li2018learning,cao2019generalization,arora2019exact,arora2019fine,allen2019learning,ji2019polylogarithmic}, neural tangent random feature \citep{cao2019generalization,chenmuch}, random feature~\citep{mei2022generalization,ma2019generalization} and mean-field \citep{nitanda2021particle,nishikawatwo} settings. The neural tangent kernel and random feature results do not precisely reflect practice due to their constraints, such as the solution's inability to deviate too far from the weights' initialization~\citep{fang2021mathematical}. A generalized NTK approach is considered in \citep{chen2020generalized}, inspired by the mean-field approach, they derive a high-probability bound on the generalization error of one-hidden-layer NNs with  convergence rate $\mathcal{O}(1/\sqrt{n})$ where $n$ is the number of training samples. A high probability generalization error bound with convergence rate $\mathcal{O}(1/\sqrt{n})$ in one-hidden-layer NN and with 0--1 and quadratic loss is studied in \citep{nitanda2021particle,nishikawatwo} assuming a mean-field regime. In this work, we study the transfer learning in mean-field regime. Recently, \citep{aminian2023mean} proposed an approach based on differential calculus on the space of probability measures to study the weak and strong generalization error in supervised learning scenario. A similar approach is utilized to study the generalization error of graph neural networks in mean-field regime by \citet{aminian2024generalization}.

\paragraph{Mean-field:} Our study employs the mean-field framework utilized in a recent line of research \citep{chizat/bach:2018,mei/montanari/nguyen:2018,mei/misiakiewicz/montanari:2019,fang2019over,fang2021mathematical,sirignano/spiliopoulos:2019,MFLD}. \citet{chizat/bach:2018} establishes the convergence of gradient descent for training one-hidden-layer NNs with infinite width under certain structural assumptions. The study of \citet{mei/montanari/nguyen:2018} proves the global convergence of noisy stochastic gradient descent and establishes approximation bounds between finite and infinite neural networks. Furthermore, \citet{mei/misiakiewicz/montanari:2019} demonstrates that this approximation error can be independent of the input dimension in certain cases, and establishes that the residual dynamics of noiseless gradient descent are close to the dynamics of NTK-based kernel regression under some conditions. The mean-field approach is mostly restricted to one-hidden-layer NNs and the extension to Deep NNs is not trivial~\citep{sirignano/spiliopoulos:2019,chizat2022infinite}. \citet{fang2019over} proposes a new concept known as neural feature repopulation inspired by the mean field view. Lastly, the mean-field approach's performance with respect to other methods in feature learning~\citep{fang2021mathematical} suggests it is a viable option for analyzing one-hidden-layer NNs. The focus of our study is based on the work conducted by \citet{MFLD} in the area of non-linear functional minimization with KL regularization. Their investigation demonstrated the linear convergence (in continuous time) of the resulting mean-field Langevin dynamics under the condition of sufficiently robust regularization. In this work, we focus on the study of weak transfer generalization error and weak transfer excess risk in mean-field regime.
\clearpage

\section{Technical Preliminaries}\label{app_calculus}
An overview of our main results is provided in Fig.~\ref{Diagram: TransferLearningResults}. All notations are summarized in Table~\ref{Table: notation}.
\begin{figure}[htbp]
\centering
\begin{tikzpicture}
    [
    scale=1,
    level distance=2.5cm,
    level 1/.style={sibling distance=8cm},
    level 2/.style={sibling distance=4cm},
    edge from parent/.style={draw=black!50, thick, -stealth},
    every node/.style={
        draw=black!50,
        thick,
        rounded corners,
        rectangle,
        top color=white,
        bottom color=blue!20,
        font=\tiny,
        text height=-.5ex,
        text width=2cm,
        minimum height=1.5cm,
        align=center,
    }
]
    \node {Transfer Learning\\ in \\Mean-Field Regime}
        child {
            node {$\alpha$-ERM}
            child {
                node {Weak Transfer Generalization Error}
                child { node {Upper Bound \\ (Theorem~\ref{thm: WTGE alpha erm})} }
            }
            child {
                node {Weak Transfer Excess Risk}
                child { node {Upper Bound \\ (Theorem~\ref{thm: population alpha erm})} }
            }
        }
        child {
            node {Fine-tuning}
            child {
                node {Weak Transfer Generalization Error}
                child { node {Upper Bound \\ (Theorem~\ref{thm: WTGE fine tune})} }
            }
            child {
                node {Weak Transfer Excess Risk}
                child { node {Upper Bound \\(Theorem~\ref{thm: excess-fine-tune}) } }
            }
        };
\end{tikzpicture}
\caption{Overview of Transfer Learning Results}
\label{Diagram: TransferLearningResults}
\end{figure}
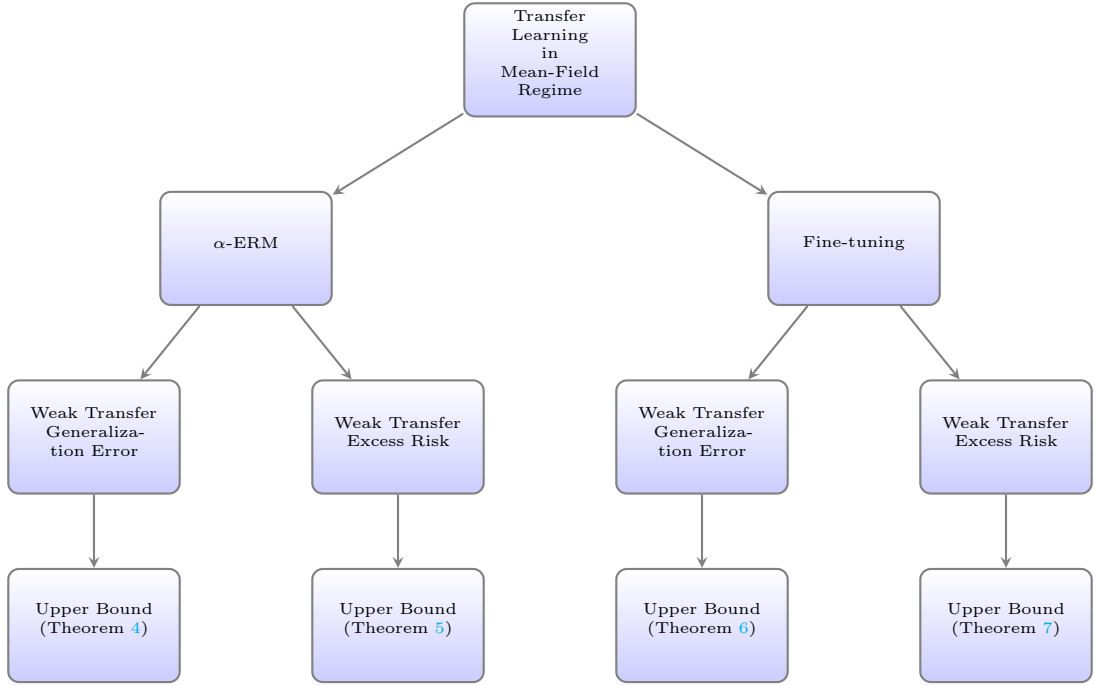


\begin{table}[ht]
    \centering
    \caption{Summary of notations in the paper}
   \resizebox{\linewidth}{!}{ \begin{tabular}{cc|cc}
         \toprule
         Notation&  Definition & Notation&  Definition\\
         \midrule
         $\mbZns$ & Source training dataset & $\mbZnt$ & Target Training dataset\\
         $n_t$ & Number of target training samples & $n_s$ & Number of source training samples\\
         $\nus$ & Source empirical measure & $\nupop^s$ & Source population measure \\
         $\nut$ & Target empirical measure & $\nupop^t$ & Target population measure \\
         $\frac{\sigma^2}{2\beta^2}$ & KL--Regularization parameter & $\mfm(.)$ & Measure over parameters\\
         $\theta$ & Parameters of model & $\Theta$ & Model's parameter space\\
         $\thetas$ & Specific parameters of model & $\Thetas$ &  Model's specific parameter space\\
         $\thetac$ & Comman parameters of model & $\Thetac$ & Model's common parameter space\\
         $\mathrm{gen}(\mfm,\nupop^t)$ & Generalization Error  &  $\mathcal{E}(\mfm,\nupop^t)$& Excess Risk \\
         $\ell(m,z)$ & Loss function & $\varphi(.)$ & Activation function\\
         $\mfm^{\beta}(\nu)$ & Gibbs measure & $\KLr(m\|\gamma_{\theta}^\sigma)$ & KL divergence between $m$ and $\gamma_{\theta}^\sigma$\\
         \bottomrule
    \end{tabular}}
    \label{Table: notation}
\end{table}

\clearpage

We provide the full versions of Assumption~\ref{ass:KLreg_assn} and Assumption~\ref{ass:KLreg_assnf}.

\begin{repassumption}{ass:KLreg_assn}[Full version]The loss function and the functional derivative of loss function with respect to $m$, are such that,

    \begin{enumerate}[(i)]
        
        \item \label{ass:KLreg_assn_item3} For all $z\in \mcZ$ and $m\in \mathcal{P}_4(\Theta)$, there exists $L_m>0$ such that the loss satisfies
        \begin{equation}0\leq \ell(m,z) \leq g(m)\big(1+\|z\|^2\big),\end{equation}
        where $g(m)=L_m (1+\mbE_{\theta\sim m}[\|\theta\|^4])$
        \item \label{ass:derivative_growth_bound}
        For all $z\in \mcZ$ and $m\in \mathcal{P}_4(\Theta)$, there exists $L_e>0$ such that the functional derivative of loss satisfies
    \begin{equation}\label{eq: bounded g ass}\Big|\frac{\delta \ell}{\delta m}(m,z,\theta) \Big| \leq g_e(m,\theta)\big(1+\|z\|^2\big),\end{equation}
where $g_e(m,\theta)= L_e\big(1+\mbE_{\theta\sim m}[\|\theta\|^4]+\|\theta\|^4\big)$.
\item \label{ass:KLreg_assn_item4}  For all $m\in \mathcal{P}_8(\Theta)$, the regularizing potential $U$ satisfies
        $\lim_{\|\theta\|\to \infty}\frac{U(\theta)}{\|\theta\|^4+g_{\mathrm{e}}(m,\theta)} = \infty$;       
\item\label{ass:KLreg_D2loss_integrability} We have the pointwise integrability conditions $g(\tilgammap)<\infty$ and, for all $\nu\in \mathcal{P}_2(\mathcal{Z})$ and $m,m'\in \mathcal{P}_8(\Theta)$,
            \begin{equation}\label{eq: bounded second moment ass}
            \mathbb{E}_{\theta,\theta'}\mathbb{E}_{z\sim \nu}\Big[\Big(\frac{\delta^2 \mathrm{\ell}}{\delta m^2}(m, z, \theta, \theta') \Big)^2\Big]^{1/2}<\infty.
            \end{equation}
    \end{enumerate}
\end{repassumption}


\begin{repassumption}{ass:KLreg_assnf}[\textbf{Full version}]
The loss function and the functional derivative of loss function with respect to $m_{\spec}$, are such that,

    \begin{enumerate}[(i)]
        
        \item \label{ass:KLreg_assn_item3f} For all $z\in \mcZ$, $m_{\mathrm{c}}\in \mathcal{P}_8(\Thetac)$ and $m_{\spec}\in\mathcal{P}_8(\thetas)$, there exists $L_m$ such that the loss satisfies
        \[ \ell(m_{\mathrm{c}}m_{\spec},z)  \leq  g(m_{\mathrm{c}}m_{\spec})\big(1+\|z\|^2\big),\]
where $g(m_{\mathrm{c}}m_{\spec})=L_m(1+\mbE_{\thetac\sim m_{\mathrm{c}}}[\|\thetac\|^4]+\mbE_{\thetas\sim m_{\spec}}[\|\thetas\|^4]).$

 \item \label{ass:KLreg_D2loss_growth_bound_FT}
        For all $z\in \mcZ$, $m_{\mathrm{c}}\in \mathcal{P}_8(\Thetac)$ and $m_{\spec}\in\mathcal{P}_8(\thetas)$, there exists $L_e$ such that the functional derivative of loss function with respect to $m_{\spec}$ satisfies
        \begin{equation}\label{eq: bounded g assf}
            \Big|\frac{\delta \ell}{\delta m_{\spec}}(m_\mathrm{c}m_{\spec},z,\thetas^t) \Big| \leq g_{\spec}(m_\mathrm{c}m_{\spec},\thetas^t)\big(1+\|z\|^2\big);
        \end{equation}
        where $g_{\spec}(m_\mathrm{c}m_{\spec},\thetas^t)=L_e(1+\mbE_{\thetac\sim m_{\mathrm{c}}}[\|\thetac\|^4]+\mbE_{\thetas\sim m_{\spec}}[\|\thetas^t\|^4]+\|\thetas^t\|^4)$.
       
        \item \label{ass:KLreg_assn_item4f}  For all $m_\mathrm{c}\in \mathcal{P}_8(\Thetac)$ and $m_{\spec}\in \mathcal{P}_8(\Thetas)$, the regularizing potential $U$ satisfies
        \[\lim_{\|\thetas\|\to \infty}\frac{U(\thetas^t)}{\|\thetas^t\|^8+g_{\spec}(m_\mathrm{c}m_{\spec},\thetas^t)} = \infty;\]       
       
         \item\label{ass:KLreg_D2loss_integrability_FT} We assume that $g(\hat\gamma_p^\sigma)<\infty$, $g(\tilde\gamma^{\sigma}_{\spec})<\infty$ and, for all $\nu\in \mathcal{P}_2(\mathcal{Z})$ and $m_{\spec},m_{\spec}^{\prime}\in \mathcal{P}_8(\Thetas)$,
            \begin{equation}\label{eq: bounded second moment ass_ft}
            \mathbb{E}_{\theta_\spec,\theta_\spec^{\prime}}\mathbb{E}_{z\sim \nu}\Big[\Big(\frac{\delta^2 \mathrm{\ell}}{\delta m_{\spec}^2}(m_{c} m_{\spec}, z, \theta_{\spec}^t, \theta_{\spec}^{t,\prime}) \Big)^2\Big]^{1/2}<\infty.
            \end{equation}

    \end{enumerate}
\end{repassumption}

\begin{lemma}\label{lem: parameter upper bound with dist}
    Suppose that $\tilde\nu\in\mathcal{P}_2(\mcZ)$, $m_{1},m_2\in\mathcal{P}_4(\Theta)$. Then, under Assumption~\ref{ass:KLreg_assn}, the following upper bound holds,
    \[
    \begin{split}
        &\mrR(m_{1},\tilde\nu)- \mrR(m_{2},\tilde\nu)\leq L_e\mathbb{E}_{Z\sim\tilde \nu}[(1+\|Z\|^2)^2] \Disfth(m_1,m_2).
    \end{split}
    \]
\end{lemma}
\begin{proof}
   Define $m(\lambda)=m_2+\lambda(m_1-m_2)$
    \begin{align*}
        &\mrR(m_{1},\tilde\nu)- \mrR(m_{2},\tilde\nu)\\&=\int_{\mcZ} \ell(m_{1},z)-\ell(m_{2},z) \tilde\nu(\mrd z)\\
        &=\int_{0}^1\int_{\mcZ}\int_{\Theta} \frac{\delta\ell}{\delta m}(m(\lambda),z,\theta) (m_{1}-m_{2})(\mrd \theta)\tilde\nu(\mrd z)\mrd \lambda\\
    &\overset{(a)}{\leq} \int_{0}^1 \int_{\Theta} \Big(\int_{\mcZ}(1+\|Z\|^2)^2\tilde\nu(\mrd z)\Big)^{1/2} \Bigg(\int_{\mcZ}\Bigg(\frac{\frac{\delta\ell}{\delta m}(m(\lambda),z,\theta)}{(1+\|z\|^2)}\Bigg)^2\tilde\nu(\mrd z)\Bigg)^{1/2}(m_{1}-m_{2})(\mrd \theta)\mrd \lambda
    \\
    &=\Big(\int_{\mcZ}(1+\|z\|^2)^2\tilde\nu(\mrd z)\Big)^{1/2} \int_{0}^1 \int_{\Theta}  \Big(\int_{\mcZ}\Bigg(\frac{\frac{\delta\ell}{\delta m}(m(\lambda),z,\theta)}{(1+\|Z\|^2)}\Bigg)^2\tilde\nu(\mrd z)\Big)^{1/2}(m_{1}-m_{2})(\mrd \theta)\mrd \lambda\\
    &= \Big(\int_{\mcZ}(1+\|z\|^2)^2\tilde\nu(\mrd z)\Big)^{1/2} \\
    &\quad \times \int_{0}^1 \int_{\Theta} g_{\mathrm{e}}(m(\lambda),\theta)\Bigg(\int_{\mcZ}\Bigg(\frac{\frac{\delta\ell}{\delta m}(m(\lambda),z,\theta)}{(1+\|Z\|^2)g_{\mathrm{e}}(m(\lambda),\theta)}\Bigg)^2\tilde\nu(\mrd z)\Bigg)^{1/2}(m_{1}-m_{2})(\mrd \theta)\mrd \lambda
     \\
    &= \Big(\int_{\mcZ}(1+\|z\|^2)^2\tilde\nu(\mrd z)\Big)^{1/2} \\
    &\quad \times \int_{0}^1 \int_{\Theta} g_{\mathrm{e}}(m(\lambda),\theta)\Bigg(\int_{\mcZ}\Bigg(\frac{\frac{\delta\ell}{\delta m}(m(\lambda),z,\theta)}{(1+\|z\|^2)g_{\mathrm{e}}(m(\lambda),\theta)}\Bigg)^2\tilde\nu(\mrd z)\Bigg)^{1/2}(m_{1}-m_{2})(\mrd \theta)\mrd \lambda
     \\
     &\leq L_e\mathbb{E}_{Z\sim\tilde \nu}[(1+\|Z\|^2)^2]  \int_{0}^1 \int_{\Theta}\big(1+\mbE_{\theta\sim m(\lambda)}[\|\theta\|^4]+\|\theta\|^4\big)(m_{1}-m_{2})(\mrd \theta)\mrd \lambda, \\
    &\overset{(b)}{\leq} L_e\mathbb{E}_{Z\sim\tilde \nu}[(1+\|Z\|^2)^2] \Disfth(m_1,m_2),
    \end{align*}
    where $g_{\mathrm{e}}(m(\lambda),\theta)= L_e\big(1+\mbE_{\theta\sim m}[\|\theta\|^4]+\|\theta\|^4\big)$ and (a) and (b) follow from Cauchy–Schwarz inequality and Assumption~\ref{ass:KLreg_assn}(\ref{ass:derivative_growth_bound}).
\end{proof}
Similar result also holds under Assumption~\ref{ass:KLreg_assnf}.
\begin{lemma}\label{lem: data upper bound with dist}
    Suppose that $\nu_1,\nu_2\in\mathcal{P}_2(\mcZ)$, $m\in\mathcal{P}_8(\Theta)$. Then, under Assumption~\ref{ass:KLreg_assn}, the following upper bound holds,
    \[
    \begin{split}
        &\mrR(m,\nu_1)- \mrR(m,\nu_2)\leq g(m)\mathrm{d}_{\mathfrak{F}_2}(\nu_1,\nu_2),
    \end{split}
    \]
    where $g(m)=L_m(1+\mbE_{\theta\sim m}[\|\theta\|^4])$.
\end{lemma}
\begin{proof}
    \begin{equation}
        \begin{split}
             &\mrR(m,\nu_1)- \mrR(m,\nu_2)\\
        &=\int_{\mcZ}\ell(m,z)(\nu_1-\nu_2)(\mrd z)\\
        &= g(m) \int_{\mcZ} \frac{\ell(m,z)}{g(m)}(\nu_1-\nu_2)(\mrd z)\\
        &\leq g(m)\Big|\int_{\mcZ} \frac{\ell(m,z)}{g(m)}(\nu_1-\nu_2)(\mrd z)\Big|\\
        &\leq g(m)\mathrm{d}_{\mathfrak{F}_2}(\nu_1,\nu_2).
        \end{split}
    \end{equation}
\end{proof}
\begin{remark}
We can see that the integral probability metric (IPM) in the above calculation is a relatively simple way to give a bound on the impact of changing $\nu_1$ to $\nu_2$, for a given parameter distribution $m$. In many contexts, a tighter bound could also be used, for example if $\ell(m,\cdot)/g(m)$ has more restrictive growth assumptions (e.g. $\ell$ is bounded), or if $\ell$ satisfies an anisotropic bound (i.e. where certain directions of $z$ are less significant than others when computing the loss). In these settings, we can simply replace the IPM wherever it appears by a corresponding bound. This may be particularly valuable when assessing the significance of the difference between source and target problem distributions, where the quantity $\mathrm{d}_{\mathfrak{F}_2}(\nu^s,\nu^t)$ is an otherwise unavoidable term in our final error bounds.
\end{remark}
\begin{lemma}
    Consider $m_1\in\mathcal{P}_4(\Theta)$ and $m_2\in\mathcal{P}_4(\Theta)$. Then, the distance $\Disfth(m_1,m_2)$ is finite.
\end{lemma}
\begin{proof}
    \begin{equation}
        \begin{split}
          &\Dispth(m_1,m_2)=\sup_{f\in\mathfrak{F}_4}\Big|\int_{\Theta}f(\theta)(m_1-m_2)(\mrd \theta)\Big|\\
            &\leq \int_{\Theta}\big(1+\|\theta\|^4\big) m_1 (\mrd \theta) + \int_{\Theta}\big(1+\|\theta\|^4\big) m_2 (\mrd \theta),
        \end{split}
    \end{equation}
    where $\int_{\Theta}\|\theta\|^p m_i (\mrd \theta)$ is bounded due to assumption $m_i\in\mathcal{P}_4(\Theta)$ for $i=1,2$.
\end{proof}

    In the our proof, we will regularly make use of the basic inequality (based on Cauchy--Schwarz), for any square integrable random variables $X,Y$,
    \[|\mathrm{Cov}(X,Y)| \leq \sqrt{\mathbb{V}(X)\mathbb{V}(Y)}\leq \sqrt{\mathbb{E}(X^2)\mathbb{E}(Y^2)}.\]
    A slightly more involved, but fundamentally similar, inequality is given in the following lemma.
\begin{lemma}\label{lem:polybound} \citep[Lemma~D.7]{aminian2023mean}.
    Let $\pi:(\mathbb{R}^+)^{2n}\to \mathbb{R}$ be a polynomial with positive coefficients. Then 
    \[\mathbb{E}_{\mathbf{Z}_n, \ols{\mathbf{Z}}_n}\Big[\pi\Big(\|Z_1\|,\|Z_2\|,..., \|Z_n\|,\|\ols Z_1\|, \|\ols Z_2\|,..., \|\ols Z_n\|\Big )\Big]\leq \mathbb{E}_{Z_1}\Big[\pi\Big(\|Z_1\|,\|Z_1\|,...,\|Z_1\|\Big )\Big].\]
    The same result also holds when the polynomial involves expectation of $\mathbb{E}_{Z_i}[\|Z_i\|^\zeta]$ terms.
\end{lemma}

\section{Proofs and details from Section~\ref{sec_main_results}}\label{App: Proofs of GE}

\begin{tcolorbox}
    \begin{replemma}{Lemma: rep gen tf}
   Consider a generic loss function $(m,z)\mapsto \ell(m,z)$, and $(\nutr,\nu_{n_s}^s)$ as defined in \eqref{Eq: nu replace one}. The WTGE \eqref{eq:WTGE} is given by,
    \begin{align}
        \label{eq_intro_loss_gap-transfer_app}
     &\mathrm{gen}(\mfm(\nut,\nus),\nupop^t) =\mathbb{E}_{\mbZnt,\mbZns,\ols{Z}_1^t} \Big[\ell\big(\mfm(\nut,\nus),\ols{Z}_1^t\big) -   \ell\big(\mfm(\nutr,\nus),\ols{Z}_1^t\big) \Big].
     \end{align}
\end{replemma}
\end{tcolorbox}
\begin{proof}[Proof of Lemma~\ref{Lemma: rep gen tf}]
  Recall that $\ols{Z}_1^t\sim \nupop$ is independent of $\{Z_i^t\}_{i=1}^n$. Since elements of $\{Z_i^t\}_{i=1}^n$ are i.i.d., the WTGE \eqref{eq:WTGE} can be written 
\begin{equation}\label{eq_linear_functionaU_lderivative_linear_intro_transfer}
 \begin{split}
   \mathrm{gen}(\mfm,\nupop^t)&= \mathbb{E}_{\mbZnt,\mbZns}\big[ \mrR(\mfm(\nut,\nus),\nupop^t) - \mrR(\mfm(\nut,\nus),\nut) \big]
    \\
     &= \mathbb{E}_{\mbZnt,\mbZns}\Big[ \int_{\mathcal Z} \ell(\mfm(\nut,\nus),z)\nupop^t(\mrd z) -  \int_{\mathcal Z} \ell(\mfm(\nut,\nus),z)\nut(\mrd z) \Big]
     \\
    &= \mathbb{E}_{\mbZnt,\mbZns}\Big[\mathbb E_{\ols{Z}_1^t}[ \ell(\mfm(\nut,\nus),\ols{Z}_1^t)] -   \frac{1}{n_t}\sum_{i=1}^n\ell(\mfm(\nut,\nus),Z_i^t) \Big]
    \\
   &= \mathbb{E}_{\mbZnt,\mbZns}\Big[\mathbb E_{\ols{Z}_1^t}[ \ell(\mfm(\nut,\nus),\ols{Z}_1^t) ]-   \ell(\mfm(\nut,\nus),Z_1^t) \Big].
     \end{split}
\end{equation}
Recall that in \eqref{Eq: nu replace one} we defined the perturbation 
$\nutr:= \nut + \frac{1}{n_t} \big( \delta_{\ols{Z}_1^t} - \delta_{Z_1^t}\big),$ which corresponds to a target data measure with one different data point, and that, as $Z_1^t$ and $\ols{Z}_1^t$ are i.i.d., we have
\begin{equation}\label{Eq: LOO}
    \mathbb{E}_{\mbZnt,\mbZns}\big[  \ell(\mfm(\nut,\nus),Z_1^t) \big] = 
\mathbb{E}_{\mbZnt,\mbZns}\mathbb{E}_{\ols{Z}_1^t}\big[ \ell(\mfm(\nutr,\nus),\ols{Z}_1^t) \big]\,.
\end{equation}
Combining this observation with \eqref{eq_linear_functionaU_lderivative_linear_intro_transfer} yields the representation \eqref{eq_intro_loss_gap-transfer}.
\end{proof}
\begin{tcolorbox}
    \begin{reptheorem}{thm: weak_gen_func_transfer}[\textbf{Restated}] 
Given Assumption \ref{main_ass_square_bounded}, the weak transfer generalization error has the representation
\[\begin{split}
    &\mathrm{gen}(\mfm,\nupop^t)=\mathbb{E}_{\mbZnt,\mbZns,\ols{Z}_1^t}\Bigg[\int_{0}^1 \int_{\Theta}\Big( \frac{\delta \ell}{\delta m_{\spec}}\big(\mfm_{(1)}(\lambda),\ols{Z}_1^t,\theta\big)\Big)\big(\mfm(\nut,\nus)-\mfm(\nutr,\nus)\big)(\mrd \theta)\,\mrd \lambda\Bigg].
    \end{split}
\]
\end{reptheorem}
\end{tcolorbox}
\begin{proof}[Proof of Theorem \ref{thm: weak_gen_func_transfer}]
Recall that from \eqref{mperturb1},  $\nutr = \nut + \frac{1}{n_t} (\delta_{\ols{Z}_1^t} - \delta_{Z_1^t})$ and $
    \mfm_{(1)}(\lambda) = \mfm(\nut,\nus) + \lambda\big( \mfm(\nutr,\nus) - \mfm(\nut,\nus)\big).
$
Using the definition of the linear functional derivative (Definition~\ref{def:flatDerivative}) and Lemma \ref{Lemma: rep gen tf} we have
\[
\begin{split}
    \mathrm{gen}(\mfm,\nupop^t) &=\mathbb{E}_{\mbZnt,\mbZns,\ols{Z}_1^t} \Big[\ell\big(\mfm(\nut,\nus),\ols{Z}_1^t\big) -   \ell\big(\mfm(\nutr,\nus),\ols{Z}_1^t\big) \Big]\,\\
    &=\mathbb{E}_{\mbZnt,\mbZns,\ols{Z}_1^t}\Big[ \int_{0}^1 \int_{\Theta}\Big( \frac{\delta \ell}{\delta m_{\spec}}\big(\mfm_{(1)}(\lambda),\ols{Z}_1^t,\theta\big)\Big) \big(\mfm(\nut,\nus)-\mfm(\nutr,\nus)\big)(\mrd \theta)\,\mrd \lambda\Big].
    \end{split}
\]  
\end{proof}
\begin{tcolorbox}
\begin{reptheorem}{thm: another rep WTGE-alpha-scenario}[\textbf{Restated}] The WTGE of $\alpha$-ERM can be written 
\begin{equation}\label{Eq:Finite-term-trnasfer_app}
        \mathrm{gen}(\mfm,\nupop^t)  =\frac{\alpha}{n_t}\mathbb{E}_{\mbZnt,\mbZns,\ols{Z}_1^t}\Bigg[ \int_{0}^1\int_{0}^1\int_{\mcZ} A(\lambda,\tilde \lambda)\big(\delta_{Z_1^t}-\delta_{\tilde Z_1^t}\big)(\mrd z) \mrd \tilde{\lambda}\mrd \lambda\Bigg],
       \end{equation} 
       where \\
     {\small
    $ A(\lambda,\tilde \lambda)= \int_{\Theta}  \frac{\delta \ell}{\delta m_{\spec}}\big(\mfm_{(1)}(\lambda),\ols{Z}_1^t,\theta\big)\frac{\delta \mfm}{\delta \nu}\big(\nu_{\alpha,(1)}(\tilde{\lambda}),z\big)(\mrd \theta),$} and $ \nu_{\alpha,(1)}(\tilde\lambda)= ( \alpha \nutr + (1-\alpha)\nus) + \tilde\lambda \alpha (\nut-\nutr).$
   \end{reptheorem}
\end{tcolorbox}

\begin{proof}[Proof of Theorem \ref{thm: another rep WTGE-alpha-scenario}]
        
        We know from \eqref{Eq: nu replace one} that 
        \[\nut - \nutr = \frac{1}{n_t}\big(\delta_{Z_1^t}-\delta_{\ols{Z}_1^t}\big).\]
         From Theorem \ref{thm: weak_gen_func_transfer}, for $\alpha$-ERM we have 
        \begin{equation}\label{eq: proof alpha 1}
        \begin{split}
        &\mathrm{gen}(\mfm,\nupop^t)\\
        &=\mathbb{E}_{\mbZnt,\mbZns,\ols{Z}_1^t}\Big[\int_{0}^1 \int_{\Theta}\Big( \frac{\delta \ell}{\delta m_{\spec}}\big(\mfm_{(1)}(\lambda),\ols{Z}_1^t,\theta\big)\Big) \\
        &\qquad\qquad\qquad\big(\mfm(\alpha \nut + (1-\alpha)\nus)-\mfm(\alpha \nutr + (1-\alpha) \nus)\big)(\mrd \theta)\,\mrd \lambda\Big].
        \end{split}
        \end{equation}
        
        Applying the functional derivative to $\mfm$  and recalling the definition of $\nu_{(1)}(\tilde\lambda)$ in \eqref{mperturb1}, for any Borel set $B\subset \Theta$ we have 
        \[\begin{split}
            &\big(\mfm_{\alpha}(\nualpha)-\mfm( \alpha \nutr + (1-\alpha)\nus)\big)(B) \\
            &= \alpha\int_0^1\int_{\mcZ} \Big(\frac{\delta \mfm}{\delta\nu}(\nu_{\alpha,(1)}(\tilde\lambda), z)(B)\Big)(\nut-\nutr)(\mrd z)\mrd\tilde\lambda\\
            &= \frac{\alpha}{n_t}\int_0^1\Big(\frac{\delta \mfm}{\delta\nu}(\nu_{\alpha,(1)}(\tilde\lambda), Z_1^t)(B)-\frac{\delta \mfm}{\delta\nu}(\nu_{\alpha,(1)}(\tilde\lambda), \ols{Z}_1^t)(B)\Big)\mrd\tilde\lambda,
        \end{split}\]
        where $\nu_{\alpha,(1)}(\tilde\lambda)= ( \alpha \nutr + (1-\alpha)\nus) + \tilde\lambda \alpha (\nut-\nutr)$.
       Then, we substitute this expression for and $\mfm(\alpha \nut + (1-\alpha)\nus)- \mfm(\alpha \nutr + (1-\alpha)\nus)$ in \eqref{eq: proof alpha 1}. The result follows immediately.
\end{proof}
\begin{tcolorbox}
\begin{reptheorem}{thm: another rep fine-tune}[\textbf{Restated}]
The WTGE of fine-tuning can be written,
\begin{equation}\label{eq: fine-tune_func_app}
\mathrm{gen}(\mfm,\nupop^t)  =\frac{1}{n_t}\mathbb{E}_{\mbZnt,\mbZns,\ols{Z}_1^t}\big[ K(\mbZnt,\mbZns,\ols{Z}_1^t)\big] ,    
\end{equation}
where
\begin{align*}&K(\mbZnt,\mbZns,\ols{Z}_1^t)\\
&=\int_{0}^1 \int_{\Theta_{\spec}}\Big( \frac{\delta \ell}{\delta m_{\spec}}\big(\mfm_{(1)}^f(\lambda),\ols{Z}_1^t,\theta_{\spec}^t\big)\Big) \\&\quad
 \times\bigg(\int_{0}^1\int_{\mcZ}\frac{\delta\mfms}{\delta \nu}(\mfmc(\nus),\nutr(\lambda_1),z)(\mrd \thetas^t)(\delta_{Z_1^t}-\delta_{\ols{Z}_1^t})(\mrd z) \mrd \lambda_1\bigg)\,\mrd \lambda,\end{align*}
and $\mfm_{(1)}^f(\lambda)=\mfm_{\mathrm{c}}(\nus)\big[(1-\lambda)(\mfm_{\spec}^t(\mfm_{\mathrm{c}}(\nus), \nut)+\lambda\mfm_{\spec}^t(\mfm_{\mathrm{c}}(\nus), \nutr)\big].$
\end{reptheorem}
\end{tcolorbox}

\begin{proof}[Proof of Theorem~\ref{thm: another rep fine-tune}]
From Theorem~\ref{thm: weak_gen_func_transfer} and considering $\theta=\theta_{\mathrm{c}}\cup \theta_{\spec}^t$ for target task, we have,

\begin{align*}
        &\mathrm{gen}(\mfm,\nupop^t)\\
        &=\mathbb{E}_{\mbZnt,\mbZns,\ols{Z}_1^t}\Big[\int_{0}^1 \int_{\Theta_{\spec}}\Big( \frac{\delta \ell}{\delta m_{\spec}}\big(\mfm_{(1)}^f(\lambda),\ols{Z}_1^t,\thetas^t\big)\Big) \\
        &\qquad \big(\mfmc(\nus)\mfms(\mfmc(\nus),\nut)-\mfmc(\nus)\mfms(\mfmc(\nus),\nutr)\big)(\mrd \thetac \mrd \thetas^t)\,\mrd \lambda\Big]\\
        &=\mathbb{E}_{\mbZnt,\mbZns,\ols{Z}_1^t}\Big[\int_{0}^1 \int_{\Theta_{\spec}}\Big( \frac{\delta \ell}{\delta m_{\spec}}\big(\mfm_{(1)}^f(\lambda),\ols{Z}_1^t,\thetas^t\big)\Big) \\
        &\qquad \big(\mfmc(\nus)\mfms(\mfmc(\nus),\nut)-\mfmc(\nus)\mfms(\mfmc(\nus),\nutr)(\mrd \thetac \mrd \thetas^t)\,\mrd \lambda\Big]\\
        &=\mathbb{E}_{\mbZnt,\mbZns,\ols{Z}_1^t}\Bigg[\int_{0}^1 \int_{\Theta_{\spec}}\Big( \frac{\delta \ell}{\delta m_{\spec}}\big(\mfm_{(1)}^f(\lambda),\ols{Z}_1^t,\thetas^t\big)\Big) \\
        &\qquad \Big(\int_{0}^1\int_{\mcZ}\frac{\delta\mfms}{\delta \nu}(\mfmc(\nus),\nutr(\lambda_1),z)(\mrd \thetas^t)(\nut-\nutr)(\mrd z) \mrd \lambda_1\Big)\,\mrd \lambda\Bigg]\\
        &=\mathbb{E}_{\mbZnt,\mbZns,\ols{Z}_1^t}\Bigg[\int_{0}^1 \int_{\Theta_{\spec}}\Big( \frac{\delta \ell}{\delta m_{\spec}}\big(\mfm_{(1)}^f(\lambda),\ols{Z}_1^t,\thetas^t\big)\Big) \\
        &\qquad \Big(\int_{0}^1\int_{\mcZ}\frac{\delta\mfms}{\delta \nu}(\mfmc(\nus),\nutr(\lambda_1),z)(\mrd \thetas^t)(\nut-\nutr)(\mrd z) \mrd \lambda_1\Big)\,\mrd \lambda\Bigg]\\
        &=\frac{1}{n_t}\mathbb{E}_{\mbZnt,\mbZns,\ols{Z}_1^t}\Bigg[\int_{0}^1 \int_{\Theta_{\spec}}\Big( \frac{\delta \ell}{\delta m_{\spec}}\big(\mfm_{(1)}^f(\lambda),\ols{Z}_1^t,\thetas^t\big)\Big) \\
        &\qquad \Big(\int_{0}^1\int_{\mcZ}\frac{\delta\mfms}{\delta \nu}(\mfmc(\nus),\nutr(\lambda_1),z)(\mrd \thetas^t)(\delta_{Z_1^t}-\delta_{\ols{Z}_1^t})(\mrd z) \mrd \lambda_1\Big)\,\mrd \lambda\Bigg].
        \end{align*}
  where \[\mfm_{(1)}^f(\lambda)=\mfm_{\mathrm{c}}(\nus)\big[(\mfm_{\spec}^t(\mfm_{\mathrm{c}}(\nus), \nut)+\lambda(\mfm_{\spec}^t(\mfm_{\mathrm{c}}(\nus), \nutr)-\mfm_{\spec}^t(\mfm_{\mathrm{c}}(\nus), \nut))\big].\]
\end{proof}

 \section{Proofs and details from Section \ref{Sec: KL-reg}}\label{app: proofs KL regularized}
\subsection{Technical Tools}
We borrowed some main technical tools which are used in our proofs from \citep{aminian2023mean}. The results in \citep{aminian2023mean} are applicable for supervised learning scenario. The following assumption is needed for the results in \citep{aminian2023mean}.
\begin{assumption}\label{ass:KLreg_assn_SL}
For a fixed $p\geq 2$, there exists $g:\mathcal{P}_2(\Theta)\to(0,\infty)$ and $g_{\mathrm{e}}:\mathcal{P}_2(\Theta)\times\Theta\to(0,\infty)$ such that, 
    \begin{enumerate}[(i)]
        \item \label{ass:KLreg_assn_item1_SL} The loss function is $\ell$ is $\mathcal{C}^2$ (Definition \ref{def:flatDerivative}), nonnegative, and convex with respect to $m$;   
        \item \label{ass:KLreg_assn_item3_SL} For all $z\in \mathcal{Z}$ and $m\in \mathcal{P}_8(\Theta)$, the loss satisfies
        \[\big|\ell(m,z) \big| \leq g(m)\big(1+\|z\|^2\big)\]
        and the derivative of the loss satisfies
        \[\Big|\frac{\delta \ell}{\delta m}(m,z,\theta) \Big| \leq g_{1}(m,\theta)\big(1+\|z\|^2\big);\]
        \item \label{ass:KLreg_assn_item4_SL}  For all $m\in \mathcal{P}_8(\Theta)$, the regularizing potential $U$ satisfies
        $\lim_{\|\theta\|\to \infty}\frac{U(\theta)}{\|\theta\|^p+g_{\mathrm{e}}(m,\theta)} = \infty$;       
        \item \label{ass:KLreg_D2loss_growth_bound_SL}
        
   There exists $L_e >0$ such that, for all $m,m'\in \mathcal{P}_8(\Theta)$,
        \begin{equation}
            \label{eq: bounded g ass_SL}
            \mathbb{E}_{\theta\sim m'}\Big[\big(g_{1}(m, \theta)\big)^2\Big]^{1/2}\leq L_e\Big(1+\mathbb{E}_{\theta\sim m'}\big[\|\theta\|^p\big]+\mathbb{E}_{\theta\sim m}\big[\|\theta\|^p\big]\Big),
            \end{equation}
        \item\label{ass:KLreg_D2loss_integrability_SL} We have the pointwise integrability conditions $g(\tilgammap)<\infty$ and, for all $\nu\in \mathcal{P}_2(\mathcal{Z})$ and $m,m'\in \mathcal{P}_8(\Theta)$;
            \begin{equation}\label{eq: bounded second moment ass_SL}
            \mathbb{E}_{\theta,\theta'\sim m}\mathbb{E}_{z\sim \nu}\Big[\Big(\frac{\delta^2 \mathrm{\ell}}{\delta m^2}(m, z, \theta, \theta') \Big)^2\Big]^{1/2}<\infty.
            \end{equation}
    \end{enumerate}
\end{assumption}


\begin{lemma}\label{lem:optimizerinBnu}
    Defining the set valued map 
        \[\begin{split}B(\nu) &:= \Big\{m\in \mathcal{P}_8(\Theta): \frac{\sigma^2}{2\beta^2}\mathrm{KL}(m\|\gamma^\sigma)\leq \mrR(\gamma^\sigma, \nu)\\
        &\qquad \text{ and } \int_\Theta \|\theta\|^8 m(\mrd \theta)\leq \mrR(\tilgammap,\nu) + \int_\Theta \|\theta\|^8\tilgammap(\mrd\theta)\Big\}\subset \mathcal{P}_8(\Theta),\end{split}\]
        for any $\nu\in \mathcal{P}_2(\mcZ)$, the KL-regularized risk minimizer satisfies $\mfm^{\beta}(\nu) \in B(\nu)$.
\end{lemma}

\begin{lemma}\label{lem:Cmonotone}\citep[Lemma~E.2]{aminian2023mean}.
    For any choice of $m\in B(\nu)$, the  linear map $\mathcal{C}_m:L^2(m,\Theta) \to L^2(m,\Theta)$ defined by 
    \[\mathcal{C}_{m} f(\theta) := \mathbb{E}_{\theta'\sim m}\Big[\frac{\delta^2 \mrR}{\delta m^2}(m, \nu, \theta, \theta') f(\theta')\Big]\]
    is positive, in the sense that
    \[\langle f, \mathcal{C}_{m}f\rangle_{L^2(m,\Theta)} = \int_\Theta f(\theta) (\mathcal{C}_{m} f)(\theta) m(\mrd\theta)\geq 0.\]
    In particular, $\mathcal{C}_{m}$ is a Hilbert--Schmidt operator with discrete spectrum 
    \[\sigma(\mathcal{C}_{m}) = \{\lambda^\mathcal{C}_i\}_{i\ge 0} \subset [0,\infty).\]
\end{lemma}

\begin{lemma}\label{lem:Sdifferentiable}\citep[Lemma~E.3]{aminian2023mean}.
Define 
\begin{equation}\label{eq:Sdefn}
    S(\nu, \theta) :=\frac{\delta\mrR}{\delta m}(\mfm^{\beta}(\nu),\nu,\theta) = \int_\mcZ \frac{\delta \ell}{\delta m}(\mfm^{\beta}(\nu), z, \theta)\nu(\mrd z).
\end{equation}
    Under Assumption \ref{ass:KLreg_assn_SL}, we know $S$ is differentiable in $\nu$, and its derivative satisfies the bound
    \[\begin{split}&\int_\Theta\Big(\frac{\delta S}{\delta \nu}(\nu, \theta, z)\Big)^2 \mfm^{\beta}(\nu; \mrd \theta)\\& \leq\int_\Theta\Big( \frac{\delta\ell}{\delta m}(\mfm^{\beta}(\nu),\theta,z) - \int_{\mcZ}\frac{\delta\ell}{\delta m}(\mfm^{\beta}(\nu),\theta,z')\nu(\mrd z')\Big)^2 \mfm^{\beta}(\nu; \mrd \theta).\end{split}\]
     In particular, we have the representation:
    \[\begin{split}    \frac{\delta S}{\delta \nu}(\nu,\theta, z)&= \frac{\delta \ell}{\delta m}\big(\mfm^{\beta}(\nu),z,\theta\big) - \int_{\mcZ}\frac{\delta \ell}{\delta m}\big(\mfm^{\beta}(\nu),z',\theta\big)\nu(\mrd z') \\
    &\quad -\frac{2\beta^2}{\sigma^2}\mathrm{Cov}_{\theta'\sim \mfm^{\beta}(\nu)}\Big[ \int_{\mcZ}\frac{\delta^2 \ell}{\delta m^2}(\mfm^{\beta}(\nu), z', \theta, \theta')\nu(\mrd z') ,\, \frac{\delta S}{\delta\nu}(\nu,\theta',z)\Big].
\end{split}    \]
\end{lemma}

\begin{lemma}
\label{lem_parameter_measure_derivative}\citep[Lemma~E.4]{aminian2023mean}.
Under Assumption \ref{ass:KLreg_assn_SL}, the density of the Gibbs measure, that is, the map $\nu\mapsto \mfm_{\theta}^{\beta}(\nu)$, has derivative
\[\frac{\delta\mfm^{\beta}_\theta}{\delta \nu}(\nu, \theta, z)= -\frac{2\beta^2}{\sigma^2}\mfm^{\beta}_\theta(\nu;\theta)\Big( \frac{\delta S}{\delta\nu}(\nu,\theta,z) - \Big[\int_\Theta \mfm^{\beta}_\theta(\nu;\theta') \frac{\delta S}{\delta\nu}(\nu,\theta',z)\mrd\theta \Big]\Big).\]
In particular, for any $f\in L^2(\mrd \theta)$,
\[\int_\Theta f(\theta)\frac{\delta\mfm^{\beta}_\theta}{\delta \nu}(\nu, \theta, z)\mrd\theta= -\frac{2\beta^2}{\sigma^2}\mathrm{Cov}_{\theta\sim \mfm^{\beta}(\nu)}\Big[f(\theta),\,  \frac{\delta S}{\delta\nu}(\nu,\theta,z) \Big].\]

\end{lemma}

\subsection{\texorpdfstring{$\alpha$}{alpha}-ERM details and proofs}\label{app: alpha erm}

Note that under Assumption~\ref{ass:KLreg_assn}, Assumption~\ref{ass:KLreg_assn_SL} also holds. Therefore, the results in Lemma~\ref{lem_parameter_measure_derivative}, Lemma~\ref{lem:optimizerinBnu}, Lemma~\ref{lem:Cmonotone}
and Lemma~\ref{lem:Sdifferentiable} also hold.

\begin{proposition}\label{Prop_reccurence_estimate} 
Suppose Assumption \ref{ass:KLreg_assn} holds. Then
\[\begin{split}&\int_{\Theta}\Big[\frac{\delta S}{\delta\nu}(\nu,\theta, z)-\frac{\delta S}{\delta\nu}(\nu,\theta, z')\Big]^2\mfm^{\beta}(\nu; \mrd \theta) \leq 6L_e^2\Big(1+2\mathbb{E}_{\theta\sim \mfm^{\beta}(\nu)}\big[\|\theta\|^8\big]\Big)\Big(2+\|z\|^2+\|z'\|^2\Big)^2,
\end{split}\]
where $S(\nu,\delta)$ is defined in \eqref{eq:Sdefn}.
\end{proposition}
\begin{proof}[Proof of Proposition~\ref{Prop_reccurence_estimate}]

We know from Lemma~\ref{lem:Cmonotone} that 
$\frac{\delta S}{\delta\nu}(\nu,\theta, z)-\frac{\delta S}{\delta\nu}(\nu,\theta, z')$
is bounded in $L^2(\mfm^{\beta}(\nu))$, in particular
\[\begin{split}
        &\int_{\Theta}\Big[\frac{\delta S}{\delta\nu}(\nu,\theta, z)-\frac{\delta S}{\delta\nu}(\nu,\theta, z')\Big]^2\mfm^{\beta}(\nu; \mrd \theta) \\
    &\leq \int_{\Theta}\Big[\frac{\delta \ell}{\delta m}(\mfm^{\beta}(\nu),\theta, z) - \frac{\delta \ell}{\delta m}(\mfm^{\beta}(\nu),\theta, z')\Big]^2\mfm^{\beta}(\nu; \mrd \theta)\\
    &\leq 2\int_{\Theta}\Big[\frac{\delta \ell}{\delta m}(\mfm^{\beta}(\nu),\theta, z)\Big]^2 + \Big[\frac{\delta \ell}{\delta m}(\mfm^{\beta}(\nu),\theta, z')\Big]^2\mfm^{\beta}(\nu; \mrd \theta).
\end{split}\]

Assumption \ref{ass:KLreg_assn}(\ref{ass:derivative_growth_bound}) yields
\[\begin{split}&\int_{\Theta}\Big[\frac{\delta S}{\delta\nu}(\nu,\theta, z)-\frac{\delta S}{\delta\nu}(\nu,\theta, z')\Big]^2\mfm^{\beta}(\nu; \mrd \theta) \\
    &\leq 2\Big(2+\|z\|^2+\|z'\|^2\Big)^2\int_{\Theta}\big[1+\mbE_{\theta\sim \mfm^{\beta}(\nu; \mrd \theta)}[\|\theta\|^4]+\|\theta\|^4\big]^2\mfm^{\beta}(\nu; \mrd \theta)\\
    &\leq 6\Big(2+\|z\|^2+\|z'\|^2\Big)^2\int_{\Theta}\big[1+\mbE_{\theta\sim \mfm^{\beta}(\nu; \mrd \theta)}[\|\theta\|^4]^2+\|\theta\|^8\big]\mfm^{\beta}(\nu; \mrd \theta)\\
    &\leq6\Big(2+\|z\|^2+\|z'\|^2\Big)^2\big[1+2\mbE_{\theta\sim \mfm^{\beta}(\nu; \mrd \theta)}[\|\theta\|^8]\big].
\end{split}\]

\end{proof}

\begin{tcolorbox}
 \begin{reptheorem}{thm: WTGE alpha erm}[\textbf{Full Version}]
    Given Assumption \ref{ass:KLreg_assn}, $\mathbb{E}_{Z\sim\nupop^t}\big[\|Z\|^{8}\big]<\infty$ and $\mathbb{E}_{Z\sim\nupop^s}\big[\|Z\|^{4}\big]<\infty$, the weak transfer generalization error for the $\alpha$-ERM satisfies 
\[\begin{split}  |\mathrm{gen}( \mfm^{\beta}(\nualpha),\nupop^t)|&\leq \sqrt{2} L_e^2 (1+\alpha L_m)^2(1+\frac{2}{n_t})^2 \comp(\theta)
 \\
 & \times \Bigg[2 (2+\alpha)^2\mathbb{E}_{Z_1^t}\Big[(1+\|Z_1^t\|^2)^4\Big]\\&\qquad+ 2 (1-\alpha)^2\mathbb{E}_{Z_1^t}\Big[(1+\|Z_1^t\|^2)^2\Big]\mathbb{E}_{Z_1^s}\Big[(1+\|Z_1^s\|^2)^2\Big]\Bigg],\end{split}\]
 where $\comp(\theta)=\big[1+ 2 \int_{\Theta}\|\theta\|^8\tilgammap(\mrd \theta) +2\mbE_{\theta\sim\tilgammap}[\|\theta\|^4]\big]^2.$
 \end{reptheorem}
\end{tcolorbox}
\begin{proof}[Proof of Theorem \ref{thm: WTGE alpha erm}]

Recall from \eqref{mperturb1} that $\nu_{(1)}(\lambda)=\nu_{n,(1)} +\lambda(\nu_{n}-\nu_{\alpha,(1)})$. From Theorem~\ref{thm: another rep WTGE-alpha-scenario},
\begin{equation}
    \label{WGEKLbasic}
    \mathrm{gen}(\mfm_{\alpha}(\nualpha),\nupop^t)=\frac{1}{n_t}\mathbb{E}_{\mbZnt,\mbZns,\ols{Z}_1^t}\big[ h(\mbZnt
    ,\mbZns,\ols{Z}_1^t)  \big],\end{equation}
where
\begin{equation}
\label{eq:hrepeat}
    \begin{split}
       & h(\mbZnt,\mbZns,\ols{Z}_1^t)  \\&=\int_{0}^1\int_{0}^1\int_{\mcZ}\Big(\int_{\Theta}  \frac{\delta \ell}{\delta m}\big(\mfm_{(1)}(\lambda),\ols{Z}_1^t,\theta\big)\frac{\delta \mfm}{\delta \nu}\big(\nu_{\alpha,(1)}(\tilde{\lambda}),z\big)(\mrd \theta)\Big)\,\big(\delta_{Z_1^t}-\delta_{\tilde Z_1^t}\big)(\mrd z) \mrd \tilde{\lambda}\,\mrd \lambda,
    \end{split}
\end{equation}
As $\mathbb{E}_{\theta\sim \mfm^{\beta, \sigma}(\nu_{\alpha,(1)}(\tilde\lambda))}\big[\frac{\delta S}{\delta \nu}\big(\nu_{\alpha,(1)}(\tilde\lambda), z, \theta\big)\big]\equiv 0$ (by definition of $S= \delta \mrR/\delta m$ and the normalization condition and chain rule), from Lemma~\ref{lem_parameter_measure_derivative} we know 
\begin{equation}\label{WGEKLexp1}
    \begin{split}
    &\int_{\Theta} \Big( \frac{\delta \ell}{\delta m}\big(\mfm_{(1)}(\lambda),Z_1^t,\theta\big)\Big)\Big(\frac{\delta \mfm}{\delta \nu}\big(\nu_{\alpha,(1)}(\tilde{\lambda}),Z_1^t\big)-\frac{\delta \mfm}{\delta \nu}\big(\nu_{\alpha,(1)}(\tilde{\lambda}),\ols{Z}_1^t\big)\Big)(\mrd \theta)\\
 & = -\frac{2\beta^2}{\sigma^2}\mathrm{Cov}_{\theta\sim \mfm^{\beta, \sigma}(\nu_{\alpha,(1)}(\tilde\lambda))}\Big[ \frac{\delta \ell}{\delta m}\big(\mfm_{(1)}(\lambda),Z_1^t,\theta\big),\, \frac{\delta S}{\delta \nu}\big(\nu_{\alpha,(1)}(\tilde\lambda), Z_1^t, \theta\big)-\frac{\delta S}{\delta \nu}\big(\nu_{\alpha,(1)}(\tilde\lambda), \tilde Z_1^t, \theta\big)\Big]\\
 &\leq \frac{2\beta^2}{\sigma^2}\mathbb{E}_{\theta\sim \mfm^{\beta, \sigma}(\nu_{\alpha,(1)}(\tilde\lambda))}\Big[\Big(\frac{\delta \ell}{\delta m}\big(\mfm_{(1)}(\lambda),Z_1^t,\theta\big)\Big)^2\Big]^{1/2}\\
 &\qquad \times \mathbb{E}_{\theta\sim \mfm^{\beta, \sigma}(\nu_{\alpha,(1)}(\tilde\lambda))}\Big[\Big(\frac{\delta S}{\delta \nu}\big(\nu_{\alpha,(1)}(\tilde\lambda), Z_1^t, \theta\big)-\frac{\delta S}{\delta \nu}\big(\nu_{\alpha,(1)}(\tilde\lambda), \tilde Z_1^t, \theta\big)\Big)^2\Big]^{1/2}.
\end{split}
\end{equation}
Using Proposition \ref{Prop_reccurence_estimate} and \eqref{eq: bounded g ass},
\begin{equation}
    \begin{split}
    &\Big[\mathbb{E}_{\theta\sim \mfm^{\beta, \sigma}(\nu_{\alpha,(1)}(\tilde\lambda))}\Big(\Big|\frac{\delta S}{\delta \nu}\big(\nu_{\alpha,(1)}(\tilde\lambda), Z_1^t, \theta\big)-\frac{\delta S}{\delta \nu}\big(\nu_{\alpha,(1)}(\tilde\lambda), \tilde Z_1^t, \theta\big)\Big|^2\Big)\Big]^{1/2}\\
    &\leq L_e \Big(\int_{\Theta}\big[1+\mathbb{E}_{\theta \sim \mfm^{\beta, \sigma}(\nu_{\alpha,(1)}(\tilde\lambda))}\big[\|\theta\|^4\big]+\|\theta\|^4\big]^2\mfm^{\beta, \sigma}(\nu_{\alpha,(1)}(\tilde\lambda); \mrd \theta)\Big)^{1/2}\Big(2+\|Z_1^t\|^2+\|\ols{Z}_1^t\|^2\Big)\\
    &\leq \sqrt{2}L_e\Big[1+2\mathbb{E}_{\theta \sim \mfm^{\beta, \sigma}(\nu_{\alpha,(1)}(\tilde\lambda))}\big[\|\theta\|^8\big]\Big]\Big(2+\|Z_1^t\|^2+\|\ols{Z}_1^t\|^2\Big).
\end{split}\end{equation}

From Lemma \ref{lem:optimizerinBnu}, we know 
\[\mathbb{E}_{\theta \sim \mfm^{\beta, \sigma}(\nu_{\alpha,(1)}(\tilde\lambda))}\big[\|\theta\|^8\big]\leq \mrR(\tilgammap, \nu_{\alpha,(1)}(\tilde\lambda)) + \int_{\Theta}\|\theta\|^8\tilgammap(\mrd \theta).\]
By construction,
\[\nu_{\alpha,(1)}(\tilde\lambda) = \nualpha +\alpha\frac{1-\tilde \lambda}{n_t}\big(\delta_{\ols{Z}_1^t} - \delta_{Z_1^t}\big),\]
so 
\begin{equation}\label{eq:Rboundperturb1}
    \begin{split}
    \mrR(\tilgammap, \nu_{\alpha,(1)}(\tilde\lambda)) &= \mrR(\tilgammap, \nualpha) + \alpha\frac{1-\tilde \lambda}{n_t}\Big(\ell(\tilgammap, \ols Z_1^t) - \ell(\tilgammap, Z_1^t)\Big)\\
    &\leq \mrR(\tilgammap, \nualpha) + g(\tilgammap)\alpha\frac{1-\tilde \lambda}{n_t}\Big(2+\|\ols{Z}_1^t\|^2+ \|Z_1^t\|^2)\Big),
\end{split}
\end{equation}
where $g(\tilgammap)=L_m(1+\mbE_{\theta\sim \tilgammap}[\|\theta\|^4])$,
and hence 
\begin{equation}
    \label{WGEKLexp2}
    \begin{split}
    &\mathbb{E}_{\theta\sim \mfm^{\beta, \sigma}(\nu_{\alpha,(1)}(\tilde\lambda))}\Big[\Big|\frac{\delta S}{\delta \nu}\big(\nu_{\alpha,(1)}(\tilde\lambda), Z_1^t, \theta\big)-\frac{\delta S}{\delta \nu}\big(\nu_{\alpha,(1)}(\tilde\lambda), \tilde Z_1^t, \theta\big)\Big|^2\Big]^{1/2}\\
    &\leq \sqrt{2} L_e\Big(1+ 2\mrR(\tilgammap, \nualpha) +2\int_{\Theta}\|\theta\|^8\tilgammap(\mrd \theta)+2g(\tilgammap)\alpha\frac{1-\tilde \lambda}{n_t}\Big(2+ \|Z_1^t\|^2+\|\ols{Z}_1^t\|^2\Big)\Big)\\&\quad \times\Big(2+\|Z_1^t\|^2+\|\ols{Z}_1^t\|^2\Big).
\end{split}\end{equation}

We also know 
\begin{equation}
    \label{WGEKLexp3}
    \begin{split}
    &\mathbb{E}_{\theta\sim \mfm^{\beta, \sigma}(\nu_{\alpha,(1)}(\tilde\lambda))}\Big[\Big(\frac{\delta \ell}{\delta m}\big(\mfm_{(1)}(\lambda),Z_1^t,\theta\big)\Big)^2\Big]^{1/2}\\
    &\leq \mathbb{E}_{\theta\sim \mfm^{\beta, \sigma}(\nu_{\alpha,(1)}(\tilde\lambda))}\Big[g_{\mathrm{e}}(\mfm_{(1)}(\lambda),\theta)^2\Big]^{1/2} \big(1+\|Z_1^t\|^2\big)\\
    &\leq \sqrt{2}L_e \Big(1+ \mathbb{E}_{\theta\sim \mfm^{\beta, \sigma}(\nu_{\alpha,(1)}(\tilde\lambda))}\big[\|\theta\|^8\big]+\mathbb{E}_{\theta\sim \mfm_{(1)}(\lambda)}\big[\|\theta\|^8\big]\Big)\big(1+\|Z_1^t\|^2\big).
    \end{split}
\end{equation}
By definition, $\mfm_{(1)}(\lambda) = \mfm(\nu_{\alpha,(1)}) + \lambda(\mfm_{\alpha}(\nualpha) - \mfm(\nu_{\alpha,(1)}))$, and hence
\[\mathbb{E}_{\theta\sim \mfm_{(1)}(\lambda)}\big[\|\theta\|^8\big] = (1-\lambda)\mathbb{E}_{\theta\sim \mfm(\nu_{\alpha,(1)})}\big[\|\theta\|^8\big]+ \lambda\mathbb{E}_{\theta\sim \mfm_{\alpha}(\nualpha)}\big[\|\theta\|^8\big].\]
In a similar approach to \eqref{eq:Rboundperturb1}, we have,
\begin{equation}
    \begin{split}
    \mrR(\tilgammap, \nu_{\alpha,(1)}) &= \mrR(\tilgammap, \nualpha) + \frac{\alpha}{n_t}\Big(\ell(\tilgammap, \ols Z_1^t) - \ell(\tilgammap, Z_1^t)\Big)\\
    &\leq \mrR(\tilgammap, \nualpha) + g(\tilgammap)\frac{\alpha}{n_t}\Big(2+\|\ols{Z}_1^t\|^2+ \|Z_1^t\|^2)\Big).
\end{split}
\end{equation} 
As before, it follows from Lemma \ref{lem:optimizerinBnu} that
\begin{equation}
\label{momentbound2}\begin{split}
&\mathbb{E}_{\theta\sim \mfm^{\beta, \sigma}(\nu_{\alpha,(1)}(\tilde\lambda))}\big[\|\theta\|^8\big]+\mathbb{E}_{\theta\sim \mfm_{(1)}(\lambda)}\big[\|\theta\|^8\big]\\
& = \mathbb{E}_{\theta\sim \mfm^{\beta, \sigma}(\nu_{\alpha,(1)}(\tilde\lambda))}\big[\|\theta\|^8\big]+(1-\lambda)\mathbb{E}_{\theta\sim \mfm(\nu_{\alpha,(1)})}\big[\|\theta\|^8\big]+ \lambda\mathbb{E}_{\theta\sim \mfm_{\alpha}(\nualpha)}\big[\|\theta\|^8\big]\\
&\leq \Big[\mrR(\tilgammap, \nu_{\alpha,(1)}(\tilde\lambda)) + \int_{\Theta}\|\theta\|^8\tilgammap(\mrd \theta) \Big] + (1-\lambda)\Big[\mrR(\tilgammap, \nu_{\alpha,(1)}) + \int_{\Theta}\|\theta\|^8\tilgammap(\mrd \theta) \Big] \\
& \qquad + \lambda\Big[\mrR(\tilgammap, \nualpha) + \int_{\Theta}\|\theta\|^8\tilgammap(\mrd \theta) \Big]\\
    &\leq 2\mrR(\tilgammap, \nualpha) + 2\int_{\Theta}\|\theta\|^8\tilgammap(\mrd \theta) + g(\tilgammap)\alpha\frac{(2-\tilde \lambda-\lambda)}{n_t}\Big(2+\|\ols{Z}_1^t\|^2+ \|Z_1^t\|^2\Big).
\end{split} 
\end{equation}

Therefore, substituting \eqref{WGEKLexp1}, \eqref{WGEKLexp2}, \eqref{WGEKLexp3} and \eqref{momentbound2} into \eqref{eq:hrepeat} and simplifying, we have
 \begin{equation}\label{WGEKLexp4}
 \begin{split}
        &h(\mbZnt,\mbZns,\ols{Z}_1^t)   \\
        &\leq \frac{2\beta^2}{\sigma^2} L_e  \Big[1+2\mrR(\tilgammap, \nualpha) + 2 \int_{\Theta}\|\theta\|^8\tilgammap(\mrd \theta) + 2\frac{\alpha g(\tilgammap)}{n_t}\Big(2+\|\ols{Z}_1^t\|^2+ \|Z_1^t\|^2\Big)\Big]\big(1+\|Z_1^t\|^2\big)\\
        &\qquad \times \sqrt{2} L_e\Big[1+ 2\mrR(\tilgammap, \nualpha)+ 2\int_{\Theta}\|\theta\|^8\tilgammap(\mrd \theta) + 2\frac{\alpha g(\tilgammap)}{n_t}\Big(2+ \|Z_1^t\|^2+\|\ols{Z}_1^t\|^2\Big)\Big]\Big(2+\|Z_1^t\|^2+\|\ols{Z}_1^t\|^2\Big)\\
        &\leq \sqrt{2}\frac{2\beta^2}{\sigma^2} L_e^2  \Big[1+2\mrR(\tilgammap, \nualpha) + 2 \int_{\Theta}\|\theta\|^8\tilgammap(\mrd \theta) + 2\frac{\alpha g(\tilgammap)}{n_t}\Big(2+\|Z_1^t\|^2+\|\ols{Z}_1^t\|^2\Big)\Big]^2\\
        &\qquad \times \big(1+\|Z_1^t\|^2\big)\Big(2+\|Z_1^t\|^2+\|\ols{Z}_1^t\|^2\Big).
    \end{split}
    \end{equation}
We have the bound
\begin{equation}\label{eq:Rboundsum}
\mrR(\tilgammap, \nualpha) \leq \int_\mcZ g(\tilgammap) (1+\|z\|^2)\nualpha(\mrd{z}) = g(\tilgammap)\Big(\frac{\alpha}{n_t}\sum_{j=1}^{n_t}(1+\|Z_j^t\|^2)+\frac{(1-\alpha)}{n_s}\sum_{j=1}^{n_s}(1+\|Z_j^s\|^2)\Big).
\end{equation}

Therefore, using the inequality $(a+b(c+d))\leq  (a+cb)(c+d)/c$ for $a,b,c,d\ge 0$,
\begin{equation}\label{hexpansionfinal1} \begin{split}
        &h(\mbZnt,\mbZns,\ols{Z}_1^t)   \\
        &\leq   \Big[1+ 2 \int_{\Theta}\|\theta\|^8\tilgammap(\mrd \theta) + 2\frac{g(\tilgammap)}{n_t}\Big(2+\|Z_1^t\|^2+\|\ols{Z}_1^t\|^2+\frac{\alpha}{n_t}\sum_{j=1}^{n_t}(1+\|Z_j^t\|^2)+\frac{(1-\alpha)}{n_s}\sum_{j=1}^{n_s}(1+\|Z_j^s\|^2)\Big)\Big]^2\\
        &\qquad \times \sqrt{2}\frac{2\beta^2}{\sigma^2} L_e^2\big(1+\|Z_1^t\|^2\big)\Big(2+\|Z_1^t\|^2+\|\ols{Z}_1^t\|^2\Big)\\
        &\leq \frac{\sqrt{2}}{2}\frac{\beta^2}{\sigma^2} L_e^2  \Big[1+ 2 \int_{\Theta}\|\theta\|^8\tilgammap(\mrd \theta) + 4\frac{\alpha g(\tilgammap)}{n_t}\Big]^2\\
        &\quad \times \Big(2+\|Z_1^t\|^2+\|\ols{Z}_1^t\|^2+\frac{\alpha}{n_t}\sum_{j=1}^{n_t}(1+\|Z_j^t\|^2)+\frac{(1-\alpha)}{n_s}\sum_{j=1}^{n_s}(1+\|Z_j^s\|^2)\Big)^2\big(1+\|Z_1^t\|^2\big)\Big(2+\|Z_1^t\|^2+\|\ols{Z}_1^t\|^2\Big).
        \end{split}
        \end{equation}

Applying Lemma \ref{lem:polybound} we have 
\begin{equation}\label{polyexpansion1}\begin{split}
&\mathbb{E}_{\mbZnt,\mbZns, \ols{\mathbf{Z}}_n}\Big[\Big(2+\|Z_1^t\|^2+\|\ols{Z}_1^t\|^2+\frac{\alpha}{n_t}\sum_{j=1}^{n_t}(1+\|Z_j^t\|^2)+\frac{(1-\alpha)}{n_s}\sum_{j=1}^{n_s}(1+\|Z_j^s\|^2)\Big)^2\\
&\qquad \qquad \times \big(1+\|Z_1^t\|^2\big)\Big(2+\|Z_1^t\|^2+\|\ols{Z}_1^t\|^2\Big)\Big]\\
&\leq \mathbb{E}_{Z_1^t,Z_1^s}\Big[\Big(2+\|Z_1^t\|^2+\|Z_1^t\|^2+\alpha(1+\|Z_1^t\|^2)+(1-\alpha)(1+\|Z_1^s\|^2)\Big)^2\big(1+\|Z_1^t\|^2\big)\Big(2+\|Z_1^t\|^2+\|Z_1^t\|^2\Big)\Big]\\
&\leq 2 (2+\alpha)^2\mathbb{E}_{Z_1^t}\Big[(1+\|Z_1^t\|^2)^4\Big]+ 2 (1-\alpha)^2\mathbb{E}_{Z_1^t}\Big[(1+\|Z_1^t\|^2)^2\Big]\mathbb{E}_{Z_1^s}\Big[(1+\|Z_1^s\|^2)^2\Big].
\end{split}
\end{equation}
And therefore, substituting \eqref{polyexpansion1} and \eqref{hexpansionfinal1} in \eqref{WGEKLbasic},
\begin{equation}\label{WGEKLexp5}
\begin{split}
 |\mathrm{gen}(\mfm^{\beta}(\nualpha),\nupop^t)|
 &\leq \frac{\alpha}{n_t} \frac{\sqrt{2}}{2}\frac{\beta^2}{\sigma^2} L_e^2  \Big[1+ 2 \int_{\Theta}\|\theta\|^8\tilgammap(\mrd \theta) + 4\frac{\alpha g(\tilgammap)}{n_t}\Big]^2 
 \\
 &\quad \times \Bigg[2 (2+\alpha)^2\mathbb{E}_{Z_1^t}\Big[(1+\|Z_1^t\|^2)^4\Big]+ 2 (1-\alpha)^2\mathbb{E}_{Z_1^t}\Big[(1+\|Z_1^t\|^2)^2\Big]\mathbb{E}_{Z_1^s}\Big[(1+\|Z_1^s\|^2)^2\Big]\Bigg]\\
 &=\frac{\alpha}{n_t} \frac{\sqrt{2}}{2}\frac{\beta^2}{\sigma^2} L_e^2  \Big[1+ 2 \int_{\Theta}\|\theta\|^8\tilgammap(\mrd \theta) + 4\frac{\alpha L_m(1+\mbE_{\theta\sim\tilgammap}[\|\theta\|^4])}{n_t}\Big]^2 
 \\
 & \quad\times \Bigg[2 (2+\alpha)^2\mathbb{E}_{Z_1^t}\Big[(1+\|Z_1^t\|^2)^4\Big]+ 2 (1-\alpha)^2\mathbb{E}_{Z_1^t}\Big[(1+\|Z_1^t\|^2)^2\Big]\mathbb{E}_{Z_1^s}\Big[(1+\|Z_1^s\|^2)^2\Big]\Bigg]\\
 &\leq \frac{\alpha}{n_t} 2\sqrt{2}\frac{\beta^2}{\sigma^2} L_e^2 (1+\alpha L_m)^2(1+\frac{2}{n_t})^2  \Big[1+ 2 \int_{\Theta}\|\theta\|^8\tilgammap(\mrd \theta) +2\mbE_{\theta\sim\tilgammap}[\|\theta\|^4]\Big]^2 
 \\
 &\quad \times \Bigg[2 (2+\alpha)^2\mathbb{E}_{Z_1^t}\Big[(1+\|Z_1^t\|^2)^4\Big]+ 2 (1-\alpha)^2\mathbb{E}_{Z_1^t}\Big[(1+\|Z_1^t\|^2)^2\Big]\mathbb{E}_{Z_1^s}\Big[(1+\|Z_1^s\|^2)^2\Big]\Bigg].
\end{split}
\end{equation}

\end{proof}
\begin{tcolorbox}
\begin{reptheorem}{thm: population alpha erm}[\textbf{Full version}]
      Under the same assumptions in Theorem~\ref{thm: another rep WTGE-alpha-scenario}, the following upper bound holds on expected target population risk,
    \[\begin{split}
    &\mathcal{E}\big(\mdens^{\beta}(\nualpha),\nupop^t \big)\\
&\leq\frac{c_t\alpha}{n_t}\frac{2\beta^2}{\sigma^2}\comp(\theta) \Bigg[ (2+\alpha)^2\mathbb{E}_{Z_1^t}\Big[(1+\|Z_1^t\|^2)^4\Big]\\&\qquad\qquad+  (1-\alpha)^2\mathbb{E}_{Z_1^t}\Big[(1+\|Z_1^t\|^2)^2\Big]\mathbb{E}_{Z_1^s}\Big[(1+\|Z_1^s\|^2)^2\Big]\Bigg]\\
    &\quad+ \frac{c_s(1-\alpha)}{n_s}\frac{2\beta^2}{\sigma^2} \comp(\theta)\Bigg[ (3-\alpha)^2\mathbb{E}_{Z_1^s}\Big[(1+\|Z_1^s\|^2)^4\Big]\\&\qquad\qquad+  \alpha^2\mathbb{E}_{Z_1^t}\Big[(1+\|Z_1^t\|^2)^2\Big]\mathbb{E}_{Z_1^s}\Big[(1+\|Z_1^s\|^2)^2\Big]\Bigg]\\
    &\quad+6L_m(1+L_m)(1-\alpha)\Disnd(\nupop^s,\nupop^t)\\
    &\qquad\times \Big(1+\int_{\Theta}\|\theta\|^4\bar{m}_{\alpha}(\mrd\theta) \Big)\Big(1+\int_\Theta \|\theta\|^4\tilgammap(\mrd\theta)\Big)\\&\qquad\times\Big(1+\alpha \mathbb{E}_{Z_1^t}[(1+\|Z_1^t \|^2)]+(1-\alpha) \mathbb{E}_{Z_1^s}[(1+\|Z_1^s \|^2)]\Big) \\
    &\quad+  \frac{\sigma^2}{2\beta^2} \KLr(\bar{m}_\alpha\|\gamma^{\sigma})\\
    &\quad+L_e\mbE[(1+\|Z_1^t\|^2)^2]\Disfth(\bar{m}_{\alpha},\bar{m}^t),\end{split}\]
where 
\[c_t= \sqrt{2} L_e^2 (1+\alpha L_m)^2(1+\frac{1}{n_t})^2 >0,\]
\[c_s= \sqrt{2} L_e^2 (1+(1-\alpha) L_m)^2(1+\frac{1}{n_s})^2 >0,\]
\[\comp(\theta)=\Big[1+ 2 \int_{\Theta}\|\theta\|^8\tilgammap(\mrd \theta) +2\int_{\Theta}\|\theta\|^4\tilgammap(\mrd \theta)\Big]^2,\]
and $\bar{m}_{\alpha}=\inf_{m\in\mathcal{P}_2(\Theta)} R(m,\nupop^{\alpha})$ provided that $\KLr(\bar{m}_{\alpha}\|\gamma^{\sigma})<\infty$ and $\bar{m}^t=\inf_{m\in\mathcal{P}_2(\Theta)} R(m,\nupop^{t})$.
\end{reptheorem}
\end{tcolorbox}

\begin{proof}[Proof of Theorem \ref{thm: population alpha erm}]From the optimization defining $\mdens^{\beta}$, we know that
\begin{equation} 
 \mrR(\mdens^{\beta}(\nualpha),\nualpha) \leq 
\mathcal{V}^{\beta}(\mdens^{\beta}(\nualpha),\nualpha) \leq
\mathcal{V}^{\beta}(\bar{m}_{\alpha},\nualpha) = \mrR(\bar{m}_{\alpha},\nualpha)+\frac{\sigma^2}{2\beta^2} \KLr(\bar{m}_{\alpha}\|\gamma^{\sigma}) \,.
\end{equation}
We also know that
\[\mrR(\mdens^{\beta}(\nualpha),\nut)=\mrR(\mdens^{\beta}(\nualpha),\nualpha)+(1-\alpha)[\mrR(\mdens^{\beta}(\nualpha),\nut)-\mrR(\mdens^{\beta}(\nualpha),\nus)].\]
Hence,
\begin{align*}
&\mathcal{E}\big(\mdens^{\beta}(\nualpha),\nupop^t \big)\\
&=\mathbb{E}_{\mbZns,\mbZnt}\Big[ \mrR(\mdens^{\beta}(\nualpha),\nupop^t)-\mrR(\bar{m}^t,\nupop^t)\Big] \\ 
& =\mathbb{E}_{\mbZns,\mbZnt}\Big[\mrR(\mdens^{\beta}(\nualpha),\nupop^t)-\mrR(\mdens^{\beta}(\nualpha),\nualpha)+\mrR(\mdens^{\beta}(\nualpha),\nualpha)-\mrR(\bar{m}^t,\nupop^t)\Big]\\&
\leq \mathbb{E}_{\mbZns,\mbZnt}\Big[\mrR(\mdens^{\beta}(\nualpha),\nupop^t) - \mrR(\mdens^{\beta}(\nualpha),\nut) + (1-\alpha)[\mrR(\mdens^{\beta}(\nualpha),\nut)-\mrR(\mdens^{\beta}(\nualpha),\nus)]\Big] \\
&\quad+ \frac{\sigma^2}{2\beta^2} \KLr(\bar{m}_\alpha\|\gamma^{\sigma})+\mrR(\bar{m}_{\alpha},\nupop^{\alpha})-\mrR(\bar{m}^t,\nupop^t) \\
&=\mathbb{E}_{\mbZns,\mbZnt}\Big[\mrR(\mdens^{\beta}(\nualpha),\nupop^t) - \mrR(\mdens^{\beta}(\nualpha),\nut)\Big] \\
&\quad+ (1-\alpha)\mathbb{E}_{\mbZns,\mbZnt}\Big[\mrR(\mdens^{\beta}(\nualpha),\nut)-\mrR(\mdens^{\beta}(\nualpha),\nupop^t)\\
&\qquad+ \mrR(\mdens^{\beta}(\nualpha),\nupop^t)- \mrR(\mdens^{\beta}(\nualpha),\nupop^s)\\
&\qquad+\mrR(\mdens^{\beta}(\nualpha),\nupop^s)-\mrR(\mdens^{\beta}(\nualpha),\nus)\Big] \\
&\quad+ \frac{\sigma^2}{2\beta^2} \KLr(\bar{m}_\alpha\|\gamma^{\sigma})+\mrR(\bar{m}_{\alpha},\nupop^{\alpha})-\mrR(\bar{m}^t,\nupop^t)\\
&=\alpha\mathrm{gen}(\mfm_{\alpha}(\nualpha),\nupop^t)+(1-\alpha)\mathrm{gen}(\mfm_{\alpha}(\nualpha),\nupop^s)\\
&\quad+ (1-\alpha)\mathbb{E}_{\mbZns,\mbZnt}\big[ \mrR(\mdens^{\beta}(\nualpha),\nupop^t)- \mrR(\mdens^{\beta}(\nualpha),\nupop^s)\big] \\&\quad+ \frac{\sigma^2}{2\beta^2} \KLr(\bar{m}_\alpha\|\gamma^{\sigma})+ \mrR(\bar{m}_{\alpha},\nupop^{\alpha})-\mrR(\bar{m}^t,\nupop^t)\\
&=\alpha\mathrm{gen}(\mfm_{\alpha}(\nualpha),\nupop^t)+(1-\alpha)\mathrm{gen}(\mfm_{\alpha}(\nualpha),\nupop^s)\\
&\quad+ (1-\alpha)\mathbb{E}_{\mbZns,\mbZnt}\big[ \mrR(\mdens^{\beta}(\nualpha),\nupop^t)- \mrR(\mdens^{\beta}(\nualpha),\nupop^s)\big] \\&\quad+ \frac{\sigma^2}{2\beta^2} \KLr(\bar{m}_\alpha\|\gamma^{\sigma})+ \big(\mrR(\bar{m}_{\alpha},\nupop^{t})-\mrR(\bar{m}^t,\nupop^t)\big)\\
&\quad+ (1-\alpha)\big(\mrR(\bar{m}_{\alpha},\nupop^{s})-\mrR(\bar{m}_{\alpha},\nupop^t)\big).   
\end{align*}
Note that we have,
\[\begin{split}
    &\mbE_{\mbZns,\mbZnt}\Big[\int_\Theta \|\theta\|^4 \mdens^{\beta}(\nualpha,\mrd \theta)\Big]\\&\leq \mbE_{\mbZns,\mbZnt}\Big[\mrR(\tilgammap,\nualpha)\Big] + \int_\Theta \|\theta\|^4\tilgammap(\mrd\theta)\\
    &\leq g(\tilgammap)\Big(\alpha \mathbb{E}_{Z_1^t}[(1+\|Z_1^t \|^2)]+(1-\alpha) \mathbb{E}_{Z_1^s}[(1+\|Z_1^s \|^2)]\Big) + \int_\Theta \|\theta\|^4\tilgammap(\mrd\theta)\\
    &\leq (1+L_m)\Big(1+\int_\Theta \|\theta\|^4\tilgammap(\mrd\theta)\Big)\Big(1+\alpha \mathbb{E}_{Z_1^t}[(1+\|Z_1^t \|^2)]+(1-\alpha) \mathbb{E}_{Z_1^s}[(1+\|Z_1^s \|^2)]\Big) ,
\end{split}\]
where $g(\tilgammap)=L_m(1+\mbE_{\theta\sim \tilgammap }[\|\theta\|^4]).$
Using Theorem~\ref{thm: WTGE alpha erm}, Lemma~\ref{lem: data upper bound with dist} and Lemma~\ref{lem: parameter upper bound with dist} implies that
\begin{align*}
&\mathcal{E}\big(\mdens^{\beta}(\nualpha),\nupop^t \big) \\
&\leq \alpha\mathrm{gen}(\mfm_{\alpha}(\nualpha),\nupop^t)+(1-\alpha)\mathrm{gen}(\mfm_{\alpha}(\nualpha),\nupop^s)\\
    &\quad+ (1-\alpha)\big[ \mrR(\mdens^{\beta}(\nualpha),\nupop^t)- \mrR(\mdens^{\beta}(\nualpha),\nupop^s)\big] \\&\quad+  \frac{\sigma^2}{2\beta^2} \KLr(\bar{m}_\alpha\|\gamma^{\sigma})+\big(\mrR(\bar{m}_{\alpha},\nupop^{t})-\mrR(\bar{m}^t,\nupop^t)\big)\\
    &\quad+ (1-\alpha)\big(\mrR(\bar{m}_{\alpha},\nupop^{s})-\mrR(\bar{m}_{\alpha},\nupop^t)\big)\\
&\leq \frac{c_t\alpha}{n_t}\frac{2\beta^2}{\sigma^2} \comp(\theta) \Bigg[ (2+\alpha)^2\mathbb{E}_{Z_1^t}\Big[(1+\|Z_1^t\|^2)^4\Big]+  (1-\alpha)^2\mathbb{E}_{Z_1^t}\Big[(1+\|Z_1^t\|^2)^2\Big]\mathbb{E}_{Z_1^s}\Big[(1+\|Z_1^s\|^2)^2\Big]\Bigg]\\
    &\quad+ \frac{c_s(1-\alpha)}{n_s}\frac{2\beta^2}{\sigma^2} \comp(\theta)  \Bigg[ (3-\alpha)^2\mathbb{E}_{Z_1^s}\Big[(1+\|Z_1^s\|^2)^4\Big]+  (\alpha)^2\mathbb{E}_{Z_1^t}\Big[(1+\|Z_1^t\|^2)^2\Big]\mathbb{E}_{Z_1^s}\Big[(1+\|Z_1^s\|^2)^2\Big]\Bigg]\\&\quad+  \frac{\sigma^2}{2\beta^2} \KLr(\bar{m}_\alpha\|\gamma^{\sigma})+\big(\mrR(\bar{m}_{\alpha},\nupop^{t})-\mrR(\bar{m}^t,\nupop^t)\big)\\
    &\quad+ (1-\alpha)\big(\mrR(\bar{m}_{\alpha},\nupop^{s})-\mrR(\bar{m}_{\alpha},\nupop^t)\big)\\
    &\quad+(1-\alpha)\mathbb{E}_{\mbZns\mbZnt}\big[ \mrR(\mdens^{\beta}(\nualpha),\nupop^t)- \mrR(\mdens^{\beta}(\nualpha),\nupop^s)\big]\\
 &\leq\frac{c_t\alpha}{n_t}\frac{2\beta^2}{\sigma^2} \comp(\theta)\Bigg[ (2+\alpha)^2\mathbb{E}_{Z_1^t}\Big[(1+\|Z_1^t\|^2)^4\Big]+  (1-\alpha)^2\mathbb{E}_{Z_1^t}\Big[(1+\|Z_1^t\|^2)^2\Big]\mathbb{E}_{Z_1^s}\Big[(1+\|Z_1^s\|^2)^2\Big]\Bigg]\\
    &\quad+ \frac{c_s(1-\alpha)}{n_s}\frac{2\beta^2}{\sigma^2} \comp(\theta) \Bigg[ (3-\alpha)^2\mathbb{E}_{Z_1^s}\Big[(1+\|Z_1^s\|^2)^4\Big]+  (\alpha)^2\mathbb{E}_{Z_1^t}\Big[(1+\|Z_1^t\|^2)^2\Big]\mathbb{E}_{Z_1^s}\Big[(1+\|Z_1^s\|^2)^2\Big]\Bigg]\\
    &\quad+ L_m(1-\alpha)\Disnd(\nupop^s,\nupop^t)\big(2+\mbE_{\theta\sim \bar{m}_{\alpha}}[\|\theta\|^4]  \big)\big(2+\mathbb{E}_{\mbZns\mbZnt}\big[\mbE_{\theta\sim \mdens^{\beta}(\nualpha)}[\|\theta\|^4]\big]\big)
    \\&\quad+  \frac{\sigma^2}{2\beta^2} \KLr(\bar{m}_\alpha\|\gamma^{\sigma})\\
    &\quad+L_e\mbE[(1+\|Z_1^t\|^2)^2]\Disfth(\bar{m}_{\alpha},\bar{m}^t)\\
&\leq\frac{c_t\alpha}{n_t}\frac{2\beta^2}{\sigma^2} \comp(\theta)\Bigg[ (2+\alpha)^2\mathbb{E}_{Z_1^t}\Big[(1+\|Z_1^t\|^2)^4\Big]+  (1-\alpha)^2\mathbb{E}_{Z_1^t}\Big[(1+\|Z_1^t\|^2)^2\Big]\mathbb{E}_{Z_1^s}\Big[(1+\|Z_1^s\|^2)^2\Big]\Bigg]\\
    &\quad+ \frac{c_s(1-\alpha)}{n_s}\frac{2\beta^2}{\sigma^2} \comp(\theta)\Bigg[ (3-\alpha)^2\mathbb{E}_{Z_1^s}\Big[(1+\|Z_1^s\|^2)^4\Big]+  (\alpha)^2\mathbb{E}_{Z_1^t}\Big[(1+\|Z_1^t\|^2)^2\Big]\mathbb{E}_{Z_1^s}\Big[(1+\|Z_1^s\|^2)^2\Big]\Bigg]\\
    &\quad+L_m (1-\alpha)\Disnd(\nupop^s,\nupop^t)\\
    &\qquad\times \Big(2+\mbE_{\theta\sim \bar{m}_{\alpha}}[\|\theta\|^4] \Big)\Big(2+\int_\Theta \|\theta\|^4\tilgammap(\mrd\theta)+g(\tilgammap)\Big(\alpha \mathbb{E}_{Z_1^t}[(1+\|Z_1^t \|^2)]+(1-\alpha) \mathbb{E}_{Z_1^s}[(1+\|Z_1^s \|^2)]\Big) \Big)\\
    &\quad+  \frac{\sigma^2}{2\beta^2} \KLr(\bar{m}_\alpha\|\gamma^{\sigma})\\
    &\quad+L_e\mbE[(1+\|Z_1^t\|^2)^2]\Disfth(\bar{m}_{\alpha},\bar{m}^t)\\
&\leq\frac{c_t\alpha}{n_t}\frac{2\beta^2}{\sigma^2}\comp(\theta) \Bigg[ (2+\alpha)^2\mathbb{E}_{Z_1^t}\Big[(1+\|Z_1^t\|^2)^4\Big]+  (1-\alpha)^2\mathbb{E}_{Z_1^t}\Big[(1+\|Z_1^t\|^2)^2\Big]\mathbb{E}_{Z_1^s}\Big[(1+\|Z_1^s\|^2)^2\Big]\Bigg]\\
    &\quad+ \frac{c_s(1-\alpha)}{n_s}\frac{2\beta^2}{\sigma^2} \comp(\theta)\Bigg[ (3-\alpha)^2\mathbb{E}_{Z_1^s}\Big[(1+\|Z_1^s\|^2)^4\Big]+  (\alpha)^2\mathbb{E}_{Z_1^t}\Big[(1+\|Z_1^t\|^2)^2\Big]\mathbb{E}_{Z_1^s}\Big[(1+\|Z_1^s\|^2)^2\Big]\Bigg]\\
    &\quad+ L_m(1-\alpha)\Disnd(\nupop^s,\nupop^t) \Big(2+\mbE_{\theta\sim \bar{m}_{\alpha}}[\|\theta\|^4] \Big)\\
    &\qquad\times\Big(2+(1+L_m)\Big(1+\int_\Theta \|\theta\|^4\tilgammap(\mrd\theta)\Big)\Big(1+\alpha \mathbb{E}_{Z_1^t}[(1+\|Z_1^t \|^2)]+(1-\alpha) \mathbb{E}_{Z_1^s}[(1+\|Z_1^s \|^2)]\Big)  \Big) \\
    &\quad+  \frac{\sigma^2}{2\beta^2} \KLr(\bar{m}_\alpha\|\gamma^{\sigma})\\
    &\quad+L_e\mbE[(1+\|Z_1^t\|^2)^2]\Disfth(\bar{m}_{\alpha},\bar{m}^t)\\
&\leq\frac{c_t\alpha}{n_t}\frac{2\beta^2}{\sigma^2}\comp(\theta) \Bigg[ (2+\alpha)^2\mathbb{E}_{Z_1^t}\Big[(1+\|Z_1^t\|^2)^4\Big]+  (1-\alpha)^2\mathbb{E}_{Z_1^t}\Big[(1+\|Z_1^t\|^2)^2\Big]\mathbb{E}_{Z_1^s}\Big[(1+\|Z_1^s\|^2)^2\Big]\Bigg]\\
    &\quad+ \frac{c_s(1-\alpha)}{n_s}\frac{2\beta^2}{\sigma^2} \comp(\theta)\Bigg[ (3-\alpha)^2\mathbb{E}_{Z_1^s}\Big[(1+\|Z_1^s\|^2)^4\Big]+  (\alpha)^2\mathbb{E}_{Z_1^t}\Big[(1+\|Z_1^t\|^2)^2\Big]\mathbb{E}_{Z_1^s}\Big[(1+\|Z_1^s\|^2)^2\Big]\Bigg]\\
    &\quad+ 6L_m(1+L_m)(1-\alpha)\Disnd(\nupop^s,\nupop^t)\\
    &\qquad\times \Big(1+\mbE_{\theta\sim \bar{m}_{\alpha}}[\|\theta\|^4] \Big)\Big(1+\int_\Theta \|\theta\|^4\tilgammap(\mrd\theta)\Big)\Big(1+\alpha \mathbb{E}_{Z_1^t}[(1+\|Z_1^t \|^2)]+(1-\alpha) \mathbb{E}_{Z_1^s}[(1+\|Z_1^s \|^2)]\Big) \\
    &\quad+  \frac{\sigma^2}{2\beta^2} \KLr(\bar{m}_\alpha\|\gamma^{\sigma})\\
    &\quad+L_e\mbE[(1+\|Z_1^t\|^2)^2]\Disfth(\bar{m}_{\alpha},\bar{m}^t).
\end{align*}

It completes the proof.
 \end{proof}
 \blue

\subsection{Fine-tuning details and proofs}\label{app: fine tune}
We first provide the similar results to Lemma~\ref{lem:optimizerinBnu}, Lemma~\ref{lem:optimizerinBnu} and Proposition~\ref{Prop_reccurence_estimate} for fine-tuning scenario.
\begin{lemma}\label{lem:optimizerinBnuf}
    Defining the set valued map 
        \[\begin{split}B_f(m_{\mathrm{c}},\nu) &:= \Big\{m_{\spec}\in \mathcal{P}_8(\Thetas): \frac{\sigma^2}{2\beta_t^2}\mathrm{KL}(m_{\spec}\|\gamma_{\spec}^\sigma)\leq \mrR(m_{\mathrm{c}}\gamma_{\spec}^\sigma, \nu)\\
        &\qquad \text{ and } \int_{\Thetas} \|\thetas^t\|^8 m_{\spec}(\mrd \thetas^t)\leq \mrR(m_{\mathrm{c}}\tilgammasp,\nu) + \int_{\Thetas} \|\thetas^t\|^8\tilgammasp(\mrd\thetas^t)\Big\}\subset \mathcal{P}_8(\Thetas),\end{split}\]
        for any $\nu\in \mathcal{P}_2(\mcZ)$, the KL-regularized risk minimizer satisfies $\mfm^{\beta_t}(m_c,\nu) \in B(m_c,\nu)$.
\end{lemma}

\begin{lemma}\label{lem:optimizerin-Bnus-shared}
    Defining the set valued map 
        \[\begin{split}&B_f(\nu) := \Big\{m\in \mathcal{P}_q(\Theta): \frac{\sigma^2}{2\beta_s^2}\mathrm{KL}(m\|\gamma_{p}^\sigma)\leq \mrR(\gamma_{p}^{\sigma}, \nu),\\
        &\quad\int_{\Thetac} \|\thetac\|^q m_{\mathrm{c}}(\mrd \thetac)\leq \mrR(\hat{\gamma}_{p}^\sigma,\nu) + \int_{\Thetac} \|\thetac\|^q\hat{\gamma}_{p,\mathrm{c}}^\sigma(\mrd\thetac),\\
        &\hat{\gamma}_{p,\mathrm{c}}^\sigma:=\int_{\Thetas}\hat{\gamma}_p^{\sigma}(\mrd \thetas^s)\quad \text{ and } m_{\mathrm{c}}:=\int_{\Thetas}m(\mrd\thetas^t)\Big\}\subset \mathcal{P}_q(\Theta) ,\end{split}\]
        for any $\nu\in \mathcal{P}_2(\mcZ)$, $q=4,8$, the KL-regularized risk minimizer satisfies $\mfm^{\beta}(\nu) \in B_f(\nu)$.
\end{lemma}

\begin{proposition}\label{Prop_reccurence_estimatef} 
Suppose Assumption \ref{ass:KLreg_assnf} holds. Then
\[\begin{split}&\int_{\Theta}\Big[\frac{\delta S_f}{\delta\nu}(\nu^t,\mfmc(\nu^s),\thetas, z)-\frac{\delta S_f}{\delta\nu}(\nu^t,\mfmc(\nu^s),\thetas, z')\Big]^2\mfms^{\beta_t}(\mfmc(\nu^s),\nu^t; \mrd \thetas) \\
&\leq 8L_e^2\Big(1+\mbE_{\thetac\sim \mfmc(\nu^s)}[\|\thetac\|^8]+2\mbE_{\thetas\sim \mfms^{\beta_t}(\mfmc(\nu^s),\nu^t)}[\|\thetas\|^8]\Big).
\end{split}\]
where $S_f:=\frac{\delta R}{\delta m_{\spec}}$.
\end{proposition}
\begin{proof}[Proof of Proposition \ref{Prop_reccurence_estimatef}]
Using the similar approach to  approach to Proposition~\ref{Prop_reccurence_estimate} and linearity,
\begin{align*}
&\frac{\delta S_f}{\delta\nu}(\nu^t,\mfmc(\nu^s),\thetas, z^t)-\frac{\delta S_f}{\delta\nu}(\nu^t,\mfmc(\nu^s),\thetas, \tilde z^{t})\\
&=\Big(\mathrm{id}+\frac{2\beta^2}{\sigma^2}\mathcal{C}_{\mfms^{\beta_t}(\mfmc(\nu^s),\nu^t)}\Big)^{-1}\Big[\frac{\delta \ell}{\delta m_{\spec}}(\mfmc(\nu^s)\mfms^{\beta_t}(\mfmc(\nu^s),\nu^t),\cdot,  z^{t}) - \frac{\delta \ell}{\delta m_{\spec}}(\mfmc(\nu^s)\mfms^{\beta_t}(\mfmc(\nu^s),\nu^t),\cdot, \tilde z^{t})\Big].
\end{align*}
where $\mathcal{C}_{\mfms^{\beta_t}(\mfmc(\nu^s),\nu^t)}:L^2(m_{\spec},\Thetas) \to L^2(m_{\spec},\Thetas)$ is defined as,
\[\mathcal{C}_{\mfms^{\beta_t}(\mfmc(\nu^s),\nu^t)}f(\thetas^t):=\mathbb{E}_{\thetas'\sim \mfms^{\beta_t}(\mfmc(\nu^s),\nu^t)}\Big[\frac{\delta^2 R}{\delta m_{\spec}^2}(\mfmc(\nu^s)\mfms^{\beta_t}(\mfmc(\nu^s),\nu^t),\nu^t,\theta_{\spec}^t,\theta_{\spec}^{t,\prime})\Big].\]
Note that this is bounded in $L^2(\mfm_{\spec}^{\beta}(\mfmc(\nu^s),\nu^t))$, in particular
\[\begin{split}
        &\int_{\Theta_{\spec}}\Big[\frac{\delta S_f}{\delta\nu}(\nu^t,\mfmc(\nu^s),\thetas,  z^{t})-\frac{\delta S_f}{\delta\nu}(\nu^t,\mfmc(\nu^s),\thetas^t, \tilde z^{t})\Big]^2\mfms^{\beta_t}(\mfmc(\nu^s),\nu^t; \mrd \thetas^t) \\
    &\leq \int_{\Theta_{\spec}}\Big[\frac{\delta \ell}{\delta m_{\spec}}(\mfmc(\nu^s)\mfms^{\beta}(\mfmc(\nu^s),\nu^t),\cdot,  z^{t}) - \frac{\delta \ell}{\delta m_{\spec}}(\mfmc(\nu^s)\mfm^{\beta_t}(\mfmc(\nu^s),\nu^t),\cdot, \tilde z^{t})\Big]^2\mfms^{\beta_t}(\mfmc(\nu^s),\nu^t; \mrd \thetas^t)\\
    &\leq 2\int_{\Theta_{\spec}}\Big[\frac{\delta \ell}{\delta m_{\spec}}(\mfmc(\nu^s)\mfms^{\beta}(\mfmc(\nu^s),\nu^t),\cdot,  z^{t})\Big]^2 \\&\qquad+ \Big[\frac{\delta \ell}{\delta m_{\spec}}(\mfmc(\nu^s)\mfms^{\beta}(\mfmc(\nu^s),\nu^t),\cdot, \tilde z^{t})\Big]^2\mfms^{\beta_t}(\mfmc(\nu^s),\nu^t; \mrd \thetas^t).
\end{split}\]

Assumption \ref{ass:KLreg_assnf}(\ref{ass:KLreg_assn_item3f}) yields
\[\begin{split}&\int_{\Theta_{\spec}}\Big[\frac{\delta S_f}{\delta\nu}(\nu^t,\mfmc(\nu^s),\thetas^t,  z^{t})-\frac{\delta S_f}{\delta\nu}(\nu^t,\mfmc(\nu^s),\thetas^t, \tilde z^{t})\Big]^2\mfms^{\beta_t}(\mfmc(\nu^s),\nu^t; \mrd \thetas^t) \\
    &\leq 2\Big(2+\| z^{t}\|^2+\|\tilde z^{t}\|^2\Big)^2\int_{\Theta_{\spec}} \big[g_{\spec}(\mfmc(\nu^s)\mfms^{\beta_t}(\mfmc(\nu^s),\nu^t), \thetas^t)\big]^2\mfms^{\beta_t}(\mfmc(\nu^s),\nu^t; \mrd \thetas^t)\\
    &\leq 8L_e\Big(2+\| z^{t}\|^2+\|\tilde z^{t}\|^2\Big)^2
    \Big(1+\mbE_{\thetac\sim \mfmc(\nu^s)}[\|\thetac\|^8]+2\mbE_{\thetas^t\sim \mfms^{\beta_t}(\mfmc(\nu^s),\nu^t)}[\|\thetas^t\|^8]\Big).
\end{split}\]
where $g_{\spec}(\mfmc(\nu^s)\mfms^{\beta_t}(\mfmc(\nu^s),\nu^t), \thetas^t)=L_e(1+\mbE_{\thetac\sim \mfmc(\nu^s)}[\|\thetac\|^4]+\mbE_{\thetas^t\sim \mfms^{\beta_t}(\mfmc(\nu^s),\nu^t)}[\|\thetas^t\|^4]+\|\thetas^t\|^4)$. 
The final result follows from Assumption \ref{ass:KLreg_assnf}(\ref{ass:KLreg_D2loss_growth_bound_FT}).

\end{proof}
\begin{tcolorbox}
\begin{reptheorem}{thm: WTGE fine tune}[\textbf{Full Version}]
    Given Assumption \ref{ass:KLreg_assnf}, $\mathbb{E}_{Z\sim\nupop^t}\big[\|Z\|^{8}\big]<\infty$ and $\mathbb{E}_{Z\sim\nupop^s}\big[\|Z\|^{4}\big]<\infty$, the weak transfer generalization error under fine-tuning satisfies 
\[ \begin{split} &|\mathrm{gen}(\mfm_{\mathrm{c}}(\nus)\mfm_{\spec}^{t,\beta_t}(\mfm_{\mathrm{c}}(\nus),\nut),\nut),\nupop^t)|\leq  \frac{2}{n_t}(1+\frac{2}{n_t})^2\frac{16\beta^2}{\sigma^2}  L_e^2(1+L_m)^2\comp(\thetac,\thetas^t,\thetas^s)
 \\
 & \times \mathbb{E}_{Z_1^s}\Big[(1+\|Z_1^s\|^2)^2\Big] \mathbb{E}_{Z_1^t}\Big[(1+\|Z_1^t\|^2)^4\Big],\end{split} \]
 \end{reptheorem}
 where \begin{equation*}
     \begin{split}
          \comp(\thetac,\thetas^t,\thetas^s)&=\Big[1+\int_{\Thetac} \|\thetac\|^4(2+\|\thetac\|^4)\gammac(\mrd\thetac)\\
&\quad+\int_{\Thetas}\|\thetas^t\|^4(1+ 2\|\thetas^t\|^4)\tilgammasp(\mrd \thetas^t)+\int_{\Thetas}\|\thetas^s\|^4\gammasp(\mrd \thetas^s)\Big]^2.
     \end{split}
 \end{equation*}
\end{tcolorbox}
  
\begin{proof}[Proof of Theorem \ref{thm: WTGE fine tune}]

Recall from \eqref{mperturb1} that $\nu_{(1)}^t(\lambda)=\nu_{n_t,(1)}^t +\lambda(\nu_{n_t}^t-\nu_{n_t,(1)}^t)$. From Theorem~\ref{thm: another rep fine-tune},
\begin{equation}\label{FWGEKLbasic}
    \mathrm{gen}(\mfmc\mfms,\nupop^t)  =\frac{1}{n_t}\mathbb{E}_{\mbZnt,\mbZns,\ols{Z}_1^t}\big[ K(\mbZnt,\mbZns,\ols{Z}_1^t)\big] ,
\end{equation}

where 
\begin{equation}\label{eq:Fhrepeat}\begin{split}
&K(\mbZnt,\mbZns,\ols{Z}_1^t)=\int_{0}^1 \int_{\Theta_{\spec}}\Big( \frac{\delta \ell}{\delta m_{\spec}}\big(\mfm_{(1)}^f(\lambda),\ols{Z}_1^t,\theta_{\spec}\big)\Big) \\
 &\qquad \times\Bigg(\int_{0}^1\int_{\mcZ}\frac{\delta\mfms}{\delta \nu}(\mfmc(\nus),\nu_{(1)}^t(\lambda_1),z)(\mrd \thetas)(\delta_{Z_1^t}-\delta_{\ols{Z}_1^t})(\mrd z) \mrd \lambda_1\Bigg)\,\mrd \lambda,
\end{split}\end{equation}
As $\mathbb{E}_{\thetas\sim \mfms^{\beta_t}(\mfmc(\nus),\nu_{(1)}^t(\tilde\lambda))}\big[\frac{\delta S_f}{\delta \nu}\big(\mfmc(\nus),\nu_{(1)}^t(\tilde\lambda), z, \thetas\big)\big]\equiv 0$ (by definition of $S_f= \delta \mrR/\delta m_{\spec}$ and the normalization condition and chain rule), inspired by Lemma~\ref{lem_parameter_measure_derivative} we know 
\begin{equation}\label{FWGEKLexp1}
    \begin{split}
    &\int_{\Theta_{\spec}} \Big( \frac{\delta \ell}{\delta m_{\spec}}\big(\mfm_{(1)}^f(\lambda),Z_1^t,\thetas\big)\Big)\\&\qquad\times\Big(\frac{\delta \mfms}{\delta \nu}\big(\mfmc(\nus),\nu_{(1)}^t(\tilde{\lambda}),Z_1^t\big)-\frac{\delta \mfms}{\delta \nu}\big(\mfmc(\nus),\nu_{(1)}^t(\tilde{\lambda}),\ols{Z}_1^t\big)\Big)(\mrd \thetas) \\
 & = -\frac{2\beta_t^2}{\sigma^2}\mathrm{Cov}_{\thetas\sim \mfms^{\beta_t}(\mfmc(\nus),\nu_{(1)}^t(\tilde\lambda))}\Big[ \frac{\delta \ell}{\delta m_{\spec}}\big(\mfm_{(1)}^f(\lambda),Z_1^t,\thetas\big)\\
 &\qquad\qquad,\, \frac{\delta S_f}{\delta \nu}\big(\mfmc(\nus),\nu_{(1)}^t(\tilde\lambda), Z_1^t, \thetas\big)-\frac{\delta S_f}{\delta \nu}\big(\mfmc(\nus),\nu_{(1)}^t(\tilde\lambda), \tilde Z_1^t, \thetas\big)\Big]\\
 &\leq \frac{2\beta_t^2}{\sigma^2}\mathbb{E}_{\thetas\sim \mfms^{\beta_t}(\mfmc(\nus),\nu_{(1)}^t(\tilde\lambda))}\Big[\Big(\frac{\delta \ell}{\delta m_{\spec}}\big(\mfm_{(1)}^f(\lambda),Z_1^t,\thetas\big)\Big)^2\Big]^{1/2}\\
 &\qquad \times \mathbb{E}_{\thetas\sim \mfms^{\beta_t}(\mfmc(\nus),\nu_{(1)}^t(\tilde\lambda))}\Big[\Big(\frac{\delta S_f}{\delta \nu}\big(\mfmc(\nus),\nu_{(1)}^t(\tilde\lambda), Z_1^t, \thetas\big)-\frac{\delta S_f}{\delta \nu}\big(\mfmc(\nus),\nu_{(1)}^t(\tilde\lambda), \tilde Z_1^t, \thetas\big)\Big)^2\Big]^{1/2}.
\end{split}
\end{equation}
Using Proposition \ref{Prop_reccurence_estimatef} and \eqref{eq: bounded g assf},
\begin{equation}
    \begin{split}
    &\mathbb{E}_{\thetas\sim \mfms^{\beta_t}(\mfmc(\nus),\nu_{(1)}^t(\tilde\lambda))}\Big[\Big(\frac{\delta S_f}{\delta \nu}\big(\mfmc(\nus),\nu_{(1)}^t(\tilde\lambda), Z_1^t, \thetas\big)-\frac{\delta S_f}{\delta \nu}\big(\mfmc(\nus),\nu_{(1)}^t(\tilde\lambda), \tilde Z_1^t, \thetas\big)\Big)^2\Big]^{1/2}\\
    &\leq 2\sqrt{2} L_e\Big(1+ 2\mathbb{E}_{\thetas \sim \mfms^{\beta_t}(\mfmc^{\beta_s}(\nus),\nu_{(1)}^t(\tilde\lambda))}\big[\|\thetas\|^8\big]+\mathbb{E}_{\thetac \sim \mfmc^{\beta_s}(\nus)}\big[\|\thetac\|^8\big]\Big)\Big(2+\|Z_1^t\|^2+\|\ols{Z}_1^t\|^2\Big),
\end{split}\end{equation}

From Lemma \ref{lem:optimizerinBnuf}, we know 
\[\mathbb{E}_{\thetas \sim \mfms^{\beta_t}(\mfmc(\nus),\nu_{(1)}^t(\tilde\lambda))}\big[\|\thetas\|^8\big]\leq \mrR(\mfmc(\nus)\tilgammasp, \nu_{(1)}^t(\tilde\lambda)) + \int_{\Theta_{\spec}}\|\thetas\|^8\tilgammasp(\mrd \thetas^t),\]
and from Lemma~\ref{lem:optimizerin-Bnus-shared}, we have,
\[\mathbb{E}_{\thetac \sim \mfmc^{\beta_s}(\nus)}\big[\|\thetac\|^8\big]\leq \mrR(\hat\gamma_{p}^\sigma, \nus) + \int_{\Thetac}\|\thetac\|^8\gammac(\mrd \thetac),\]
where $\gamma_{p,\mathrm{c}}^\sigma(\mrd \thetac)=\int_{\Thetas}\hat\gamma_{p}^\sigma(\mrd \thetas).$
By construction,
\[\nu_{(1)}^t(\tilde\lambda) = \nut +\frac{1-\tilde \lambda}{n_t}\big(\delta_{\ols{Z}_1^t} - \delta_{Z_1^t}\big),\]
so 
\begin{equation}\label{eq:FRboundperturb1}
    \begin{split}
    \mrR(\mfmc(\nus)\tilgammasp, \nu_{(1)}^t(\tilde\lambda)) &= \mrR(\mfmc(\nus)\tilgammasp, \nut) + \frac{1-\tilde \lambda}{n_t}\Big(\ell(\mfmc(\nus)\tilgammasp, \ols Z_1^t) - \ell(\mfmc(\nus)\tilgammasp, Z_1^t)\Big)\\
    &\leq \mrR(\mfmc(\nus)\tilgammasp, \nut) + g(\mfmc(\nus)\tilgammasp)\frac{1-\tilde \lambda}{n_t}\Big(2+\|\ols{Z}_1^t\|^2+ \|Z_1^t\|^2)\Big).
\end{split}
\end{equation}
We also have,
\begin{equation}
    \begin{split}
       \mrR(\hat\gamma_{p}^\sigma, \nus)\leq g(\hat\gamma_{8,c}^\sigma \gammasp) \frac{1}{n_s}\sum_{j=1}^{n_s} (1+\norm{Z_j^s}^2),
    \end{split}
\end{equation}
and hence 
\begin{equation}
    \label{FWGEKLexp2}
    \begin{split}
    &\mathbb{E}_{\thetas\sim \mfms^{\beta_t}(\mfmc(\nus),\nu_{(1)}^t(\tilde\lambda))}\Big[\Big(\frac{\delta S_f}{\delta \nu}\big(\mfmc(\nus),\nu_{(1)}^t(\tilde\lambda), Z_1^t, \thetas\big)-\frac{\delta S_f}{\delta \nu}\big(\mfmc(\nus),\nu_{(1)}^t(\tilde\lambda), \tilde Z_1^t, \thetas\big)\Big)^2\Big]^{1/2}\\
    &\leq 2\sqrt{2} L_e\Big(1+ 2\mrR(\mfmc(\nus)\tilgammasp, \nut) +2\int_{\Theta_{\spec}}\|\thetas\|^8\tilgammasp(\mrd \thetas^t)\\
    &\qquad+2g(\mfmc(\nus)\tilgammasp)\frac{1-\tilde \lambda}{n_t}\Big(2+ \|Z_1^t\|^2+\|\ols{Z}_1^t\|^2\Big)+g(\hat\gamma_{8,c}^\sigma \gammasp) \frac{1}{n_s}\sum_{j=1}^{n_s} (1+\norm{Z_j^s}^2)+\int_{\Thetac}\|\thetac\|^8\gammac(\mrd \thetac)\Big)\\&\quad \times\Big(2+\|Z_1^t\|^2+\|\ols{Z}_1^t\|^2\Big).
\end{split}\end{equation}

By definition, $\mfm_{(1)}^f(\lambda) = \mfmc(\nus)\mfms(\mfmc(\nus),\nu^t_{(1)})(\lambda),$ where $\mfms(\mfmc(\nus),\nu^t_{(1)})(\lambda):=\mfms(\mfmc(\nus),\nutr) + \lambda(\mfms(\mfmc(\nus),\nut) - \mfms(\mfmc(\nus),\nutr))$, we also have 
\begin{equation}
    \label{FWGEKLexp3}
    \begin{split}
    &\mathbb{E}_{\thetas\sim \mfms^{\beta_t}(\mfmc(\nus),\nu_{(1)}^t(\tilde\lambda))}\Big[\Big(\frac{\delta \ell}{\delta m_{\spec}}\big(\mfm_{(1)}^f(\lambda),Z_1^t,\thetas\big)\Big)^2\Big]^{1/2}\\
    &\leq \mathbb{E}_{\thetas\sim \mfms^{\beta_t}(\mfmc(\nus),\nu_{(1)}^t(\tilde\lambda))}\Big[g_{\spec}(\mfm_{(1)}^f(\lambda),\thetas)^2\Big]^{1/2} \big(1+\|Z_1^t\|^2\big)\\
    &\leq\sqrt{2} L_e \Big(1+ \mathbb{E}_{\thetas\sim \mfms^{\beta_t}(\mfmc(\nus),\nu_{(1)}^t(\tilde\lambda))}\big[\|\thetas\|^8\big]\\
   &\qquad+ \mathbb{E}_{\thetas\sim \mfms^{\beta_t}(\mfmc(\nus),\nu_{(1)}^t)(\lambda)}\big[\|\thetas\|^8\big]+\mathbb{E}_{\thetac\sim \mfmc(\nus)}\big[\|\thetac\|^8\big]\Big)\big(1+\|Z_1^t\|^2\big).
    \end{split}
\end{equation}
 and hence
\[\mathbb{E}_{\mfms^{\beta_t}(\mfmc(\nus),\nu_{(1)}^t)(\lambda)}\big[\|\thetas\|^8\big] = (1-\lambda)\mathbb{E}_{\thetas\sim \mfms(\mfmc(\nus),\nutr)}\big[\|\thetas\|^8\big]+ \lambda\mathbb{E}_{\thetas\sim \mfms(\mfmc(\nus),\nut)}\big[\|\thetas\|^8\big].\]

As before, it follows from Lemma \ref{lem:optimizerinBnuf} that
\begin{equation}
\label{Fmomentbound2}
\begin{split}
&\mathbb{E}_{\mfms^{\beta_t}(\mfmc(\nus),\nu_{(1)}^t)(\lambda)}\big[\|\thetas\|^8\big]+\mathbb{E}_{\mfms^{\beta_t}(\mfmc(\nus),\nu_{(1)}^t(\tilde\lambda))}\big[\|\thetas\|^8\big]\\
& = \mathbb{E}_{\thetas\sim \mfms^{\beta_t}(\mfmc(\nus),\nu_{(1)}^t(\tilde\lambda))}\big[\|\thetas\|^8\big]+(1-\lambda)\mathbb{E}_{\thetas\sim \mfms(\mfmc(\nus),\nutr)}\big[\|\thetas\|^8\big]+ \lambda\mathbb{E}_{\thetas\sim \mfms(\mfmc(\nus),\nut)}\big[\|\thetas\|^8\big]\\
&\leq \Big[\mrR(\mfmc(\nus)\tilgammasp, \nu_{(1)}^t(\tilde\lambda)) + \int_{\Theta_{\spec}}\|\thetas\|^8\tilgammasp(\mrd \thetas^t)\Big] + (1-\lambda)\Big[\mrR(\mfmc(\nus)\tilgammasp,\nutr) + \int_{\Theta_{\spec}}\|\thetas\|^8\tilgammasp(\mrd \thetas^t) \Big] \\
& \qquad + \lambda\Big[\mrR(\mfmc(\nus)\tilgammasp, \nut) +\int_{\Theta_{\spec}}\|\thetas\|^8\tilgammasp(\mrd \thetas^t)\Big]\\
    &\leq 2\mrR(\mfmc(\nus)\tilgammasp, \nut) + 2\int_{\Thetas}\|\thetas\|^8\tilgammasp(\mrd \theta) + g(\mfmc(\nus)\tilgammasp)\frac{(2-\tilde \lambda-\lambda)}{n_t}\Big(2+\|\ols{Z}_1^t\|^2+ \|Z_1^t\|^2\Big)\\
  &\leq 2\mrR(\mfmc(\nus)\tilgammasp, \nut) + 2\int_{\Thetas}\|\thetas\|^8\tilgammasp(\mrd \theta) \\
   &\quad+\Big( 1+ g(\hat\gamma_{8,c}^\sigma \gammasp)\frac{1}{n_s}\sum_{j=1}^{n_s}\big(1+\|Z_j^s\|^2\big) + \int_{\Thetac} \|\thetac\|^4\gammac(\thetac)\mrd\thetac+\mathbb{E}_{\thetas\sim\tilgammasp}[\|\thetas\|^4]\Big) \\&\qquad\times\frac{(2-\tilde \lambda-\lambda)}{n_t}\Big(2+\|\ols{Z}_1^t\|^2+ \|Z_1^t\|^2\Big) .
\end{split} 
\end{equation}
It also follows from Lemma~\ref{lem:optimizerin-Bnus-shared},
\begin{equation}
    \begin{split}
        &\mathbb{E}_{\thetac\sim \mfmc(\nus)}\big[\|\thetac\|^4\big]\\
        &\leq  \mrR(\hat\gamma_{p}^\sigma,\nus) + \int_{\Thetac} \|\thetac\|^4\gammac(\thetac)\mrd\thetac\\
        &\leq g(\hat\gamma_{8,c}^\sigma \gammasp)\frac{1}{n_s}\sum_{j=1}^{n_s}(1+\|Z_j^s\|^2)+\int_{\Thetac} \|\thetac\|^4\gammac(\thetac)\mrd\thetac.
    \end{split}
\end{equation}

We also have the following bound on $\mrR(\mfmc(\nus)\tilgammasp, \nut) $ from Lemma~\ref{lem:optimizerin-Bnus-shared},
\begin{equation}\label{eq:FRboundsum}
\begin{split}
    &\mrR(\mfmc(\nus)\tilgammasp, \nut) \\&\leq \int_\mcZ g(\mfmc(\nus)\tilgammasp) (1+\|z\|^2)\nut(\mrd{z})\\
    &= g(\mfmc(\nus)\tilgammasp)\mbE\Big[\frac{1}{n_t}\sum_{j=1}^{n_t}(1+\|Z_j^t\|^2)\Big]\\
    &\leq L_m\big(1+\mathbb{E}_{\thetac\sim\mfmc(\nus)}[\|\thetac\|^4]+\mathbb{E}_{\thetas\sim\tilgammasp}[\|\thetas\|^4]\big)\Big(\frac{1}{n_t}\sum_{j=1}^{n_t}(1+\|Z_j^t\|^2)\Big)\\
    &\leq L_m\Big( 1+ \mrR(\hat\gamma_{p}^\sigma,\nus) + \int_{\Thetac} \|\thetac\|^4\gammac(\mrd\thetac)+\mathbb{E}_{\thetas\sim\tilgammasp}[\|\thetas\|^4]\Big)\Big(\frac{1}{n_t}\sum_{j=1}^{n_t}(1+\|Z_j^t\|^2)\Big)\\
    &\leq L_m\Big( 1+ g(\hat\gamma_{8,c}^\sigma\gammasp)\frac{1}{n_s}\sum_{j=1}^{n_s}\big(1+\|Z_j^s\|^2\big) + \int_{\Thetac} \|\thetac\|^4\gammac(\thetac)\mrd\thetac+\mathbb{E}_{\thetas\sim\tilgammasp}[\|\thetas\|^4]\Big)\Big(\frac{1}{n_t}\sum_{j=1}^{n_t}(1+\|Z_j^t\|^2)\Big) .
    \end{split}
\end{equation}

Therefore, substituting \eqref{FWGEKLexp1}, \eqref{FWGEKLexp2}, \eqref{FWGEKLexp3} and \eqref{Fmomentbound2} into \eqref{eq:Fhrepeat} and simplifying and denoting $A_c:=\int_{\Thetac} \|\thetac\|^8\gammac(\mrd\thetac)$, $B_p:=\int_{\Thetas}\|\thetas\|^8\tilgammasp(\mrd \thetas^t),$ $A_{c,4}=\int_{\Thetac} \|\thetac\|^4\gammac(\thetac)\mrd\thetac$ and $B_{p,4}=\int_{\Thetas}\|\thetas\|^4\tilgammasp(\mrd \thetas^t)$, we have,
 \begin{align*}
&k(\mbZnt,\mbZns,\ols{Z}_1^t,\lambda,\tilde\lambda)   \\
&\leq \frac{2\beta^2}{\sigma^2} 2\sqrt{2} L_e\Big(1+ 2\mrR(\mfmc(\nus)\tilgammasp, \nut) +2\int_{\Theta_{\spec}}\|\thetas\|^8\tilgammasp(\mrd \thetas^t)\\
         &\qquad+2g(\mfmc(\nus)\tilgammasp)\frac{1-\tilde \lambda}{n_t}\Big(2+ \|Z_1^t\|^2+\|\ols{Z}_1^t\|^2\Big)+2g(\hat\gamma_{8,c}^\sigma \gammasp) \frac{1}{n_s}\sum_{j=1}^{n_s} (1+\norm{Z_j^s}^2)+\int_{\Thetac}\|\thetac\|^8\gammac(\mrd \thetac)\Big)\\&\quad \times\Big(2+\|Z_1^t\|^2+\|\ols{Z}_1^t\|^2\Big)\\
         &\qquad \times L_e \Big(1+ \mathbb{E}_{\thetas\sim \mfms^{\beta_t}(\mfmc(\nus),\nu_{(1)}^t(\tilde\lambda))}\big[\|\thetas\|^8\big]\\
        &\qquad+ \mathbb{E}_{\thetas\sim \mfms^{\beta_t}(\mfmc(\nus),\nu_{(1)}^t)(\lambda)}\big[\|\thetas\|^8\big]+\mathbb{E}_{\thetac\sim \mfmc(\nus)}\big[\|\thetac\|^8\big]\Big)\big(1+\|Z_1^t\|^2\big)\\
 &\leq \frac{2\beta^2}{\sigma^2} 8 L_e^2 \big(1+\|Z_1^t\|^2\big) \Big(2+\|Z_1^t\|^2+\|\ols{Z}_1^t\|^2\Big)\\
         &\quad \times \Bigg(1+ 2g\big(\hat\gamma_{8,c}^\sigma\gammasp\big)\frac{1}{n_s}\sum_{j=1}^{n_s}(1+\|Z_j^s\|^2)+A_c+2B_p+ 2G_1\Big(\frac{1}{n_t}\sum_{j=1}^{n_t}(1+\|Z_j^t\|^2)\Big) \\&\qquad+2G_1\frac{1-\tilde \lambda}{n_t}\Big(2+ \|Z_1^t\|^2+\|\ols{Z}_1^t\|^2\Big)\Bigg)\\
          &\quad \times \Bigg(1+2g\big(\hat\gamma_{8,c}^\sigma\gammasp\big)\frac{1}{n_s}\sum_{j=1}^{n_s} (1+\norm{Z_j^s}^2)+A_c+2B_p+ 2G_1\Big(\frac{1}{n_t}\sum_{j=1}^{n_t}(1+\|Z_j^t\|^2)\Big)\\&\qquad+G_1 \frac{(2-\tilde \lambda-\lambda)}{n_t}\Big(2+\|\ols{Z}_1^t\|^2+ \|Z_1^t\|^2\Big)\Bigg),
    \end{align*}
where $G_1=L_m\Big(1+g(\hat \gamma_{8,c}^{\sigma}\hat \gamma_{8,\spec}^{\sigma})\frac{1}{n_s}\sum_{j=1}^{n_s}(1+\|Z_j^s\|^2)+A_{c,4} +B_{p,4}\Big)$.

More simplifying with respect to $\lambda$ and $\tilde \lambda$,
\begin{equation}\label{Fhexpansionfinal1} 
\begin{split}
&K(\mbZnt,\mbZns,\ols{Z}_1^t)   \\
&\leq \frac{2\beta^2}{\sigma^2} 8 L_e^2 \big(1+\|Z_1^t\|^2\big) \Big(2+\|Z_1^t\|^2+\|\ols{Z}_1^t\|^2\Big)\\
         &\quad \times \Bigg(1+ 2g\big(\hat\gamma_{8,c}^\sigma\gammasp\big)\frac{1}{n_s}\sum_{j=1}^{n_s}(1+\|Z_j^s\|^2)+A_c+2B_p+ 2G_1\Big(\frac{1}{n_t}\sum_{j=1}^{n_t}(1+\|Z_j^t\|^2)\Big)\\&\qquad +G_1\frac{2}{n_t}\Big(2+ \|Z_1^t\|^2+\|\ols{Z}_1^t\|^2\Big)\Bigg)^2
         \end{split}
\end{equation}

And therefore, applying Lemma \ref{lem:polybound} and substituting \eqref{Fhexpansionfinal1} in \eqref{FWGEKLbasic},
\begin{equation}\label{FWGEKLexp5}
\begin{split}
 &|\mathrm{gen}(\mfmc(\nus)\mfms(\mfmc(\nus),\nut),\nupop^t)|
 \leq \frac{2}{n_t}(1+\frac{1}{n_t})^2\frac{16\beta^2}{\sigma^2}  L_e^2(1+L_m)^2\\
 &\times \Big[1+ g(\hat\gamma_{8,c}^\sigma \gammasp)+\int_{\Thetac} \|\thetac\|^8\gammac(\mrd\thetac)+2\int_{\Thetas}\|\thetas\|^8\tilgammasp(\mrd \thetas^t)+\int_{\Thetac} \|\thetac\|^4\gammac(\mrd\thetac)+\int_{\Thetas}\|\thetas\|^4\tilgammasp(\mrd \thetas^t)\Big]^2
 \\
 & \times \mathbb{E}_{Z_1^s}\Big[(1+\|Z_1^s\|^2)^2\Big] \mathbb{E}_{Z_1^t}\Big[(1+\|Z_1^t\|^2)^4\Big]\\
 &\leq \frac{2}{n_t}(1+\frac{2}{n_t})^2\frac{16\beta^2}{\sigma^2}  L_e^2(1+L_m)^2\\
 &\times \Big[1+\int_{\Thetac} \|\thetac\|^8\gammac(\mrd\thetac)+2\int_{\Thetas}\|\thetas\|^8\tilgammasp(\mrd \thetas^t)+2\int_{\Thetac} \|\thetac\|^4\gammac(\mrd\thetac)\\&\qquad+\int_{\Thetas}\|\thetas^t\|^4\tilgammasp(\mrd \thetas^t)+\int_{\Thetas}\|\thetas^s\|^4\gammasp(\mrd \thetas^s)\Big]^2
 \\
 & \times \mathbb{E}_{Z_1^s}\Big[(1+\|Z_1^s\|^2)^2\Big] \mathbb{E}_{Z_1^t}\Big[(1+\|Z_1^t\|^2)^4\Big].
\end{split}
\end{equation}
\end{proof}

\begin{tcolorbox}
    \begin{reptheorem}{thm: excess-fine-tune}[\textbf{Full version}]
    Under Assumption~\ref{ass:KLreg_assnf} and Assumption~\ref{ass:KLreg_assn}, the following upper bound holds on the WTER under fine-tuning scenario,
   \begin{equation}
       \begin{split}
           &\mathcal{E}(\mfmc^{\beta_s}(\nus)\mfms^{\beta_t}(\mfmc^{\beta_s}(\nus),\nut),\nupop^t)\\
           &\leq \frac{2\beta_t^2}{\sigma^2}  \frac{16}{n_t}(1+\frac{2}{n_t})^2L_e^2(1+L_m)^2\comp(\thetac,\thetas^t,\thetas^s) \mathbb{E}_{Z_1^s}\Big[(1+\|Z_1^s\|^2)^2\Big] \mathbb{E}_{Z_1^t}\Big[(1+\|Z_1^t\|^2)^4\Big]\\
    &\quad+\frac{\sigma^2}{2\beta_t^2}\KLr(\tilde{m}_{\spec}^t\| \tilde\gamma_{\spec}^{\sigma})\\
    &\quad+ \frac{2\beta_s^2}{\sigma^2} \frac{9\sqrt{2}}{n_s} (1+\frac{2}{n_s})^2 L_e^2 (1+ L_m)^2 \comp(\theta) \mathbb{E}_{Z_1^t}\Big[(1+\|Z_1^s\|^2)^4\Big]\\
    &\quad+\frac{\sigma^2}{2\beta_s^2}\KLr(\mrmbarc^s\otimes\mrmbarsp^s\|\gamma_c^\sigma\otimes\gamma_{\spec}^{\sigma})\\
    &\quad+4\Big(1+\int_{\Thetac} \|\thetac\|^4\hat\gamma_{8,c}^\sigma(\mrd\thetac)+\int_{\Thetas}\|\thetas^s\|^4\gammasp(\mrd \thetas^s)\Big)\\&\qquad\times\Big(1+\int_{\Thetac} \|\thetac\|^4\mrmbarc^{s}(\mrd \thetac)+\int_{\Thetas}\|\thetas^s\|^4\mrmbarsp^{s}(\mrd \thetas^s)
    +\int_{\Thetas}\|\thetas^t\|^4 \tilde{m}_{\spec}^t(\mrd \thetas^t)\Big)\mbE_{Z_1^s}\big[(1+\|Z_1^s\|^2)\big]
     \\&\qquad\times\Disnd(\nupop^s,\nupop^t) \\
    &\quad+L_e\mbE_{Z_1^s}[(1+\|Z_1^s\|^2)^2]\Disfth(\tilde{m}_{\spec}^s,\tilde{m}_{\spec}^t)\\
     &\quad + L_e\mbE_{Z_1^t}[(1+\|Z_1^t\|^2)^2]\Disfth(\mrmbarc^{s}\otimes\mrmbarsp^{s},\mrmbarc^{t}\otimes\mrmbarsp^{t}),
     \end{split} 
   \end{equation}
   where \[\begin{split}\comp(\theta)&=\Big[1+ 2 \int_{\Theta}\|\theta\|^8\tilgammap(\mrd \theta) +2\mbE_{\theta\sim\tilgammap}[\|\theta\|^4]\Big]^2,\\
   \comp(\thetac,\thetas^t,\thetas^s)&=\Big[1+\int_{\Thetac} \|\thetac\|^4(2+\|\thetac\|^4)\gammac(\mrd\thetac)\\
&\quad+\int_{\Thetas}\|\thetas^t\|^4(1+ 2\|\thetas^t\|^4)\tilgammasp(\mrd \thetas^t)+\int_{\Thetas}\|\thetas^s\|^4\gammasp(\mrd \thetas^s)\Big]^2,
    \\\tilde{m}_{\spec}^t&=\mathop{\arg\min}_{m_{\spec}\in\mathcal{P}_4(\Thetas)} R(\mfmc^{\beta_s}(\nus) m_{\spec},\nupop^t)\,,\end{split}\] provided that $\KLr(\tilde{m}_{\spec}^t\| \tilde\gamma_{\spec}^{\sigma})<\infty$, 
  
  \[\begin{split}\tilde{m}_{\spec}^s&=\mathop{\arg\min}_{m_{\spec}\in\mathcal{P}_4(\Thetas)} R(\mfmc^{\beta_s}(\nus) m_{\spec},\nupop^s),\\\mrmbarc^s\otimes\mrmbarsp^s&=\mathop{\arg\min}_{\mrmbarc^s\otimes\mrmbarsp^s\in\mathcal{P}_4(\Thetac)\times \mathcal{P}_4(\Thetas)} R(\mrmbarc^s\otimes\mrmbarsp^s,\nupop^s),\end{split}\] 
  
  and \[\mrmbarc^t\otimes\mrmbarsp^t=\mathop{\arg\min}_{\mrmbarc^t\otimes\mrmbarsp^t\in\mathcal{P}_4(\Thetac)\times \mathcal{P}_4(\Thetas)} R(\mrmbarc^t\otimes\mrmbarsp^t,\nupop^t).\]
\end{reptheorem}
\end{tcolorbox}

\begin{proof}[Proof of Theorem~\ref{thm: excess-fine-tune}]
    \begin{equation}\label{eq: ex fine}
   \begin{split}
       &\mathcal{E}(\mfmc^{\beta_s}(\nus)\mfms^{\beta_t}(\mfmc^{\beta_s}(\nus),\nut),\nupop^t)\\
       &=\mathbb{E}_{\mbZnt,\mbZns}[\mrR(\mfmc^{\beta_s}(\nus)\mfms^{\beta_t}(\mfmc^{\beta_s}(\nus),\nut),\nupop^t)]-\mrR(\mrmbarc^{t}\otimes\mrmbarsp^{t},\nupop^t)
       \\
       &= \underbrace{\mathbb{E}_{\mbZnt,\mbZns}[\mrR(\mfmc^{\beta_s}(\nus)\mfms^{\beta_t}(\mfmc^{\beta_s}(\nus),\nut),\nupop^t)] 
        - \mathbb{E}_{\mbZnt,\mbZns}[\mrR(\mfmc^{\beta_s}(\nus)\mfms^{\beta_t}(\mfmc^{\beta_s}(\nus),\nut),\nut)]}_{I_1}
       \\ &\quad+\underbrace{\mathbb{E}_{\mbZnt,\mbZns}[\mrR(\mfmc^{\beta_s}(\nus)\mfms^{\beta_t}(\mfmc^{\beta_s}(\nus),\nut),\nut)]
      -\mathbb{E}_{\mbZnt,\mbZns}[\mrR(\mfmc^{\beta_s}(\nus)\tilde{m}_{\spec}^t,\nupop^t)]}_{I_2}
       \\ &\quad+\underbrace{ \mathbb{E}_{\mbZnt,\mbZns}[\mrR(\mfmc^{\beta_s}(\nus)\tilde{m}_{\spec}^t,\nupop^t)]
        -\mathbb{E}_{\mbZnt,\mbZns}[\mrR(\mfmc^{\beta_s}(\nus)\tilde{m}_{\spec}^t,\nupop^s)]}_{I_3}
       \\&\quad+ \underbrace{\mathbb{E}_{\mbZnt,\mbZns}[\mrR(\mfmc^{\beta_s}(\nus)\tilde{m}_{\spec}^t,\nupop^s)] 
       - \mathbb{E}_{\mbZnt,\mbZns}[\mrR(\mfmc^{\beta_s}(\nus)\tilde{m}_{\spec}^s,\nupop^s)]}_{I_4} 
       \\&\quad +\underbrace{\mathbb{E}_{\mbZnt,\mbZns}[\mrR(\mfmc^{\beta_s}(\nus)\tilde{m}_{\spec}^s,\nupop^s)]
       - \mathbb{E}_{\mbZnt,\mbZns}[\mrR(\mfmc^{\beta_s}(\nus)\mfms^{\beta_s}(\mfmc^{\beta_s}(\nus),\nus),\nus),\nupop^s)]}_{I_5}
        \\&\quad+ \underbrace{\mathbb{E}_{\mbZnt,\mbZns}[\mrR(\mfmc^{\beta_s}(\nus)\mfms^{\beta_s}(\mfmc^{\beta_s}(\nus),\nus),\nus),\nupop^s)] 
      - \mathbb{E}_{\mbZnt,\mbZns}[\mrR(\mfmc^{\beta_s}(\nus)\mfms^{\beta_s}(\mfmc^{\beta_s}(\nus),\nus),\nus),\nus)]}_{I_6}
       \\&\quad + \underbrace{\mathbb{E}_{\mbZnt,\mbZns}[\mrR(\mfmc^{\beta_s}(\nus)\mfms^{\beta_s}(\mfmc^{\beta_s}(\nus),\nut),\nus),\nus)]- \mrR(\mrmbarc^{s}\otimes\mrmbarsp^{s},\nupop^s)}_{I_7}\\
       &\quad + \underbrace{\mrR(\mrmbarc^{s}\otimes\mrmbarsp^{s},\nupop^s)-\mrR(\mrmbarc^{s}\otimes\mrmbarsp^{s},\nupop^t)}_{I_8}\\
&\quad+\underbrace{\mrR(\mrmbarc^{s}\otimes\mrmbarsp^{s},\nupop^t)-\mrR(\mrmbarc^{t}\otimes\mrmbarsp^{t},\nupop^t)}_{I_9},
       \end{split}
   \end{equation}
   where $\mrmbarc^{\mathcal{T}}=\int_{\Thetas}\bar{m}^{\mathcal{T}}\mrd \thetas$, $\mrmbarsp^{\mathcal{T}}=\int_{\Thetac}\bar{m}^{\mathcal{T}}\mrd \thetac$, $\bar{m}^{\mathcal{T}}=\mathop{\arg \inf}_{m\in\mathcal{P}(\Theta)}\mrR(m,\nupop^{\mathcal{T}})$ for $\mathcal{T}\in\{s,t\}$ and $\tilde{m}_{\spec}^t=\mathop{\arg \min}_{m_{\spec}\in\mathcal{P}(\Thetas)}\mrR(\mfmc(\nus)m_{\spec},\nupop^t)$.
   Using the Theorem~\ref{thm: WTGE fine tune} result, we have the following bound on term $I_1$ in \eqref{eq: ex fine},
   \begin{equation}\label{eq: ex fine1}
   \begin{split}
   &\mathbb{E}_{\mbZnt,\mbZns}[\mrR(\mfmc^{\beta_s}(\nus)\mfms^{\beta_t}(\mfmc^{\beta_s}(\nus),\nut),\nupop^t)] 
        \\&\qquad- \mathbb{E}_{\mbZnt,\mbZns}[\mrR(\mfmc^{\beta_s}(\nus)\mfms^{\beta_t}(\mfmc^{\beta_s}(\nus),\nut),\nut)]\\
        &\leq \frac{2\beta_t^2}{\sigma^2}  \frac{16}{n_t}(1+\frac{2}{n_t})^2L_e^2(1+L_m)^2\comp(\thetac,\thetas^t,\thetas^s) \mathbb{E}_{Z_1^s}\Big[(1+\|Z_1^s\|^2)^2\Big] \mathbb{E}_{Z_1^t}\Big[(1+\|Z_1^t\|^2)^4\Big].
 \end{split}\end{equation}
Using the definition of $\mrmbarsp$ and the KL-regularized, we have the following bound on term $I_2$ in \eqref{eq: ex fine},
 \begin{equation}\label{eq: ex fine2}
 \begin{split}
& \mathbb{E}_{\mbZnt,\mbZns}[\mrR(\mfmc^{\beta_s}(\nus)\mfms^{\beta_t}(\mfmc^{\beta_s}(\nus),\nut),\nut)]
       -\mathbb{E}_{\mbZnt,\mbZns}[\mrR(\mfmc^{\beta_s}(\nus)\tilde{m}_{\spec}^t,\nupop^t)]\\
        &\leq \frac{\sigma^2}{2\beta_t^2}\KLr(\mrmbarsp^t\|\tilde\gamma_{\spec}^{\sigma}).
 \end{split}
 \end{equation}
Using Lemma~\ref{lem: data upper bound with dist}, we have the following bound on term $I_3$ in \eqref{eq: ex fine},
 \begin{equation}\label{eq: ex fine5}
     \begin{split}
          &\mathbb{E}_{\mbZnt,\mbZns}[\mrR(\mfmc^{\beta_s}(\nus)\tilde{m}_{\spec}^t,\nupop^t)]-\mathbb{E}_{\mbZnt,\mbZns}[\mrR(\mfmc^{\beta_s}(\nus)\tilde{m}_{\spec}^t,\nupop^s)]\\
          &\leq L_m\big(1+\mbE_{\thetac\sim \mfmc^{\beta_s}(\nus)}[\|\thetac\|^4]+\mbE_{\thetas\sim \tilde{m}_{\spec}^t}[\|\thetas^t\|^4]\big)\Disnd(\nupop^s,\nupop^t)\\
          &\leq 2L_m\mbE_{Z_1^s}\big[(1+\|Z_1^s\|^2)\big] \\&\qquad\times\Big(1+\int_{\Thetac} \|\thetac\|^4\hat\gamma_{8,c}^\sigma(\mrd\thetac)+\int_{\Thetas}\|\thetas^s\|^4\gammasp(\mrd \thetas^s)+\int_{\Thetas}\|\thetas\|^4 \tilde{m}_{\spec}^t(\mrd \thetas^t)\Big)\Disnd(\nupop^s,\nupop^t).
     \end{split}
 \end{equation}
 where the last inequality follows from applying  Lemma~\ref{lem:optimizerin-Bnus-shared},
\begin{equation}
    \begin{split}
        &\mathbb{E}_{\thetac\sim \mfmc(\nus)}\big[\|\thetac\|^4\big]\\
        &\leq  \mrR(\hat\gamma_{p}^\sigma,\nus) + \int_{\Thetac} \|\thetac\|^4\gammac(\thetac)\mrd\thetac\\
        &\leq g(\hat\gamma_{8,c}^\sigma \gammasp)\frac{1}{n_s}\sum_{j=1}^{n_s}(1+\|Z_j^s\|^2)+\int_{\Thetac} \|\thetac\|^4\gammac(\thetac)\mrd\thetac,
    \end{split}
\end{equation}

 Using a similar approach to Lemma~\ref{lem: parameter upper bound with dist}, we have the following bound on term $I_4$ in \eqref{eq: ex fine}, 
   \begin{equation}
       \begin{split}
          & \mathbb{E}_{\mbZnt,\mbZns}[\mrR(\mfmc^{\beta_s}(\nus)\tilde{m}_{\spec}^t,\nupop^s)] - \mathbb{E}_{\mbZnt,\mbZns}[\mrR(\mfmc^{\beta_s}(\nus)\tilde{m}_{\spec}^s,\nupop^s)] \\
           &\leq L_e\mbE_{Z_1^s}[(1+\|Z_1^s\|^2)^2]\Big(1+\mbE_{\mbZns}\big[\mbE_{\thetac\sim\mfmc^{\beta_s}(\nus)}[\| \thetac\|^4] \big]+ \mbE_{\thetas\sim\frac{\tilde{m}_{\spec}^s+\tilde{m}_{\spec}^t}{2}}[\| \thetas\|^4] \Big)\\
           &\qquad\times\Disfth(\tilde{m}_{\spec}^s,\tilde{m}_{\spec}^t).
       \end{split}
   \end{equation}
Using the definition of $\tilde{m}_{\spec}^s$, we have the following bound on term $I_5$ in \eqref{eq: ex fine},
 \begin{equation}\label{eq: last fine}
     \begin{split}
         &\mathbb{E}_{\mbZnt,\mbZns}[\mrR(\mfmc^{\beta_s}(\nus)\tilde{m}_{\spec}^s,\nupop^s)]
       - \mathbb{E}_{\mbZnt,\mbZns}[\mrR(\mfmc^{\beta_s}(\nus)\mfms^{\beta_s}(\mfmc^{\beta_s}(\nus),\nus),\nus),\nupop^s)]\leq 0.
     \end{split}
 \end{equation}

 Using the same approach in Theorem~\ref{thm: WTGE fine tune} in the first step under Assumption~\ref{ass:KLreg_assn}, we have the following bound on term $I_6$ in \eqref{eq: ex fine},
 \begin{equation}\label{eq: ex fine3}
 \begin{split}
 &\mathbb{E}_{\mbZnt,\mbZns}[\mrR(\mfmc^{\beta_s}(\nus)\mfms^{\beta_t}(\mfmc^{\beta_s}(\nus),\nus),\nus),\nupop^s)] \\&\qquad- \mathbb{E}_{\mbZnt,\mbZns}[\mrR(\mfmc^{\beta_s}(\nus)\mfms^{\beta_s}(\mfmc^{\beta_s}(\nus),\nus),\nus),\nus)] \\
 &\leq   \frac{9\sqrt{2}}{n_s} L_e^2 (1+ L_m)^2(1+\frac{1}{n_s})^2 \frac{2\beta_s^2}{\sigma^2}\comp(\theta) \mathbb{E}_{Z_1^t}\Big[(1+\|Z_1^s\|^2)^4\Big].
 \end{split}
 \end{equation}
 Furthermore, we have the following bound on term $I_7$ in \eqref{eq: ex fine},
 \begin{equation}\label{eq: ex fine4}
     \begin{split}
         &\mathbb{E}_{\mbZnt,\mbZns}[\mrR(\mfmc^{\beta_s}(\nus)\mfms^{\beta_t}(\mfmc^{\beta_s}(\nus),\nut),\nus),\nus)]- \mrR(\mrmbarc^{s}\mrmbarsp^{s},\nupop^s)\\&\leq \frac{\sigma^2}{2\beta_s^2}\KLr(\mrmbarc^s\mrmbarsp^s\|\gamma_c^\sigma\otimes\gamma_{\spec}^{\sigma}).
     \end{split}
 \end{equation}
 Using Lemma~\ref{lem: data upper bound with dist}, we have the following bound on term $I_8$ in \eqref{eq: ex fine},
 \begin{equation}
     \begin{split}
    &\mrR(\mrmbarc^{s}\mrmbarsp^{s},\nupop^s)-\mrR(\mrmbarc^{s}\mrmbarsp^{s},\nupop^t)\\&\leq (1+\mbE_{\thetac\sim \mrmbarc^{s}}[\|\thetac\|^4]+\mbE_{\thetas\sim \mrmbarsp^{s}}[\|\thetas\|^4])\Disnd(\nupop^s,\nupop^t).
     \end{split}
 \end{equation}
Using Lemma~\ref{lem: parameter upper bound with dist}, we have the following bound on term $I_9$ in \eqref{eq: ex fine},
 \begin{equation}
     \begin{split}
         &\mrR(\mrmbarc^{s}\otimes\mrmbarsp^{s},\nupop^t)-\mrR(\mrmbarc^{t}\otimes\mrmbarsp^{t},\nupop^t)\\
         &\leq \mbE_{Z_1^t}[(1+\|Z_1^t\|^2)^2]\Disfth(\mrmbarc^{s}\otimes\mrmbarsp^{s},\mrmbarc^{t}\otimes\mrmbarsp^{t}).
     \end{split}
 \end{equation}
   The final result holds via substituting \eqref{eq: ex fine1}-\eqref{eq: last fine} in \eqref{eq: ex fine},
   \begin{equation}
       \begin{split}
&\mathcal{E}(\mfmc^{\beta_s}(\nus)\mfms^{\beta_t}(\mfmc^{\beta_s}(\nus),\nut),\nupop^t)\\
 &\leq \frac{2\beta_t^2}{\sigma^2}  \frac{16}{n_t}(1+\frac{2}{n_t})^2L_e^2(1+L_m)^2\comp(\thetac,\thetas^t,\thetas^s) \mathbb{E}_{Z_1^s}\Big[(1+\|Z_1^s\|^2)^2\Big] \mathbb{E}_{Z_1^t}\Big[(1+\|Z_1^t\|^2)^4\Big]\\
     &\quad+\frac{\sigma^2}{2\beta_t^2}\KLr(\tilde{m}_{\spec}^t\| \tilde\gamma_{\spec}^{\sigma})\\
     &\quad+2\mbE_{Z_1^s}\big[(1+\|Z_1^s\|^2)\big] \Big(1+\int_{\Thetac} \|\thetac\|^4\hat\gamma_{8,c}^\sigma(\mrd\thetac)+\int_{\Thetas}\|\thetas\|^4\gammasp(\mrd \thetas^s)+\int_{\Thetas}\|\thetas\|^4 \tilde{m}_{\spec}^t(\mrd \thetas)\Big)\\&\qquad\times\Disnd(\nupop^s,\nupop^t) \\
     &\quad+\mbE_{Z_1^s}[(1+\|Z_1^s\|^2)^2]\Disfth(\tilde{m}_{\spec}^s,\tilde{m}_{\spec}^t)\\
     &\quad+ \frac{2\beta_s^2}{\sigma^2} \frac{9\sqrt{2}}{n_s} L_e^2 (1+ L_m)^2(1+\frac{2}{n_s})^2 \comp(\theta) \mathbb{E}_{Z_1^t}\Big[(1+\|Z_1^s\|^2)^4\Big]\\
     &\quad+\frac{\sigma^2}{2\beta_s^2}\KLr(\mrmbarc^s\otimes\mrmbarsp^s\|\gamma_c^\sigma\otimes\gamma_{\spec}^{\sigma})\\
     &\quad +L_e (1+\mbE_{\thetac\sim \mrmbarc^{s}}[\|\thetac\|^4]+\mbE_{\thetas\sim \mrmbarsp^{s}}[\|\thetas\|^4])\Disnd(\nupop^s,\nupop^t)\\
     &\quad +L_e \mbE_{Z_1^t}[(1+\|Z_1^t\|^2)^2]\Disfth(\mrmbarc^{s}\otimes\mrmbarsp^{s},\mrmbarc^{t}\otimes\mrmbarsp^{t})\\
 &\leq \frac{2\beta_t^2}{\sigma^2}  \frac{16}{n_t}(1+\frac{2}{n_t})^2L_e^2(1+L_m)^2\comp(\thetac,\thetas^t,\thetas^s) \mathbb{E}_{Z_1^s}\Big[(1+\|Z_1^s\|^2)^2\Big] \mathbb{E}_{Z_1^t}\Big[(1+\|Z_1^t\|^2)^4\Big]\\
    &\quad+\frac{\sigma^2}{2\beta_t^2}\KLr(\tilde{m}_{\spec}^t\| \tilde\gamma_{\spec}^{\sigma})\\
    &\quad+ \frac{2\beta_s^2}{\sigma^2} \frac{9\sqrt{2}}{n_s} L_e^2 (1+ L_m)^2(1+\frac{2}{n_s})^2 \comp(\theta) \mathbb{E}_{Z_1^t}\Big[(1+\|Z_1^s\|^2)^4\Big]\\
    &\quad+\frac{\sigma^2}{2\beta_s^2}\KLr(\mrmbarc^s\otimes\mrmbarsp^s\|\gamma_c^\sigma\otimes\gamma_{\spec}^{\sigma})\\
    &\quad+4\mbE_{Z_1^s}\big[(1+\|Z_1^s\|^2)\big] \Big(1+\mbE_{\thetac\sim \mrmbarc^{s}}[\|\thetac\|^4]+\mbE_{\thetas\sim \mrmbarsp^{s}}[\|\thetas\|^4]+\int_{\Thetac} \|\thetac\|^4\hat\gamma_{8,c}^\sigma(\mrd\thetac)
    \\&\qquad+\int_{\Thetas}\|\thetas\|^4\gammasp(\mrd \thetas^s)+\int_{\Thetas}\|\thetas\|^4 \tilde{m}_{\spec}^t(\mrd \thetas)\Big)
     \\&\qquad\times\Disnd(\nupop^s,\nupop^t) \\
    &\quad+L_e\mbE_{Z_1^s}[(1+\|Z_1^s\|^2)^2]\Disfth(\tilde{m}_{\spec}^s,\tilde{m}_{\spec}^t)\\
     &\quad + L_e \mbE_{Z_1^t}[(1+\|Z_1^t\|^2)^2]\Disfth(\mrmbarc^{s}\otimes\mrmbarsp^{s},\mrmbarc^{t}\otimes\mrmbarsp^{t}).
     \end{split} 
   \end{equation}
\end{proof}

\section{Proofs and details from Section \ref{Sec: Application}}\label{app: Sec: Application}
The Fine-tuning scenario in one-hidden layer neural network is shown in Fig.~\ref{Fig: NN-combined}. Furthermore, the block diagram of $\alpha$-ERM is shown in Fig.~\ref{Fig: NN-alpha} where we have $\alpha \mrR(m,\nut) + (1-\alpha) \mrR(m,\nus)$ as risk function.
\begin{figure}[htb]
    \centering
    \begin{subfigure}[b]{0.7\textwidth}
        \centering
        \resizebox{\textwidth}{!}{%
            \begin{tikzpicture}[
    input/.style = {circle, draw, fill=green!20, minimum size=0.8cm, font=\footnotesize},
    hidden/.style = {circle, draw, fill=blue!20, minimum size=0.8cm, font=\footnotesize},
    output1/.style = {circle, draw, fill=red!20, minimum size=0.8cm, font=\footnotesize},
    output2/.style = {circle, draw, fill=yellow!20, minimum size=0.8cm, font=\footnotesize},
    every edge/.style = {draw, ->, blue},
    every node/.style = {font=\scriptsize},
]

\node[input] (x1) at (0, -1) {$x_1^s$};
\node[input] (x2) at (0, -2.5) {$x_2^s$};
\node (dots1) at (0, -3.5) {$\vdots$};
\node[input] (xq) at (0, -5) {$x_q^s$};

\node[hidden] (h1) at (4, 0) {$h_1$};
\node[hidden] (h2) at (4, -2) {$h_2$};
\node[hidden] (h3) at (4, -4) {$h_3$};
\node (dots2) at (4, -5) {$\vdots$};
\node[hidden] (hr) at (4, -6) {$h_r$};

\node[output1] (O1) at (8, -2) {$O_s$};
\node[output2] (O2) at (8, -4) {$O_t$};

\draw[->, blue] (x1) -- node[pos=0.3, above, sloped] {$w_{1,1}$} (h1);
\draw[->, blue] (x2) -- node[pos=0.3, above, sloped] {$w_{2,1}$} (h1);
\draw[->, blue] (xq) -- node[pos=0.3, above, sloped] {$w_{q,1}$} (h1);

\draw[->, blue] (h1) -- node[pos=0.3, above, sloped] {$a_{1,s}$} (O1);
\draw[->, blue] (h2) -- node[pos=0.3, above, sloped] {$a_{2,s}$} (O1);
\draw[->, blue] (h3) -- node[pos=0.3, below, sloped] {$a_{3,s}$} (O1);
\draw[->, blue] (hr) -- node[pos=0.3, below, sloped] {$a_{r,s}$} (O1);

\draw[->, black!20] (h1) --  (O2);
\draw[->, black!20] (h2) --  (O2);
\draw[->, black!20] (h3) --  (O2);
\draw[->, black!20] (hr) --  (O2);

\draw[->, blue!20] (x1) -- (h2);
\draw[->, blue!20] (x1) -- (h3);
\draw[->, blue!20] (x1) -- (hr);
\draw[->, blue!20] (x2) -- (h2);
\draw[->, blue!20] (x2) -- (h3);
\draw[->, blue!20] (x2) -- (hr);
\draw[->, blue!20] (xq) -- (h2);
\draw[->, blue!20] (xq) -- (h3);
\draw[->, blue!20] (xq) -- (hr);

\node[right=0.5cm of h1, text width=4cm, align=left, font=\scriptsize] {
    $h_1 = \varphi \left( \sum_{i=1}^q w_{i,1}x_i^s \right)$
};

\node[right=0.5cm of O1, text width=6cm, align=left, font=\scriptsize] {
    $O_s = \frac{1}{r} \sum_{j=1}^r a_{j,s} \varphi \left( \sum_{i=1}^q w_{i,j}x_i^s \right)$
};




\end{tikzpicture}
        }
        \caption{First step of Fine-tuning}
        \label{Fig: NN-first}
    \end{subfigure}
    \hfill
    \begin{subfigure}[b]{0.7\textwidth}
        \centering
        \resizebox{\textwidth}{!}{%
            \begin{tikzpicture}[
    input/.style = {circle, draw, fill=green!20, minimum size=0.8cm, font=\footnotesize},
    hidden/.style = {circle, draw, fill=blue!20, minimum size=0.8cm, font=\footnotesize},
    output1/.style = {circle, draw, fill=red!20, minimum size=0.8cm, font=\footnotesize},
    output2/.style = {circle, draw, fill=yellow!20, minimum size=0.8cm, font=\footnotesize},
    every edge/.style = {draw, ->, blue},
    every node/.style = {font=\scriptsize},
]

\node[input] (x1) at (0, -1) {$x_1^t$};
\node[input] (x2) at (0, -2.5) {$x_2^t$};
\node (dots1) at (0, -3.5) {$\vdots$};
\node[input] (xq) at (0, -5) {$x_q^t$};

\node[hidden] (h1) at (4, 0) {$h_1$};
\node[hidden] (h2) at (4, -2) {$h_2$};
\node[hidden] (h3) at (4, -4) {$h_3$};
\node (dots2) at (4, -5) {$\vdots$};
\node[hidden] (hr) at (4, -6) {$h_r$};

\node[output1] (O1) at (8, -2) {$O_s$};
\node[output2] (O2) at (8, -4) {$O_t$};

\draw[->, blue] (x1) -- node[pos=0.3, above, sloped] {$w_{1,1}$} (h1);
\draw[->, blue] (x2) -- node[pos=0.3, above, sloped] {$w_{2,1}$} (h1);
\draw[->, blue] (xq) -- node[pos=0.3, above, sloped] {$w_{q,1}$} (h1);

\draw[->, blue!20] (h1) -- (O1);
\draw[->, blue!20] (h2) -- (O1);
\draw[->, blue!20] (h3) -- (O1);
\draw[->, blue!20] (hr) -- (O1);

\draw[->, black] (h1) -- node[pos=0.3, above, sloped] {$a_{1,t}$} (O2);
\draw[->, black] (h2) -- node[pos=0.3, below, sloped] {$a_{2,t}$} (O2);
\draw[->, black] (h3) -- node[pos=0.3, below, sloped] {$a_{3,t}$} (O2);
\draw[->, black] (hr) -- node[pos=0.3, below, sloped] {$a_{r,t}$} (O2);

\draw[->, blue!20] (x1) -- (h2);
\draw[->, blue!20] (x1) -- (h3);
\draw[->, blue!20] (x1) -- (hr);
\draw[->, blue!20] (x2) -- (h2);
\draw[->, blue!20] (x2) -- (h3);
\draw[->, blue!20] (x2) -- (hr);
\draw[->, blue!20] (xq) -- (h2);
\draw[->, blue!20] (xq) -- (h3);
\draw[->, blue!20] (xq) -- (hr);

\node[right=0.5cm of h1, text width=4cm, align=left, font=\scriptsize] {
    $h_1 = \varphi \left( \sum_{i=1}^q w_{i,1}x_i^t \right)$
};



\node[right=0.5cm of O2, text width=6cm, align=left, font=\scriptsize] {
    $O_t = \frac{1}{r} \sum_{j=1}^r a_{j,t} \varphi \left( \sum_{i=1}^q w_{i,j}x_i^t \right)$
};

\end{tikzpicture}
        }
        \caption{Second step of Fine-tuning}
        \label{Fig: NN-second}
    \end{subfigure}
    \caption{ Fine-tuning Scenario}
    \label{Fig: NN-combined}
\end{figure}
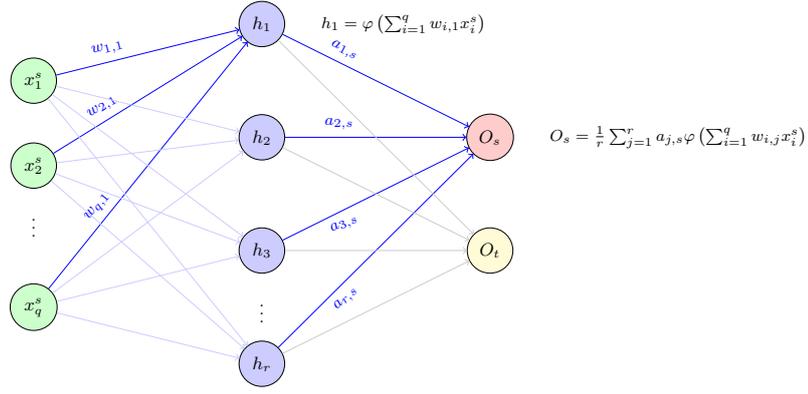
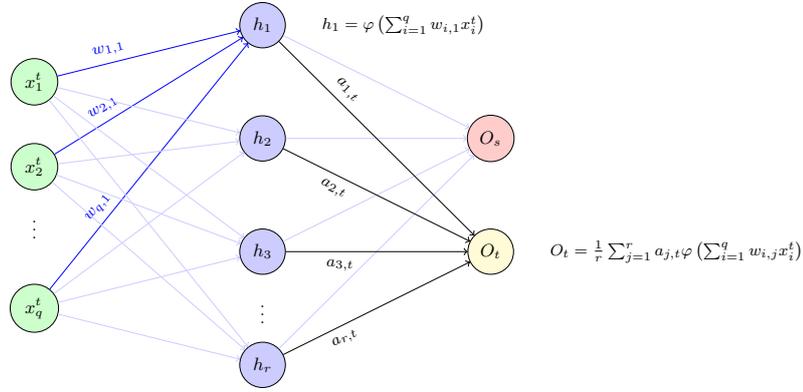
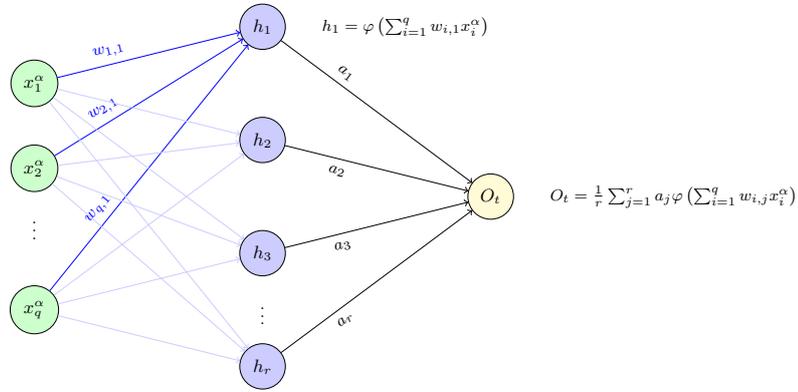
\begin{figure}[htb]
        \centering
        \resizebox{0.7\textwidth}{!}{%
            \begin{tikzpicture}[
    input/.style = {circle, draw, fill=green!20, minimum size=0.8cm, font=\footnotesize},
    hidden/.style = {circle, draw, fill=blue!20, minimum size=0.8cm, font=\footnotesize},
    output1/.style = {circle, draw, fill=red!20, minimum size=0.8cm, font=\footnotesize},
    output2/.style = {circle, draw, fill=yellow!20, minimum size=0.8cm, font=\footnotesize},
    every edge/.style = {draw, ->, blue},
    every node/.style = {font=\scriptsize},
    arrow/.style={->, line width=5pt, >=latex}
]

\node[input] (x1) at (0, -1) {$x_1^{\alpha}$};
\node[input] (x2) at (0, -2.5) {$x_2^{\alpha}$};
\node (dots1) at (0, -3.5) {$\vdots$};
\node[input] (xq) at (0, -5) {$x_q^{\alpha}$};

\node[hidden] (h1) at (4, 0) {$h_1$};
\node[hidden] (h2) at (4, -2) {$h_2$};
\node[hidden] (h3) at (4, -4) {$h_3$};
\node (dots2) at (4, -5) {$\vdots$};
\node[hidden] (hr) at (4, -6) {$h_r$};

\node[output2] (O2) at (8, -3) {$O_t$};

\draw[->, thin, blue] (x1) -- node[pos=0.3, above, sloped] {$w_{1,1}$} (h1);
\draw[->, blue] (x2) -- node[pos=0.3, above, sloped] {$w_{2,1}$} (h1);
\draw[->, blue] (xq) -- node[pos=0.3, above, sloped] {$w_{q,1}$} (h1);


\draw[->, black] (h1) -- node[pos=0.3, above, sloped] {$a_{1}$} (O2);
\draw[->, black] (h2) -- node[pos=0.3, below, sloped] {$a_{2}$} (O2);
\draw[->, black] (h3) -- node[pos=0.3, below, sloped] {$a_{3}$} (O2);
\draw[->, black] (hr) -- node[pos=0.3, below, sloped] {$a_{r}$} (O2);

\draw[->, blue!20] (x1) -- (h2);
\draw[->, blue!20] (x1) -- (h3);
\draw[->, blue!20] (x1) -- (hr);
\draw[->, blue!20] (x2) -- (h2);
\draw[->, blue!20] (x2) -- (h3);
\draw[->, blue!20] (x2) -- (hr);
\draw[->, blue!20] (xq) -- (h2);
\draw[->, blue!20] (xq) -- (h3);
\draw[->, blue!20] (xq) -- (hr);

\node[right=0.5cm of h1, text width=4cm, align=left, font=\scriptsize] {
    $h_1 = \varphi \left( \sum_{i=1}^q w_{i,1}x_i^{\alpha} \right)$
};




\node[right=0.5cm of O2, text width=6cm, align=left, font=\scriptsize] {
    $O_t = \frac{1}{r} \sum_{j=1}^r a_{j} \varphi \left( \sum_{i=1}^q w_{i,j}x_i^{\alpha} \right)$
};

\end{tikzpicture}
        }
        \caption{$\alpha$-ERM Scenario}
        \label{Fig: NN-alpha}
    \end{figure}

\clearpage
\begin{tcolorbox}
\begin{lemma}\label{lem: equ with main ass NN FT}
Under Assumption~\ref{ass:NN_KL_alpha}, there exists a constant $c>0$ independent of $n$ such that, for all $z\in\mcZ$, $\thetac\in \Thetac$, $\thetas\in \Thetas$, $m_{\spec}\in \mathcal{P}_4(\thetas)$, $m_{c}\in \mathcal{P}_4(\thetac)$ and $\nu\in \mathcal{P}_2(\mcZ)$,
\begin{align*}
&\big|\ell_o(\mathbb{E}_{a\sim m_c, w\sim m_{\spec}}[a\varphi(w.x)],y) \big| \leq 12 L_\ell L_\varphi^2 \big(1+\mathbb{E}_{a\sim m_{\spec}}[a^4]+\mbE_{w\sim m_c}[\|w\|^4]\big) \big(1+\|y\|^2 + \|x\|^2),\\
&\Big|\frac{\delta \ell_o}{\delta m_{\spec}}(\mathbb{E}_{a\sim m_c, w\sim m_{\spec}}[a\varphi(w.x)],y,a) \Big| \\&\quad\leq  64 L_{\ell,1}L_{\varphi}(1+L_{\varphi})\big(1+\mathbb{E}_{w\sim \mfmc}[\|w\|^4]+\mbE_{a\sim \mfms}[a^4]+|a|^4\big)\Big(1+\|x\|^2+\|y\|^2\Big),\\
&\mathbb{E}_{\theta,\theta'\sim m}\mathbb{E}_{z\sim \nu}\Big[\Big(\frac{\delta^2 \mathrm{\ell}}{\delta m^2}(m, z, \theta, \theta') \Big)^2\Big]^{1/2}\leq c\Big[\mathbb{E}_{\theta\sim m}[1+\|\theta\|^4]\Big]\mathbb{E}_{Z\sim \nu}[1+\|Z\|^2].
\end{align*}
In particular, if $p\ge 8$ in Assumption~\ref{ass:NN_KL_alpha}, then Assumption~\ref{ass:KLreg_assnf} is satisfied for some choice of $L_e<\infty$.
\end{lemma}
\end{tcolorbox}

\begin{proof}[Proof of Lemma~\ref{lem: equ with main ass NN FT}]
 We recall that our loss is of the form $\ell(m, z) = \ell_o(\mathbb{E}_{a\sim m_c, w\sim m_{\spec}}[a\varphi(w.x)],y)$, for $\ell_o$ a convex function. It follows that $\ell$ is convex and nonnegative with respect to $m$.
 
 By the chain rule, we can differentiate with respect to $m_{\spec}$, to obtain, with $\Phi(m_c \otimes m_{\spec},x) := \mathbb{E}_{a\sim m_{\spec},w\sim m_c}[a\varphi(w.x)]$, $\thetas = a$ and $\thetac= w$,
    \begin{equation}\label{Eq: derivative NN loss_FT}
        \begin{split}
        \frac{\delta \Phi}{\delta m_{\spec}}(m_c \otimes m_{\spec},x,a) &= (a-\mbE_{a\sim m_{\spec}}[a])\mbE_{w\sim m_c}[\varphi(w.x)],\\
            \frac{\delta \ell}{\delta m_{\spec}}(m_c \otimes m_{\spec},z,a)&=\partial_{\hat{y}}\ell_o(\Phi(m,x),y)(a-\mbE_{a\sim m_{\spec}}[a])\mbE_{w\sim m_c}[\varphi(w.x)].
         \end{split}
    \end{equation}
    Note that, 
   \[
    \frac{\delta^2 \ell}{\delta m_{\spec}^2}(m_c \otimes m_{\spec},z,a, a') = \frac{\delta}{\delta m_{\spec}}\left(  \frac{\delta \ell}{\delta m_{\spec}}(\cdot,\cdot, a)\right)(m_c \otimes m_{\spec},z, a')\,,
    \] 
  and due to the normalization convention, see \citep[Remark 2.5]{cardaliaguet2019master},  
    \begin{equation}
        \begin{split}
             \frac{\delta^2 \ell}{\delta m_{\spec}^2}(m_c \otimes m_{\spec},z,a, a')&=\partial_{\hat{y}\hat{y}}\ell_o(\Phi(m_c \otimes m_{\spec},x),y)
             (a'-\mbE_{a\sim m_{\spec}}[a])
             (a-\mbE_{a\sim m_{\spec}}[a])\mbE_{w\sim m_c}[\varphi(w.x)]^2.
        \end{split}
    \end{equation}
It follows that $\ell$ is $\mathcal{C}^2$ with respect to $m$ (the required bounds on the derivatives will be obtained below).

Without loss of generality, for notational convenience, we assume that $L_\ell, L_{\ell,1}, L_{\ell,2}$ and $L_\phi$ in Assumption \ref{ass:NN_KL_alpha} are all at least $1$. Throughout, $c$ is a constant which can vary from line to line, but does not depend on $n$.   
 
 We first consider obtaining bounds on $\ell$. We know that 
\[\begin{split}
    \ell(\mfmc(\nus)\mfms(\mfmc(\nus),\nut),z)&\leq L_\ell(1+ \|y\|^2 + \|\hat y\|^2)= L_\ell(1+ \|y\|^2 + \|\mbE_{a,w}[a\varphi(w.x)]\|^2))\,.
\end{split}
\]
As
\[\begin{split}
    \|\mbE_{a,w}[a\varphi(w.x)]\|^2&\leq 2L_\varphi^2 (1+\|x\|)^2\big(|a|(1+\|w\|)\big)^2\\
    &\leq 8 L_\varphi^2 (1+\|x\|)^2\big(1+\mathbb{E}_{a\sim m_{\spec}}[a^2]+\mbE_{w\sim m_c}[\|w\|^2]\big)^2
\end{split}\]
we know 
\[\begin{split}
    \ell(\mfmc(\nus)\mfms(\mfmc(\nus),\nut),z)
    &\leq L_\ell\big(1+ \|y\|^2 + 2 L_\varphi^2 (1+\|x\|)^2\big(1+\mathbb{E}_{a\sim m_{\spec}}[a^2]+\mbE_{w\sim m_c}[\|w\|^2]\big)^2\big)\\
    &\leq 8 L_\ell L_\varphi^2 \big(1+\mathbb{E}_{a\sim m_{\spec}}[a^2]+\mbE_{w\sim m_c}[\|w\|^2]\big)^2 \big(1+\|y\|^2 + \|x\|^2).
\end{split}
\]
As we have \begin{equation}
    \begin{split}
         8L_\ell L_\phi^2 \big(1+\mathbb{E}_{a\sim m_{\spec}}[a^2]+\mbE_{w\sim m_c}[\|w\|^2]\big)^2
        \leq 24L_\ell L_\phi^2 \big(1+\mathbb{E}_{a\sim m_{\spec}}[a^4]+\mbE_{w\sim m_c}[\|w\|^4]\big),
    \end{split}
\end{equation}
then, we take $g(m_c\otimes m_{\spec})=24L_\ell L_\phi^2 \big(1+\mathbb{E}_{a\sim m_{\spec}}[a^4]+\mbE_{w\sim m_c}[\|w\|^4]\big).$

We now consider the derivative of $\ell$. Writing $m_{2,c}:=\mathbb{E}_{w\sim \mfmc}[\|w\|^2]$ and $m_{2,\spec}:=\mbE_{a\sim \mfms}[a^2]$, we have 
    \begin{align*}
        \Big|\frac{\delta \ell}{\delta m_{\spec}}\big(m,z,\theta\big)\Big|& = \Big|\partial_{\hat{y}}\ell_o(\mbE_{a,w}[a\varphi(w.x)],y)\Big|\,|(a-\mbE_{a\sim m_{\spec}}[a])||\mbE_{w\sim m_c}[\varphi(w.x)]|\\
        &\leq L_{\ell,1}(1+|\mbE_{a,w}[a\varphi(w.x)]|+\|y\|)|(a-\mbE_{a\sim m_{\spec}}[a])||\mbE_{w\sim m_c}[\varphi(w.x)]|\\
        &\leq L_{\ell,1}L_{\varphi}\Big(1+L_{\varphi}(1+\|x\|)(1+m_{2,c}+m_{2,\spec})+\|y\|\Big)(1+\|x\|)\Big(1+m_{2,c}\Big)|(a-\mbE_{a\sim m_{\spec}}[a])|\\
        &\leq L_{\ell,1}L_{\varphi}(1+L_{\varphi})(1+m_{2,c}+m_{2,\spec})\Big(1+\|x\|+\|y\|\Big)(1+\|x\|)^2|(a-\mbE_{a\sim m_{\spec}}[a])|\\
         &\leq  2L_{\ell,1}L_{\varphi}(1+L_{\varphi})(1+m_{2,c}+m_{2,\spec})\Big(1+\|x\|+\|y\|\Big)(1+\|x\|^2)|(a-\mbE_{a\sim m_{\spec}}[a])|\\
         &\leq 8L_{\ell,1}L_{\varphi}(1+L_{\varphi})(1+m_{2,c}+m_{2,\spec})\big(1+|a|+|\mbE_{a\sim m_{\spec}}[a]|\big)\Big(1+\|x\|^2+\|y\|^2\Big)\\
         &\leq 32L_{\ell,1}L_{\varphi}(1+L_{\varphi})(1+m_{2,c}+m_{2,\spec})\big(1+|a|^2+m_{2,\spec}\big)\Big(1+\|x\|^2+\|y\|^2\Big)\\
        & \leq 64 L_{\ell,1}L_{\varphi}(1+L_{\varphi})(1+m_{2,c}+m_{2,\spec}+|a|^2\big)\Big(1+\|x\|^2+\|y\|^2\Big)
    \end{align*}
  We then take \[g_{\spec}(m_c m_{\spec},\thetas) = 64L_{\ell,1}L_{\varphi}(1+L_{\varphi})\big(1+\mathbb{E}_{w\sim \mfmc}[\|w\|^4]+\mbE_{a\sim \mfms}[a^4]+|a|^4\big),\] and we have, with $c$ a constant varying from line to line,
  \begin{equation}
  \begin{split}
      \mathbb{E}_{\theta\sim m_{\spec}^\prime}\big[g_{\spec}(m_c m_{\spec},\theta_{\spec})^2\big]^{1/2}&=c\mathbb{E}_{\thetas\sim m_{\spec}^\prime}\Big[\Big(1+\mathbb{E}_{w\sim \mfmc}[\|w\|^4]+\mbE_{a\sim \mfms}[a^4]+|a|^4\Big)^2\Big]^{1/2}\\
      &\leq c\Big(1+\mathbb{E}_{w\sim \mfmc}[\|w\|^8]+\mbE_{a\sim \mfms}[a^8]+\mbE_{a\sim \mfms^{\prime}}[a^8]\Big)^{1/2}\\
      &\leq c\Big(1+\mathbb{E}_{w\sim \mfmc}[\|w\|^8]+\mbE_{a\sim \mfms}[a^8]+\mbE_{a\sim \mfms^{\prime}}[a^8]\Big).
      \end{split}
  \end{equation}
Together with Jensen's inequality, this yields \eqref{eq: bounded g assf}.

  Finally, we consider the second derivative of $\ell$. We have,  
   \begin{equation}
   \begin{split}
        &\mathbb{E}_{a,a'\sim m_{\spec},m_{\spec}^{\prime}}\mathbb{E}_{z\sim \nu}\Big[\Big(\frac{\delta^2 \ell}{\delta m_{\spec}^2}(m_c \otimes m_{\spec},z,a, a') \Big)^2\Big]\\
&=\int_\mcZ\Big[\partial_{\hat{y}\hat{y}}\ell_o(\Phi(m_c \otimes m_{\spec},x),y)\Big]^2\mathbb{V}_{\thetas\sim m_{\spec}}\Big(a\Big)\mathbb{V}_{\thetas^\prime\sim m_{\spec}^{\prime}}\Big(a^{\prime}\Big) \mbE_{w\sim m_c}[\varphi(w.x)]^4\nu(\mrd z)\\
        &\leq \int_\mcZ\Big[\partial_{\hat{y}\hat{y}}\ell_o(\Phi(m_c \otimes m_{\spec},x),y)\Big]^2\Big[\mathbb{V}_{\thetas\sim m_{\spec}}\Big(a\Big)\Big]^2\mbE_{w\sim m_c}[(1+\|w\|)^4](1+\|x\|)^4\nu(\mrd z)\\
        &\leq c\, \mathbb{E}_{Z\sim \nu}[(1+\|Z\|)^4] \mbE_{w\sim m_c}[(1+\|w\|)^4]\mbE_{a\sim m_{\spec}}[a^4].\\
    \end{split}
    \end{equation}
giving the stated inequalities. As $\nu\in \mathcal{P}_2(\mcZ)$, and assuming $m_{\spec}\otimes m_{c}\in \mathcal{P}_8(\Thetac)\times \mathcal{P}_8(\Thetas) $. It follows that Assumption \ref{ass:KLreg_assnf} holds whenever $p\ge 8$.
\end{proof}

\begin{tcolorbox}
\begin{lemma}\label{lem: equ with main ass NN alpha}
Under Assumption~\ref{ass:NN_KL_alpha}, there exists a constant $c>0$ independent of $n$ such that, for all $z\in\mcZ$,  $\theta\in \Theta$, $m\in \mathcal{P}_8(\Theta)$, and $\nu\in \mathcal{P}_2(\mcZ)$,
\begin{align*}
\big|\ell(m,z) \big| &\leq c\Big[\mathbb{E}_{\theta\sim m}[1+\|\theta\|^4]\Big]\big(1+\|z\|^2\big),\\
\Big|\frac{\delta \ell}{\delta m}(m,z,\theta) \Big| &\leq c\Big[1+\|\theta\|^4 + \mathbb{E}_{\theta'\sim m}[\|\theta'\|^4]\Big]\big(1+\|z\|^2\big),\\
\mathbb{E}_{\theta,\theta'\sim m}\mathbb{E}_{z\sim \nu}\Big[\Big(\frac{\delta^2 \mathrm{\ell}}{\delta m^2}(m, z, \theta, \theta') \Big)^2\Big]^{1/2}&\leq c\Big[\mathbb{E}_{\theta\sim m}[1+\|\theta\|^4]\Big]\mathbb{E}_{Z\sim \nu}[1+\|Z\|^2].
\end{align*}
In particular, under Assumption~\ref{ass:NN_KL_alpha}, then Assumption~\ref{ass:KLreg_assn} is satisfied for some choice of $L_e<\infty$.
\end{lemma}
\end{tcolorbox}

\begin{proof}[Proof of Lemma~\ref{lem: equ with main ass NN alpha}]
 We recall that our loss is of the form $\ell(m, z) = \ell_o(\mathbb{E}_{\theta\sim m}[\phi(\theta,x)],y)$, for $\ell_o$ a convex function. It follows that $\ell$ is convex and nonnegative with respect to $m$.
 
 We write $\partial_{\hat y}\ell_o$ for the derivative of $\ell_o$ with respect to its first argument. By the chain rule, we can differentiate with respect to $m$, to obtain, with $\Phi(m,x) := \mathbb{E}_{\theta\sim m}[\phi(\theta,x)]$,
    \begin{equation}\label{Eq: derivative NN loss_alpha}
        \begin{split}
        \frac{\delta \Phi}{\delta m}(m,x,\theta) &= \phi(\theta,x)-\Phi(m,x),\\
            \frac{\delta \ell}{\delta m}(m,z,\theta)&=\partial_{\hat{y}}\ell_o(\Phi(m,x),y)\Big(\phi(\theta,x)-\Phi(m,x)\Big).
         \end{split}
    \end{equation}
    Note that, 
   \[
    \frac{\delta^2 \ell}{\delta m^2}(m,z,\theta, \theta') = \frac{\delta}{\delta m}\left(  \frac{\delta \ell}{\delta m}(\cdot,\cdot, \theta)\right)(m,z, \theta')\,,
    \] 
  and due to the normalization convention, see \citep[Remark 2.5]{cardaliaguet2019master},  
    \begin{equation}
        \begin{split}
             \frac{\delta^2 \ell}{\delta m^2}(m,z,\theta, \theta')&=\partial_{\hat{y}\hat{y}}\ell_o(\Phi(m,x),y)\Big(\phi(\theta,x)-\Phi(m,x)\Big)\Big(\phi(\theta',x)-\Phi(m,x)\Big).
        \end{split}
    \end{equation}
It follows that $\ell$ is $\mathcal{C}^2$ with respect to $m$ (the required bounds on the derivatives will be obtained below).

Without loss of generality, for notational convenience, we assume that $L_\ell, L_{\ell,1}, L_{\ell,2}$ and $L_\phi$ in Assumption \ref{ass:NN_KL_alpha} are all at least $1$. Throughout, $c$ is a constant which can vary from line to line, but does not depend on $n$.   
 
 We first consider obtaining bounds on $\ell$. We know that 
\[\begin{split}
    \ell(\mfm(\nu_n),z)&\leq L_\ell(1+ \|y\|^2 + \|\hat y\|^2)= L_\ell(1+ \|y\|^2 + \|\Phi(m,x)\|^2))\,.
\end{split}
\]
As
\[\begin{split}
    \|\mbE_{a,w}[a\varphi(w.x)]\|^2&\leq 2L_\varphi^2 (1+\|x\|)^2\big(|a|(1+\|w\|)\big)^2\\
    &\leq 8 L_\varphi^2 (1+\|x\|)^2\big(1+\mathbb{E}_{a\sim m_{\spec}}[a^2]+\mbE_{w\sim m_c}[\|w\|^2]\big)^2\\
    &\leq 8 L_\varphi^2 (1+\|x\|)^2\big(1+\mbE_{\theta\sim m}[\|\theta\|^2]\big)^2
\end{split}\]
we know 
\[\begin{split}
    \ell(\mfm(\nu_n),z)
    &\leq L_\ell\big(1+ \|y\|^2 + L_\phi^2 (1+\|x\|)^2\mathbb{E}_{\theta\sim m}[1+\|\theta\|^2]^2\big)\\
    &\leq 2L_\ell L_\phi^2 \mathbb{E}_{\theta\sim m}[1+\|\theta\|^2]^2 \big(1+\|y\|^2 + \|x\|^2).
\end{split}
\]
We therefore take $g(m):= 4L_\ell L_\phi^2 \mathbb{E}_{\theta\sim m}[1+\|\theta\|^4]$. As we know $p\ge 4$ and $\tilgammap$ is a density in $\mathcal{P}_4$, it follows that $g(\tilgammap)<\infty$.

We now consider the derivative of $\ell$. Writing $m_2:=\mathbb{E}_{\theta\sim m}[\|\theta\|^2]$, we have 
    \begin{align*}
        \Big|\frac{\delta \ell}{\delta m}\big(m,z,\theta\big)\Big|& = \Big|\partial_{\hat{y}}\ell_o(\Phi(m,x),y)\Big|\,|\phi(\theta,x)-\Phi(m,x)|\\
        &\leq L_{\ell,1}(1+|\Phi(m,x)|+\|y\|)\big|\phi(\theta,x)-\Phi(m,x)\big|\\
        &\leq L_{\ell,1}L_{\phi}\Big(1+L_{\phi}(1+\|x\|)(1+m_2)+\|y\|\Big)(1+\|x\|)\Big(2+\|\theta\|^2+m_2\Big)\\
        &\leq L_{\ell,1}L_{\phi}(1+L_{\phi})(1+m_2)\Big(2+\|\theta\|^2+m_2\Big)\Big(1+\|x\|+\|y\|\Big)(1+\|x\|)\\
         &\leq 2 L_{\ell,1}L_{\phi}(1+L_{\phi})\Big(2+\|\theta\|^2+m_2\Big)^2\Big(1+\|x\|^2+\|y\|^2\Big)\\
         &\leq 6L_{\ell,1}L_{\phi}(1+L_{\phi})\Big(4+\|\theta\|^4+\mathbb{E}_{\theta\sim m}[\|\theta\|^4]\Big)\Big(1+\|x\|^2+\|y\|^2\Big).
    \end{align*}
  We then take \[g_{\mathrm{e}}(m,\theta) = 6L_{\ell,1}L_{\phi}(1+L_{\phi})\Big(4+\|\theta\|^4+\mathbb{E}_{\theta\sim m}[\|\theta\|^4]\Big),\] and we have, with $c$ a constant varying from line to line,
  \begin{equation}
  \begin{split}
      \mathbb{E}_{\theta\sim m'}\big[g_{\mathrm{e}}(m,\theta)^2\big]^{1/2}&=c\mathbb{E}_{\theta\sim m'}\Big[\Big(4+\|\theta\|^4+\mathbb{E}_{\theta\sim m}[\|\theta\|^4]\Big)^2\Big]^{1/2}\\
      &\leq c\Big(8+2\mathbb{E}_{\theta\sim m'}[\|\theta\|^8]+2\mathbb{E}_{\theta\sim m}[\|\theta\|^8]\Big)^{1/2}\\
      &\leq c\Big(1+\mathbb{E}_{\theta\sim m'}[\|\theta\|^8]+\mathbb{E}_{\theta\sim m}[\|\theta\|^8]\Big).
      \end{split}
  \end{equation}
Together with Jensen's inequality, this yields \eqref{eq: bounded g ass} in the case  $p \geq 8$.

  Finally, we consider the second derivative of $\ell$. We have,  
   \begin{equation}
   \begin{split}
        &\mathbb{E}_{\theta,\theta'\sim m}\mathbb{E}_{z\sim \nu}\Big[\Big(\frac{\delta^2 \mathrm{\ell}}{\delta m^2}(m, z, \theta, \theta') \Big)^2\Big]\\
&=\int_\mcZ\Big[\partial_{\hat{y}\hat{y}}\ell_o(\Phi(m,x),y)\Big]^2\mathbb{V}_{\theta\sim m}\Big(\phi(\theta,x)\Big)\mathbb{V}_{\theta'\sim m}\Big(\phi(\theta',x)\Big)\nu(\mrd z)\\
        &= \int_\mcZ\Big[\partial_{\hat{y}\hat{y}}\ell_o(\Phi(m,x),y)\Big]^2\Big[\mathbb{V}_{\theta\sim m}\Big(\phi(\theta,x)\Big)\Big]^2\nu(\mrd z)\\
        &\leq c\, \mathbb{E}_{Z\sim \nu}[(1+\|Z\|)^2] \mathbb{E}_{\theta\sim m}\big[(1+\|\theta\|^2)^2\big].\\
    \end{split}
    \end{equation}
giving the stated inequalities. As $\nu\in \mathcal{P}_2(\mcZ)$, and assuming $m\in \mathcal{P}_8(\Theta)$. It follows that Assumption \ref{ass:KLreg_assn} holds.
\end{proof}

\section{More Discussion}\label{app: comparison}
In this section, we provide more details on comparison between $\alpha$-ERM and fine-tuning scenario.

\subsection{Complexity Terms}
In this section we compare the complexity terms in fine-tuning and $\alpha$-ERM, i.e., $\comp(\thetac,\thetas^t,\thetas^s)$ and $\comp(\theta)$, respectively.

Assuming $\gammac\in\mathcal{P}_8(\Thetac), \tilgammasp,\gammasp\in\mathcal{P}_8(\Thetas), \tilgammap \in\mathcal{P}_8(\Theta)  $, the  $ \comp_2(\thetac,\thetas)$ and $\comp_2(\theta)$ are defined based on $L_2$-norm,
\begin{equation}\label{eq: comp finetune}
\begin{split}
    \comp_2(\thetac,\thetas^t,\thetas^s)&=\Big[1+\int_{\Thetac} \|\thetac\|_2^4(2+\|\thetac\|_2^4)\gammac(\mrd\thetac)\\
&\quad+\int_{\Thetas}\|\thetas^t\|_2^4(1+ 2\|\thetas^t\|_2^4)\tilgammasp(\mrd \thetas^t)+\int_{\Thetas}\|\thetas^s\|_2^4\gammasp(\mrd \thetas^s)\Big]^2,
    \end{split}
\end{equation}
and
\begin{equation}\label{eq: comp alphaRenyi}
    \comp_2(\theta)=\Big(1+ 2 \int_{\Theta}\|\theta\|_2^8\tilgammap(\mrd \theta) +2\int_{\Theta}\|\theta\|_2^4\tilgammap(\mrd \theta)]\Big)^2,
\end{equation}
which are bounded. We think of these as measures of the complexity of our model, as they determine how extreme the parameters in the model can become, which acts as a proxy for the model's expressive power. We note that these quantities do not depend on the training data, and are only functions of the chosen initialization distribution and regularization.

\begin{lemma}\label{lem: complexity}
Let $\theta\in\Theta\subset\mbR^{\dimg}$ , $\thetac\in\Thetac\subset\mbR^{\dimc}$, $\thetas\in\Thetas\subset\mbR^{\dims}$ and $\tilgammap\in\mathcal{P}_8(\Theta)$. Assume the following conditions:

\begin{itemize}
    \item For each dimension $i\in[\dimg]$, the 8th moment of $\theta_i$ under the marginal distribution $\tilgammap[i]$ is bounded by $B$:
$\mbE_{\theta_i\sim \tilgammap[i]}[\theta_i^8]\leq B$
\item The 8th moments of each dimension for distributions $\gammac$, $\gammasp$, and $\tilgammasp$ are also bounded by $B$. 
\end{itemize}

Then, the following inequalities hold,
\begin{equation}
\begin{split}
\comp_2(\thetac,\thetas)&\leq \big(1+(\dimc^4 +2\dims^4)B+2(\dimc^2+\dims^2)B^{1/2}\big)^2, \\
\comp_2(\theta) &\leq \big(1+\dimg^{4}B+\dimg^2 B^{1/2}\big)^2.
\end{split}
\end{equation}
\end{lemma}
\begin{proof}
    Note that we have,
    \[\begin{split}
        \int_{\Thetac} \|\thetac\|_2^4\gammac(\mrd\thetac)
        &= \int_{\Thetac} \Big(\sum_{i=1}^{\dimc}\theta_{i,c}^2\Big)^2\gammac(\mrd\thetac)\\
        &\leq \int_{\Thetac} \dimc\Big(\sum_{i=1}^{\dimc}\theta_{i,c}^4\Big)\gammac(\mrd\thetac)\\
        &=  \dimc\sum_{i=1}^{\dimc}\mbE_{\gammac[i]}[\theta_{i,c}^4]\\
        &\leq \dimc^2 B^{1/2}.
    \end{split}\]
    where the last inequality follows from Jensen inequality. Similarly, we have,
    \begin{equation}
        \begin{split}
            \int_{\Thetac} \|\thetac\|_2^8\gammac(\mrd\thetac)&\leq \dimc^4 B,\\
            \int_{\Thetas} \|\thetas\|_2^8\gammasp(\mrd\thetas)&\leq \dims^4 B,\\
            \int_{\Theta} \|\theta\|_2^8\tilgammap(\mrd\theta)&\leq \dimg^4 B.
        \end{split}
    \end{equation}  
    It completes the proof.
\end{proof}
 From Lemma~\ref{lem: complexity} and due to the fact that $\dimg=\dims+\dimc$, we observe that the upper bound on $\comp_2(\thetac,\thetas)$ is demonstrably lower than that on $\comp_2(\theta)$. This difference is particularly pronounced in numerous applications where $\dims\ll\dimc$, for example, in one-hidden layer neural network in Section~\ref{Sec: Application}, we have $\dims=1$ and $\dimc=q$ where $q$ is the dimension of input. Consequently, this suggests that the upper bound on the WTER for fine-tuning (Theorem~\ref{thm: excess-fine-tune}) will typically be smaller than the upper bound on WTER for $\alpha$-ERM (Theorem~\ref{thm: population alpha erm}), while the two bounds will be asympotically of the same order as $k_d\to \infty$.
\subsection{Non-similar Tasks}
In this section, we study the effect of non-similar tasks where $\Disnd(\nupop^s,\nupop^t)> 0$. For this purpose, we compare the coefficients of $\Disnd(\nupop^s,\nupop^t)$ in Theorem~\ref{thm: excess-fine-tune} and Theorem~\ref{thm: population alpha erm}. Therefore, we have the following similarity terms,
\begin{itemize}
    \item Coefficient of $\Disnd(\nupop^s,\nupop^t)$ in $\alpha$-ERM: 
    \[\begin{split}
    &C_{\alpha}(\theta):=(1-\alpha) \Big(1+\int_{\Theta}\|\theta\|^4\bar{m}_{\alpha}(\mrd\theta) \Big)\Big(1+\int_\Theta \|\theta\|^4\tilgammap(\mrd\theta)\Big)\Big(1+\alpha \mathbb{E}_{Z_1^t}[\|Z_1^t \|^2]+(1-\alpha) \mathbb{E}_{Z_1^s}[\|Z_1^s \|^2]\Big)
    \end{split}\]
    \item Coefficient of $\Disnd(\nupop^s,\nupop^t)$ in Fine-tuning:
    \[\begin{split}
        & C_{\mathrm{FT}}(\thetac,\thetas^s,\thetas^t)\\&:=\Big(1+\int_{\Thetac} \|\thetac\|^4\hat\gamma_{8,c}^\sigma(\mrd\thetac)+\int_{\Thetas}\|\thetas^s\|^4\gammasp(\mrd \thetas^s)\Big)\Big(1+\int_{\Thetac} \|\thetac\|^4\mrmbarc^{s}(\mrd \thetac)+\int_{\Thetas}\|\thetas^s\|^4\mrmbarsp^{s}(\mrd \thetas^s)
    \\&\qquad+\int_{\Thetas}\|\thetas^t\|^4 \tilde{m}_{\spec}^t(\mrd \thetas^t)\Big)\mbE_{Z_1^s}\big[(1+\|Z_1^s\|^2)\big]
     \end{split}\]
\end{itemize}
Similar to Lemma~\ref{lem: complexity}, we can provide the following comparison.
\begin{lemma}\label{lem: comp df}
    Under the same assumptions in Lemma~\ref{lem: complexity}, we have,
    \begin{equation}
        \begin{split}
            C_{\alpha}(\theta)&\leq (1-\alpha)\big(1+\dimg^4B^{1/2})^2\Big(1+\alpha \mathbb{E}_{Z_1^t}[\|Z_1^t \|^2]+(1-\alpha) \mathbb{E}_{Z_1^s}[\|Z_1^s \|^2]\Big),
            \\C_{\mathrm{FT}}(\thetac,\thetas^s,\thetas^t)&\leq \big(1+\dimc^4B^{1/2} + \dims^4B^{1/2} \big)\big(1+\dimc^4B^{1/2}+2\dims^4B^{1/2} \big)\mbE_{Z_1^s}\big[(1+\|Z_1^s\|^2)\big].
        \end{split}
    \end{equation}
\end{lemma}
\begin{proof}
    The proof approach is similar to Lemma~\ref{lem: complexity}.
\end{proof}
 From Lemma~\ref{lem: complexity} and due to the fact that $\dimg=\dims+\dimc$, the upper bound on $C_{\alpha}(\theta)$ can be larger than the upper bound on $C_{\mathrm{FT}}(\thetac,\thetas^s,\thetas^t)$ for the case $\dims\ll\dimc$.

 \subsection{Other Comparison}
 Note that in $\alpha$-ERM, the hyper-parameter $\alpha$  must be tuned for each specific application. Fine-tuning, on the other hand, requires no such hyperparameter tuning. Moreover, in many cases, particularly in large language models, access to the original source training dataset is often impossible, and only the pre-trained model is shared. In these situations, fine-tuning can be directly applied to the target dataset and we can not apply $\alpha$-ERM scenario.
 
 \end{document}